\documentclass[opre,nonblindrev]{informs3_rz} 

\OneAndAHalfSpacedXI 

\usepackage{natbib}
 \bibpunct[, ]{(}{)}{,}{a}{}{,}%
 \def\bibfont{\small}%
\TheoremsNumberedThrough{
\newtheorem{property}{Property}
}     
\ECRepeatTheorems

\EquationsNumberedThrough    

\MANUSCRIPTNO{}

\usepackage{subcaption,enumerate}
\usepackage[ruled,vlined]{algorithm2e}
\usepackage{esvect}

\usepackage{times}
\usepackage{thm-restate}
\usepackage{enumerate}
\usepackage{tikz}
\usepackage{pgfplots}
\usepgfplotslibrary{fillbetween}
\usetikzlibrary{patterns, intersections}

\newcommand{\bbP}{\mathbb{P}}

\newcommand{\bbR}{\mathbb{R}}
\newcommand{\bbE}{\mathbb{E}}

\newcommand{\calF}{\mathcal{F}}
\newcommand{\calT}{\mathcal{T}}
\newcommand{\calI}{\mathcal{I}}
\newcommand{\calC}{\mathcal{C}}

\newcommand{\calD}{\mathcal{D}}

\newcommand{\calE}{\mathcal{E}}

\newcommand{\calN}{\mathcal{N}}
\newcommand{\calG}{\mathcal{G}}
\newcommand{\calS}{\mathcal{S}}
\newcommand{\calY}{\mathcal{Y}}

\newcommand{\calL}{\mathcal{L}}
\newcommand{\calA}{\mathcal{A}}
\newcommand{\zero}{\mathbf{0}}
\newcommand{\diag}{\mbox{diag}}
\newcommand{\one}{\mathbf{1}}

\newcommand{\Var}{\mathrm{Var}}
\newcommand{\Cov}{\mathrm{Cov}}
\newcommand{\Clip}{\mathrm{Clip}}

\newcommand{\naive}{\textsc{Naive}}
\newcommand{\consistent}{\textsc{Consistent}}
\newcommand{\oracle}{\textsc{Oracle}}
\newcommand{\feasible}{\textsc{Feasible}}

\newcommand{\Tr}{\mathrm{Tr}}
\newcommand{\ci}{{\perp\!\!\!\perp}}
\newcommand{\y}{\ensuremath{Y}} 

\newcommand{\bb}[1]{\left[#1\right]}
\newcommand{\bp}[1]{\left(#1\right)}
\newcommand{\bc}[1]{\left\{#1\right\}}
\newcommand{\bn}[1]{\left\|#1\right\|}
\newcommand{\ba}[1]{\left|#1\right|}
\newcommand{\br}[1]{\left\langle#1\right\rangle}

\usepackage[colorlinks=true, citecolor=blue, linkcolor = blue]{hyperref}

\usepackage{xcolor}
\usepackage{color-edits}
\addauthor{vs}{red}
\addauthor{rz}{brown}

\makeatletter
\newenvironment{model}[1][]
  {%
   \begin{algorithm}[#1]
   }
  {\end{algorithm}}
\makeatother

\renewenvironment{proof}{
  \par\noindent\textbf{Proof.} \ignorespaces
}{
  \hfill $\square$\par
}

\begin{document}

\RUNTITLE{Post Reinforcement Learning Inference}

\TITLE{Post Reinforcement Learning Inference}

\ARTICLEAUTHORS{%
\AUTHOR{Vasilis Syrgkanis\thanks{Stanford University, \EMAIL{vsyrgk@stanford.edu}. Vasilis Syrgkanis was supported by NSF Award IIS-2337916.}\quad Ruohan Zhan\thanks{University College London, \EMAIL{ruohan.zhan@ucl.ac.uk}}}



\RUNAUTHOR{Syrgkanis and Zhan}
}

\ABSTRACT{%
We study estimation and inference using data collected by reinforcement learning (RL) algorithms. These algorithms adaptively experiment by interacting with individual units over multiple stages, updating their strategies based on past outcomes. Our goal is to evaluate a counterfactual policy after data collection and estimate structural parameters, such as dynamic treatment effects, that support credit assignment and quantify the impact of early actions on final outcomes. These parameters can often be defined as solutions to moment equations, motivating moment-based estimation methods developed for static data. In RL settings, however, data are often collected adaptively under nonstationary behavior policies. As a result, standard estimators fail to achieve asymptotic normality due to time-varying variance. We propose a weighted
generalized method of moments (GMM)
approach that uses adaptive weights to stabilize this variance. We characterize weighting schemes that ensure consistency and asymptotic normality of the weighted
GMM estimators,
enabling valid hypothesis testing and uniform confidence region construction. Key applications include dynamic treatment effect estimation and dynamic off-policy evaluation.
}


\maketitle

\section{Introduction}

Adaptive data collection has become a staple of the digital economy. Most major digital platforms invoke adaptive experimentation algorithms to optimize their service. Adaptive experimentation allows one to progressively update their experimentation strategy and leads to efficient sample usage \citep{chu2011contextual,agrawal2013thompson} and efficient use of the experimentation budget for increased hypothesis testing power \citep{russo2016simple}. Moreover, frequent adaptive experimentation is starting to become adopted in other domains such as personalized healthcare \citep{murphy2005experimental,offer2021adaptive}.
The popularity of adaptive experiments has increased the availability of data collected from such designs. However, adaptive data collection raises many new research challenges, especially in the case of post-collection statistical analysis. For instance, constructing confidence intervals and testing statistical significance for alternative candidate policies or for structural parameters like average treatment effects, from adaptively collected data, has been shown to be theoretically challenging.

Prior work has mostly focused on the bandit setup, where at each time the experimenter only interacts with units sampled from the environment once and observes the immediate outcome~\citep{deshpande2018accurate,zhang2021statistical,hadad2021confidence,bibaut2021post,zhan2021off}. This unfortunately cannot accommodate many applications including adaptive clinical trials and dynamic treatment regimes, where a patient often receives multiple rounds of treatments to improve the outcome~\citep{murphy2003optimal,lei2012smart}, or a digital platform interacts with their users over a sequence of multiple page visits~\citep{chen2019top}.  

The goal of our work is to address this largely un-explored area of post-adaptive data collection inference, from such adaptive experiments that, within each experiment phase, involve multiple interactions with the same treated unit. In particular,  we consider data that are collected from  reinforcement learning~(RL) algorithms \citep{sutton1998introduction} and provide  estimation and inference for structural parameters of interest under semi-parametric assumptions~\citep{neyman1979c,laan2003unified,chernozhukov2022locally}.

In   RL, an adaptive experimentation algorithm (from now on ``the agent'') interacts with units that are sampled independently and identically (i.i.d.) from the environment.   For each unit, the agent observes an initial state, then applies multiple treatments that cause state transitions, and finally observes an outcome at the end.\footnote{We focus on cases where only the final outcome is observed. Our framework can be generalized to settings where intermediate outcomes are also revealed.} We term the sequence of interactions with a unit as an ``episode''.  For example, consider an educational platform aiming to encourage users to enroll in a course. The platform, acting as the agent, might first interact with a user on the homepage, then guide them to an enrollment page, and finally to a payment page, with each step representing a state transition. Throughout these stages, the platform can experiment with different page designs (treatments) to achieve the final goal of course enrollment.


The agent's behavior policy is typically adapted over time to incorporate learning from past realizations. In particular, we focus on a broad class of adaptive data collection mechanisms, where the agent interacts with a batch of units before updating its policy for subsequently arriving units. This setting encompasses the \emph{episodic reinforcement learning} framework \citep{neu2020unifying} and generalizes the classical bandit framework, in which the agent interacts with each unit for only a single round in an adaptive manner.

Such adaptivity progressively improves the agent's performance  but results in a nonstationary behavior policy and introduces dependence between  observations. 
As a result, we cannot simply view the data from each episode as i.i.d.~samples and pass them to traditional estimation pipelines \citep{lewis2020double}. 
Even when employing estimation techniques designed to address time-series correlation, these often require stationarity—a condition not met by most adaptive experimentation processes.
This nonstationarity causes  evolving discrepancy between the behavior and target policies, an issue known as changing ``overlap'' and resulting in varied estimation variances for sequentially collected samples \citep{imbens2004nonparametric,hadad2021confidence}. 
Therefore, averaging samples uniformly can be suboptimal, resulting in significant variance and a non-normal asymptotic distribution, which complicates post-experimental inference.
This problem, evident even in single-interaction bandit scenarios, has prompted recent literature to suggest re-weighting samples to stabilize time-varying variance,  such that the resulting estimators are consistent and asymptotically normal~\citep{deshpande2018accurate,hadad2021confidence,zhang2021statistical,zhan2021off,bibaut2021post}.

RL settings further complicate the estimation problem in two ways. 
First, exogenous random shocks during state transitions within an episode affect subsequent states and are correlated with future treatments; these shocks, being unobservable, cause  unmeasured confounding~\citep{robins1986new,robins2004optimal,chakraborty2013semi}.
To address the identification issue, we follow the semi-parametric  inference literature,  formulating the dynamic treatment effect estimation problem as estimating the structural parameters in a structural nested mean model~\citep{robins2004optimal,lok2012impact,vansteelandt2016revisiting}. 
In our model, each stage is linked to a specific structural parameter that quantifies the treatment effect at that stage. We demonstrate that these parameters are the solutions to stage-wise moment equations, derived through $g$-estimation and constructed in reverse order, from the last stage to the first. Given the nonstationarity of RL data, which often lead to time-varying estimates across units as discussed above, we apply nonuniform and adaptive weights to the unit samples to stabilize these variances. 

Second, estimation with RL data often faces the ``curse of horizon'' challenge, where the overlap between behavior and target policies deteriorates across episodic stages, leading to accumulated estimation variance.
 Prior work suggests heuristics like weight clipping, which---while effective in controlling variance---introduce a small bias \citep{precup2000eligibility,chen2019top}.
However, our approach, through nuanced modeling, ensures that moment equations across stages remain uncorrelated with zero covariance. This allows for the application of stage-wise adaptive weights, which stabilize the variance at each stage based on information available up to that point, effectively circumventing the problem of variance accumulation.


This work also enriches the inference literature when using adaptively collected data by providing  results for GMM-estimation. Prior work mostly addresses inference in the context of M-estimation \citep{deshpande2018accurate,zhang2021statistical}, i.e. parameters that can be defined as the minimizers of a population loss function. However, structural parameters in problems defined by moment conditions, which often arise in the dynamic treatment regime settings, cannot be phrased as M-estimators (rather they can be better thought as instrumental variable problems).

\subsection{Our Contributions}

Our main contributions are outlined as follows.
First, we propose a weighted GMM-estimator using adaptive RL data to estimate   structural parameters of counterfactual policy values and dynamic treatment effects.
These parameters are defined by moment conditions within a dynamic treatment regime, modeled through a semiparametric structure nested mean model for each unit's episodic data.
 Specifically, for a   target counterfactual policy $\pi^*$, we aim to estimate its policy value  $\theta_0^*$---the expected final outcome $Y$ under policy $\pi$. We further define $\theta_j^*$ as the stage-wise dynamic treatment effect of treatment $T_j$ at stage $j$, reflecting the expected change in the final outcome $Y$ due to $T_j$, assuming future treatments follow $\pi^*$. Let $L$ be total number of stages in an episode for each unit.  Our objective is to derive the structural parameter set $\theta^*=(\theta_0^*, \theta_1^*,\dots, \theta_L^*)$ from RL data, capturing both the policy value and stage-wise dynamic treatment effects.

We hereby  illustrate the core concept of our approach under the episodic RL setup.
We consider $n$ sequentially collected episodes $\{Z_1, \dots, Z_n\}$. Let  $\calF_i:=\sigma(Z_{1:i-1})$ represent the $\sigma$-algebra with respect to which the random variables $Z_{1:i-1}$ are measurable.
The true parameter vector $\theta^*$ satisfies the moment condition $\mathbb{E}[m(Z_i;\theta^*) \mid \mathcal{F}_i] = 0$, where $m(\cdot)$ denotes bounded linear moment functions derived from our semiparametric model. To enhance the stability and efficiency of the estimation, we adopt the generalized method of moments (GMM) throughout this paper.\footnote{The results presented here can be directly extended to Z-estimation.} The Jacobian of the moment function, denoted by $J_i := \nabla_{\theta} m(Z_i; \theta^*)$, is nonstationary across episodes due to adaptivity. As a result, the traditional (unweighted) GMM estimator, which is based on minimizing the norm of empirical moment equations, may exhibit unstable asymptotic behavior. To address this, we introduce adaptive weights $H_i$, measurable with respect to $\mathcal{F}_i$, to stabilize the estimation.

Let $\hat{\theta}_n$ denote the weighted GMM estimator that minimizes $\bn{\frac{1}{n}\sum_{i=1}^n H_i m (Z_i;\theta)}_A^2$, where $A$ is a positive-definite matrix and $\|\cdot\|_A$ is the norm induced by $A$. When $\hat{\theta}_n$ lies in the interior of $\Theta$, where $\Theta$ is the bounded parameter space of $\theta^*$, it satisfies the first-order condition: $\{\frac{1}{n}\sum_{i=1}^n H_iJ_i\}A\{\frac{1}{n}\sum_{i=1}^n H_im (Z_i;\hat\theta_n)\}=\zero$. 
By Taylor expansion, we have:
\begin{equation}
\label{eq:intro}
      B_n^\top A B_n\, (\hat{\theta}_n-\theta^*)=B_n^\top A S_n, \quad\text{where}\quad B_n=-\frac{1}{n}\sum_{i=1}^nH_i\,J_i, \quad S_n =\frac{1}{n}\sum_{i=1}^n H_i\,m(Z_i;\theta^*).
\end{equation}
    The property of $\hat{\theta}_n$ thus depends on the behavior of $B_n$ and $S_n$, both of which can be stabilized by the choice of weights $\{H_i\}$.

Our second contribution is to identify generic weighting schemes to achieve  consistency and asymptotic normality of the weighted GMM-estimator $\hat{\theta}_n$.
The choice of weights $\{H_i\}$ should address two critical aspects: ensuring the normalizing matrix $B_n^\top A B_n$ in \eqref{eq:intro} remains well-posed to prevent the explosion of the estimated parameter $\hat{\theta}_n$, and   the right-hand side $B_n^\top A S_n$ converges to a Gaussian distribution for asymptotic normality. 
 We show that $S_n$
is the sum of a martingale difference sequence, and we establish its convergence by leveraging recent advances in martingale limiting theorems  under a crucial homoscedasticity assumption. 
We note that simply providing a central limit theorem for \eqref{eq:intro} is not enough to achieve inference results on policy value $\theta_0^*$ or dynamic treatment effect $\theta_j^*$, since the normalizing matrix $B_n^\top A B_n$ is likely to not concentrate by the sheer nature of adaptivity. 
Addressing this, we adapt recent work in \cite{cattaneo2022yurinskii}, which   develops strong Gaussian approximation results for martingale data,  to our context and show that \eqref{eq:intro} converges to a Gaussian distribution at a uniform  rate. This allows us to invert the normalizing matrix $B_n^\top A B_n$ and approximate the estimation error, $\hat{\theta}_n-\theta^*$, with a Gaussian distribution, paying the way for inference results that are the primary focus of this paper. 

Our third contribution is to substantiate the generic weighting schemes and offer specific  weighting choices to achieve consistency and asymptotic normality  for  a wide range of RL algorithms with polynomially decaying exploration rates.
We focus on bilinear parametrization, which includes  common setups such as categorical treatment and polynomial scalar treatment. We show that the adaptive weights for consistency can be directly derived from the collected data. 
Further, we identify the ``oracle weights'' to achieve strong Gaussian approximation. While these  weights
may require additional estimations, 
 the necessary quantities to be estimated only depend on the state transition dynamics, which are independent of the behavior policy; this allows for their consistent estimation through online regression techniques, which we term ``feasible weights''.
We prove that the GMM-estimator, when applied with these feasible weights, 
 preserves strong Gaussian approximation, thereby facilitating  post-RL inference.

Finally, we apply our estimation and inference framework to high-dimensional  Markovian models.
We offer estimation guarantees for feasible weights and show that the adaptively weighted GMM-estimator, when using  these feasible weights, achieves strong Gaussian approximation.  This approximation enjoys a uniform convergence rate of $O(n^{-(1-L\alpha)/10})$, where $L$ represents the episode length, and $\alpha$ indicates the exploration decaying rate of the RL agent. 
We further substantiate our theoretical guarantees with numerical evidence.
Our findings show that the standard GMM-estimator, when unweighted, deviates from asymptotic normality, resulting in either insufficient coverage or overly conservative inferences for structural parameters. Conversely, our weighted GMM-estimator, applied with feasible weights, consistently achieves near-perfect coverage, tighter confidence intervals, and more precise estimates for evaluation policy values. Its performance closely matches that of estimations under oracle weights, which, though ideal, require inaccessible true structural parameters and are generally impractical.
Furthermore, we demonstrate our method's robustness under model misspecification,  employing a polynomial model to approximate the true exponential model within our semiparametric framework. Even under misspecification, our weighted GMM-estimator with feasible weights outperforms the standard approach, offering more accurate policy value estimates and maintaining near-nominal coverage as approximation complexity increases. This robustness highlights our method's suitability for real-world scenarios, where model misspecification is often inevitable.

\section{Setup}
This section establishes our problem framework by detailing the data generating process. We start by describing the stochastic control process, which involves multiple stages of interaction with individual units. Subsequently, we define the  structural parameter estimation problem within this framework.

\subsection{Episodic Potential Outcome}
We start with describing the stochastic control process to roll out an episode for a unit and the corresponding potential outcome. 
Consider each unit  starts from an  initial state~$S_1$ i.i.d.~sampled from a fixed distribution $P_S$. 
An RL agent then runs an episode of $L$ stages for each unit.
At each stage $j \in [1:L]$, the agent observes the unit’s current state $S_j \in \calS \subset \mathbb{R}^{d_s}$ and then assigns a treatment $T_j \in \calT \subset \mathbb{R}^{d_\tau}$ (also known as an action or intervention), where $d_s$ and $d_\tau$ denote the dimensions of the state and treatment, respectively. The treatment transitions the unit to a new state $S_{j+1}$ for the next stage. The long-term outcome of interest, denoted by $\y \in \calY \subset \mathbb{R}$, is observed only at the end of the episode $(S_1, T_1, S_2, \dots, S_L, T_L)$.
 
 Following the potential outcome framework, we  denote the potential outcome under a sequence of  treatment assignments  $\{\tau_1,\dots, \tau_L: \tau_j\in\calT\}$ as $\y(\tau_{1:L})$.\footnote{We use $A_{n_1:n_2}$ as a short hand for the set $\{A_{n_1},\dots, A_{n_2}\}$ for any indexed variable $A_i$ and any two integers $n_1\leq n_2$.} 
A dynamic treatment assignment policy $\pi = \{\pi_1, \dots, \pi_L\}$ specifies a sequence of treatment rules applied across stages of an episode. Specifically, each $\pi_j$ determines the treatment $T_j$ at stage $j$ based on the history $(S_{1:j}, T_{1:j-1})$, with $\pi_j : \calS^j \times \calT^{j-1} \rightarrow \calT$. We  overload the notation and use $\y(\pi_{1:L})$ to denote the potential outcome under policy $\pi$.
We introduce an assumption critical to our analysis: all confounders influencing both the treatment assignment and the final outcome are observed by the states.
\begin{assumption}[Sequential Conditional Exogeneity]
\label{assump:exogeneity}
The data generating process satisfies 
\begin{equation*}
 \y(T_{1:j-1}, \pi_{j:L})\, \ci\, T_j \mid T_{1:j-1}, S_{1:j}, \quad \mbox{for any}~ j=1,\dots,L.
\end{equation*}
\end{assumption}
This assumption is an extension of the unconfoundedness assumption in  causal inference literature to settings where a unit undergoes multiple stages of treatment \citep{rosenbaum1983central,robins2004optimal}. 
\begin{remark}
    A Markov decision process (as depicted in Figure~\ref{fig:cg}) satisfies Assumption~\ref{assump:exogeneity}. In contrast, a partially observable Markov decision process (POMDP) violates this assumption. By POMDP, we refer to settings where the actual treatment received depends on some latent state; this could occur if the behavior policy relies on unobserved information beyond $S_t$, or if there is non-compliance. However, such scenarios do not arise in our analysis framework, as we assume full knowledge of the behavior policy.
\end{remark}

\subsection{Inference Goal: Value of Dynamic Treatment Assignment Policy}
 
Given a policy $\pi^*$ to be evaluated, the goal in this paper is to estimate its policy value, which is defined as the mean outcome under this policy, outlined below: 
\begin{equation*}
    \theta^*_0 := \bbE\bb{Y(\pi^*_{1:L})}.
\end{equation*}
Inference on $\theta^*_0$ allows one to argue whether the observed outcome achieved under the behavior policy is statistically larger than the outcome under some simple baseline policies, and therefore whether the data collection agent (for example, from some RL algorithm) leads to any statistically significant benefits as compared to these baseline policies.

\begin{figure}
\centering 
\begin{subfigure}[t]{0.45\textwidth}
    \hspace*{-4.5cm}
    \centering
    \begin{tikzpicture}[scale=1.7]
    \node (S) at (0,0) {$S_1$};
    \node (T1) at (1,1) {$T_1$};
    \node (X) at (2,0) {$S_{2}$};
    \node (T2) at (3,1) {$T_{2}$};
    \node  (Y) at (4,0) {$Y$};
    \draw[thick, ->] (S) -- (T1);
    \draw[thick, ->] (S) -- (X);
    \draw[thick, ->] (T1) -- (X);
    \draw[thick, ->] (X) -- (Y);
    \draw[thick, ->] (X) -- (T2);
    \draw[thick, ->] (T2) -- (Y);
    \hspace{10mm}
    \end{tikzpicture}
    \hspace{-3cm}
    \caption{\footnotesize{Causal graph for observed data}}
    \label{fig:cg}
\end{subfigure}
\begin{subfigure}[t]{0.45\textwidth}
    \hspace*{-.7cm}
    \centering
    \begin{tikzpicture}[scale=1.7]
    \node (S) at (0,0) {$S_1$};
    \node (T1) at (0.67, 1) {$T_1$};
    \node (pi1) at (1.67, 1) {$\pi_1(S_1)$};
    \node (X) at (2.33,0) {$S_2(\pi_1)$};
    \node (T2) at (2.8,1) {$T_2$};
    \node (pi2) at (4,1) {$\pi_2(S_2(\pi_1))$};
    \node  (Y) at (4.63,0) {$Y(\pi_{1:2})$};
    \draw[thick, ->] (S) -- (T1);
    \draw[thick, ->] (S) -- (X);
    \draw[thick, ->, dashed] (S) to[bend right] (pi1);
    \draw[thick, ->, dashed] (pi1) -- (X);
    \draw[thick, ->] (X) -- (Y);
    \draw[thick, ->] (X) -- (T2);
    \draw[thick, ->, dashed] (X) to[bend right] (pi2);
    \draw[thick, ->, dashed] (pi2) -- (Y);
    \end{tikzpicture}
    \caption{\footnotesize{Intervention graph for counterfactual data}}
    \label{fig:swig}
\end{subfigure}
\vspace{1em}
\caption{
Causal graphs illustrating the sequential conditional exogeneity assumption for a two-period setting, showing both observed data and counterfactual data for each unit under an alternative policy $\pi$. 
Solid arrows represent the observed treatment assignments and state transitions in the collected data, 
while dotted arrows indicate counterfactual treatment assignments and state transitions under policy $\pi$.
}
\end{figure}
 
We  introduce the following definition to attribute the long-term outcome to intermediate treatments at each stage.

\begin{definition}[Blip function, \cite{robins2004optimal}] 
Given an evaluation policy $\pi^*$, let the  ``blip function'' $\gamma_j^{(\pi^*)}$ at stage $j$ be defined as: 
\begin{align*}
    \gamma_j^{(\pi^*)}(S_{1:j}, T_{1:j}) := \bbE\bb{\y\bp{T_{1:j}, \pi^*_{j+1:L}} - \y\bp{T_{1:j-1}, 0,\pi^*_{j+1:L}} \mid S_{1:j}, T_{1:j}}.
\end{align*}
\end{definition}

In other words, the blip effect $\gamma_j^{(\pi^*)}(\cdot)$, given the realized history $(S_{1:j}, T_{1:j-1})$ and conditional on assigned treatment $T_{i,j}$, captures the effect of assigning treatment $T_{i,j}$ in the current period relative to a baseline control policy $0 \in \calT$ that applies the status quo treatment, assuming future treatments follow policy $\pi^*$. From now on, 
we fix the evaluation policy $\pi^*$  and  omit the superscript in the blip function, i.e. we use $\gamma_j(\cdot)$ to denote $\gamma_j^{(\pi^*)}(\cdot)$ for notation convenience.

\begin{remark}
    This blip function  $\gamma_j(\cdot)$  quantifies the treatment effect of $T_j$ at stage $j$ on the long-term outcome $Y$. This concept is analogous to the ``advantage'' function used in the RL literature, which captures the discrepancy in Q-value functions  from taking alternative actions \citep{baird1995residual,shi2022statistically}.
\end{remark}

To identify the policy value of $\pi^*$ from the observed data, we introduce a semi-parametric structural assumption on the blip functions, 
following \cite{lewis2020double,robins2004optimal}.
\begin{assumption}[Linear blip assumption]
\label{assump:linear_blip_function}
The blip function has the  linear form:
$\gamma_j(S_{1:j}, T_{1:j}) =  \phi_j( S_{1:j}, T_{1:j})^\top \theta_j^* $
for some unknown  structural parameter  $\theta_j^*\in\bbR^d$  and a known feature mapping function $\phi_j:\calS^j\times \calT^j\rightarrow\bbR^d$ that satisfies $\phi_j(s_{1:j}, (\tau_{1:j-1},0))=\mathbf{0}$ for any $s_{1:j}\in\calS^j$ and $\tau_{1:j-1}\in\calT^{j-1}$. 
As shorthand notation, let $\Phi_j:=\phi_j(S_{1:j}, T_{1:j}) - \phi_j(S_{1:j}, T_{1:j-1}, \pi^*(S_{1:j}, T_{1:j-1}))$.
\end{assumption}
Note that even when the blip functions are assumed to be parametrically linear, we do not make any assumption on the conditional mean of the evaluation policy outcome  $\bbE\bb{\y(\pi^*_{1:L})\mid S_1}$, which can be potentially nonlinear in the initial state $S_1$. In later sections we  
 shall use moment equations to identify the unknown   structural parameter  $\theta_j$, and this assumption implies that the moment conditions are linear in target parameters. 
This assumption allows for a wide range of model classes including high-dimensional Markovian models substantiated  in Section \ref{sec:partial_linear_model}.

To this end, we can explicitly characterize the policy value $\theta^*_0$ for evaluation policy $\pi^*$, following  results in \cite{robins2004optimal} -- which we adapt to our setup and provide its proof (along with all other proofs) in the appendix for completeness.\footnote{This result can be viewed as an analogue of the \emph{performance difference lemma} that is frequently used in the reinforcement learning literature \cite{kakade2002}.}
\begin{lemma}[Expected Outcome via Blip Functions]
\label{lemma:identification_policy_value}
Under Assumptions~\ref{assump:exogeneity}~\&~\ref{assump:linear_blip_function}, it holds that 
\[
    \bbE\bb{\y(T_{1:j-1}, \pi^*_{j:L})\mid  S_{1:j}, T_{1:j}} = \bbE\bb{Y -\sum_{j'=j}^L\Phi_{j'}^\top \theta_{j'}^* \mid S_{1:j}, T_{1:j}}.
\]
Hence, the evaluation policy value is $\theta^*_0= \bbE\bb{\y-\sum_{j=1}^L\Phi_j^\top \theta_j^*}$.
\end{lemma}

\subsection{Research Question}

We conclude this section by formally describing our problem within the context formalized above. 

\paragraph{Adaptive Episodic RL Data.} 
We consider the setting where an RL agent sequentially rolls out episodes for $n$ units, each over $L$ stages. Before episode~$i$ begins, the agent updates the treatment assignment policy $\pi_i := \{\pi_{i,1}, \dots, \pi_{i,L}\}$ based on data from previous episodes. Let $Z_i := \{S_{i,1}, T_{i,1}, \dots, S_{i,L}, T_{i,L}, Y_i\}$ denote the data collected during episode~$i$. The distribution of $Z_i$ generally depends on the past realizations $Z_{1:i-1}$. The full dataset consists of observations $\{Z_i\}_{i=1}^n$. Our results directly extend to batched episodic RL settings, where the RL agent may interleave interactions with units within a batch and simultaneously progress these units through their respective episodes. This extension is possible because our technical analysis is based on filtrations defined with respect to unit-level episodes. As long as the behavior policy remains fixed during the rollout of each individual unit, similar filtrations can be constructed based on episode completion, and our results continue to hold.

We assume that the behavior policy $\pi_i$ used to generate data for episode~$i$ is known and recorded by the RL agent. Without this information, recovering the behavior policy from the data is generally difficult or even statistically impossible, since each $\pi_i$ may correspond to only one observed episode. This assumption is also common in prior work on inference with adaptively collected bandit data \citep{hadad2021confidence,zhang2021statistical}.

\paragraph{Goal.} Given the data, we aim to estimate and perform inference on the structural parameters  in the parameter space $\Theta\subset \bbR^{1+dL}$:
\[\theta^* = (\theta_0^*, \theta_1^* ,\dots,\theta_L^*)\in\Theta,\]
where $\theta_0^*$ corresponds to the value of the evaluation policy $\pi^*$, and  $\{\theta_j^* \in \bbR^d:j\in[1:L]\}$ denote the dynamic treatment effect parameters that appear in the linearization of the  blip functions. 
In particular, we assume that $\theta^*$ is an interior point of $\Theta$.

\paragraph{Notation.} 
We  use $i$ to index episode and $j$ to index stage within an episode. Define the $\sigma$-fields $\calF_{i} := \sigma(\{Z_{1:i-1}\})$  and $\calF_{i,j} := \sigma(\{Z_{1:i-1}, S_{i, 1:j-1}, T_{i,1:j-1}\})$ with the convention that $\calF_{i,1}\equiv \calF_{i}$ and $\calF_{i,0}\equiv \calF_{i}$. 
We also introduce an augmented $\sigma$-field: $\calF_{i,j}^+$, which extends $\calF_{i,j}$ by including the $\sigma$-algebra generated by the next state $S_{i,j}$; this additional information about state $S_{i,j}$ is available prior to the treatment assignment by the agent for unit $i$ at stage $j$.
We  use $\bbE_{i}[\cdot]$ as a short hand for the conditional expectation $\bbE[\cdot\mid \calF_{i}]$ and use $\bbE_{i,j}[\cdot]$ for $\bbE[\cdot \mid \calF_{i,j}]$, $\bbE_{i,j}^+[\cdot]$ for $\bbE[\cdot \mid \calF_{i,j}^+]$. 
Similarly, we define $\Var_{i}(\cdot), \Var_{i,j}(\cdot)$ and $\Var_{i,j}^+(\cdot)$ as shorthand for the conditional variances $\Var(\cdot\mid \calF_{i}), \Var(\cdot\mid \calF_{i,j})$ and $\Var(\cdot\mid \calF_{i,j}^+)$, respectively. For a random vector $X$, we define $\Var_i(X)=\bbE_i[XX^\top] - \bbE_i[X]\bbE_i[X]^\top$, $\Var_{i,j}(X)=\bbE_{i,j}[XX^\top] - \bbE_{i,j}[X]\bbE_{i,j}[X]^\top$ and $\Var_{i,j}^+(X) = \bbE_{i,j}^+[XX^\top]-\bbE_{i,j}^+[X] \bbE_{i,j}^+[X]^\top$.

\section{Identification and Estimation}
In this section, we provide conditional moment equations for identifying the target parameters. We then propose our adaptively weighted   method of  moments (AW-GMM) for  parameter estimation using the adaptively collected RL data.

\subsection{Identification via Moment Equations}
We construct moment equations to identify $\theta^*$, as formalized by the  lemma below.  This approach is known as the $g$-estimation framework, typically used in the context of structural nested mean models~\citep{robins2004optimal}. 
However, most prior work focuses on data collected by a fixed behavior policy \citep{robins2004optimal,chakraborty2013semi,lewis2020double}, while here we apply the $g$-estimation approach to adaptive RL data and show that the true parameter vector $\theta^*$ also satisfies the moment restrictions proposed in \cite{robins2004optimal}. 

\begin{lemma}[Identification of Blip Functions]
\label{lemma:identification_parameter}
For unit $i$ and stage $j$, define 
\begin{equation}
    \Phi_{i,j}:=\phi_j(S_{i, 1:j}, T_{i,1:j}) - \phi_j(S_{i, 1:j}, T_{i,1:j-1}, \pi^*(S_{i, 1:j}, T_{i,1:j-1})),
\end{equation}
and
\begin{equation}
    \label{eq:residual}
    R_{i,j}:= \y_i - \sum_{j'=j}^L \Phi_{i,j'}^\top \theta_{j'}^*.
\end{equation}
Here, $ R_{i,j}$ represents the  \textbf{residual}  by subtracting the  blip effects of the future treatments $T_{i,j:L}$ from the final outcome~$\y_i$ 
and adding the blip effects of the treatments assigned by the evaluation policy $\pi^*$.
Given  known functions $\psi_j:\calS^j\times \calT^j\rightarrow\bbR^{d'}$ with $d'\geq d$, write 
the instruments $\Psi_{i,j}:=\psi_j(S_{i,1:j}, T_{i,1:j})$, and define $\bar{\Psi}_{i,j}:=\bbE^+_{i,j}[\Psi_{i,j}]$.
Under Assumptions~\ref{assump:exogeneity}~\&~\ref{assump:linear_blip_function}, the true parameter vector $\theta^*$ satisfies:
\begin{equation}
\label{eq:conditional_moment_identification}
 \bbE_{i,j}^+\bb{ 
    R_{i,j}
    \bp{ \Psi_{i,j} - \bar{\Psi}_{i,j} } 
 }= 0, \quad \forall i\in\{1:n\}, j\in\{0:L\},
\end{equation}
with the convention that $\Psi_{i,0}=1$, $\bar{\Psi}_{i,0}=0$, and $\Phi_{i,0}=1$.

\end{lemma}

\begin{remark}
One approach to constructing instruments is to set $\Psi_{i,j} = \Phi_{i,j}$ by defining
\[\psi_j(s_{1:j}, \tau_{1:j})=\phi_j(s_{1:j}, \tau_{1:j}) - \phi_j(s_{1:j}, \tau_{1:j-1},\pi^*(s_{1:j}, \tau_{1:j-1}))\]
for any $(s_{1:j},\tau_{1:j})\in\calS^j\times \calT^j$. When additional information is available, it is possible to include more instruments (such that $\Psi_{i,j}$ has  higher dimension than $\Phi_{i,j}$) to improve estimation efficiency.
\end{remark}
The moment conditions in the above lemma can be expressed more compactly as a single vector of moment constraints. Given a unit $i$ with episodic data $Z_i$, 
define  $\beta_i:=(\beta_{i,0},\beta_{i,1},\dots, \beta_{i,L})$, where each component is given by $\beta_{i,j}=Y_i(\Psi_{i,j} -\bar \Psi_{i,j})$; define 
\begin{align*}
    J_i =- \begin{bmatrix}
  (\Psi_{i,0}-\bar\Psi_{i,0})\Phi_{i,0}^\top  
  &\quad(\Psi_{i,0}-\bar\Psi_{i,0})\Phi_{i,1}^\top &\quad\dots &\quad(\Psi_{i,0}-\bar\Psi_{i,0})\Phi_{i,L}^\top \\
\zero &\quad(\Psi_{i,1}-\bar\Psi_{i,1}) \Phi^\top_{i,1} &\quad \dots &\quad(\Psi_{i,0}-\bar\Psi_{i,0})\Phi^\top_{i,L}\\
\vdots &\quad \ddots &\quad\ddots &\quad\vdots\\
\zero &\quad \dots &\quad \zero &\quad (\Psi_{i,L}-\bar\Psi_{i,L})\Phi^\top_{i,L}
   \end{bmatrix}. 
\end{align*}
 Then
for any  parameter vector $\theta$, the moment function $m(Z_i;\theta)$ can be written   as:
\begin{equation}
    m(Z_i;\theta) =~ \beta_i + J_i \theta.
\label{eq:moment}
\end{equation}
In particular, the moment condition in \eqref{eq:moment} can be constructed directly from observed data. The key quantity, $\bar\Psi_{i,j} = \bbE_{i,j}^+[\Psi_{i,j}]$, is the conditional expectation over the treatment assignment $T_{i,j}$ under the behavior policy. Since the behavior policy is assumed to be known, this expectation can be computed exactly. As a result, the moment function $m(Z_i; \theta)$ is also computable from  the observed data.

Finally, Lemma \ref{lemma:identification_parameter} implies that the true model parameter vector $\theta^*$ solves the expected moment equation:
\begin{align}
\label{eq:moment_condition}
     \bar\beta_i+\bar J_i \theta^* =\zero, \quad  \forall i = 1,\dots, n,
\end{align}
where  $\bar\beta_i:=(\bar\beta_{i,0},\bar\beta_{i,1},\dots, \bar\beta_{i,L})$ for $\bar\beta_{i,j}=\bbE_{i,j}^+\bb{Y_i(\Psi_{i,j} -\bar \Psi_{i,j})}$, and 
\begin{align*}
   \bar{J}_i :=-  \begin{bmatrix}
  \bbE_{i,0}^+[(\Psi_{i,0}-\bar{\Psi}_{i,0})\Phi_{i,0}^\top]
  &\quad 
  \bbE_{i,0}^+[(\Psi_{i,0}-\bar{\Psi}_{i,0})\Phi_{i,1}^\top] 
  &\quad 
  \dots 
  &\quad 
  \bbE_{i,0}^+[(\Psi_{i,0}-\bar{\Psi}_{i,0})\Phi_{i,L}^\top] \\
\zero &\quad 
\bbE_{i,1}^+[(\Psi_{i,1}-\bar{\Psi}_{i,1})\Phi_{i,1}^\top] &\quad 
\dots &\quad 
\bbE_{i,1}^+[(\Psi_{i,1}-\bar{\Psi}_{i,1})\Phi_{i,L}^\top]\\
\vdots &\quad 
\ddots &\quad \ddots &\quad  \vdots\\
\zero &\quad  \dots &\quad  \zero &\quad  \bbE_{i,L}^+[(\Psi_{i,L}-\bar{\Psi}_{i,L})\Phi_{i,L}^\top]
   \end{bmatrix}.
\end{align*}
This observation will become handy in our subsequent analysis.

\subsection{Adaptively Weighted Generalized Method of Moments (AW-GMM)}
Given the moment conditions in \eqref{eq:moment_condition}, one can estimate the target parameter $\theta^*$ using moment-based methods such as Z-estimation or the generalized method of moments (GMM). To improve stability and efficiency (particularly when $\Psi$ has higher dimension than $\Phi$), we focus on GMM throughout this paper. The results also apply directly to Z-estimation. Specifically, the standard GMM estimator $\tilde{\theta}_n$ minimizes the norm of the empirical moment equations:
\begin{equation*}
   \tilde{\theta}_n \in \argmin_{\theta\in\Theta}~\bn{\frac{1}{n}\sum_{i=1}^n m (Z_i;\theta)}_A^2 =
   \bc{\frac{1}{n}\sum_{i=1}^n m (Z_i;\theta)}^\top A \bc{\frac{1}{n}\sum_{i=1}^n m (Z_i;\theta)},
\end{equation*}
where $A$ is a positive-definite  matrix and $\|\cdot\|_A$ denotes the norm induced by $A$. When $\tilde{\theta}_n$ lies in the interior of $\Theta$, it satisfies the first-order condition: $\{\frac{1}{n}\sum_{i=1}^n J_i\}A\{\frac{1}{n}\sum_{i=1}^n m (Z_i;\tilde\theta_n)\}=\zero$.

To build more intuition, consider the case where  the sample average  Jacobian $\frac{1}{n}\sum_{i=1}^n J_i$ is invertible. Then the first-order condition implies that the GMM estimator solves the empirical moment equations such that: $\frac{1}{n}\sum_{i=1}^n m (Z_i;\tilde\theta_n)=\zero$. The  asymptotic behavior of $\tilde{\theta}_n$ is typically analyzed through a Taylor expansion around the true parameter $\theta^*$, which links the normalized estimation error $J_i(\tilde{\theta}_n-\theta^*)$ with an empirical influence function $m(Z_i; \theta^*)$:
\begin{align}
\frac{1}{n} \sum_{i=1}^n J_i (\tilde{\theta}_n - \theta^*)  = \frac{1}{n} \sum_{i=1}^n m(Z_i; \tilde\theta_n) -\frac{1}{n} \sum_{i=1}^n m(Z_i; \theta^*) =-\frac{1}{n} \sum_{i=1}^n m(Z_i; \theta^*) 
\end{align}
However, since the behavior policy is adaptively evolving over time, the Jacobian $J_i$ is likely to diverge, 
and the variance of the empirical influence functions on the right-hand-side might fluctuate and never converge.
This risk of instability  motivates us to modify the GMM estimator by applying  \emph{non-uniform} and \emph{time-varying} adaptive weights, which are carefully chosen to stabilize the variance of the empirical influence functions $m(Z_i;\theta^*)$ over time, 
thereby preserving the desirable asymptotic characteristics of the GMM-estimator.

We hereby propose our \emph{adaptively weighted} generalized method of moments (AW-GMM) estimator $\hat{\theta}_n$, which  minimizes the norm of a \emph{non-uniform} average of empirical moment equations:
\begin{align}
\label{eq:gmm_estimator_linear}
\hat{\theta}_n \in \argmin_{\theta\in\Theta}~\calL_n(\theta):=\bn{\frac{1}{n}\sum_{i=1}^n H_i\,m (Z_i;\theta)}_A^2,
\end{align}
where $H_{i}$ denotes the time-varying  weighting matrix designed to counterbalance the potential divergence in the empirical influence functions. Specifically, we consider a block-diagonal structure for $H_{i}$:  $H_i:= \mbox{diag}\bc{H_{i,0}, H_{i,1}, \dots, H_{i,L}}$, where $H_{i,0}\in\bbR$ and $\{H_{i,j}\in \bbR^{d'\times d'}$ for each  $j\in [1:L]\}$; each diagonal  block $H_{i,j}$ is adapted to filtration $\calF_{i,j}^+$ and
stabilizes the 
empirical influence functions 
of data $Z_i$ involved in estimating  the  structural parameter  $\theta_j^*$.
Similar  weighting schemes have been introduced in the bandit setups \citep{deshpande2018accurate,hadad2021confidence,zhang2021statistical,bibaut2021post,zhan2021off}, though none applies to RL data with more than one stages. 

The technical challenge in identifying proper weights for RL data is two-fold. 
First, it is crucial that these weights remain stable and do not become degenerate or explode, otherwise the resulting estimator would suffer from diverging variance. This is known as the  ``curse of horizon'' challenge, which arises from the deteriorating overlap between the behavior and evaluation policies over lengthy episodes.  Previous RL literature on weighting-based offline evaluation has often turned to heuristic solutions, such as weight clipping, to manage variance, albeit at the expense of introducing bias \citep{precup2000eligibility,chen2019top}. We show that by carefully designing our weighting matrix $H_i$, our estimator not only controls estimation variance but also achieves asymptotic unbiasedness. 
Second,  we want to construct these weights  using information  only  up to current observations, following the practice of constructing weights for adaptive bandit data \citep{hadad2021confidence}. By doing so, we can employ martingale limiting theorems to show 
that the weighted GMM-estimator admits asymptotic normality and achieves post-RL inference.

\section{Consistency of AW-GMM Estimation}
\label{sec:consistency}
We begin by introducing a general weighting scheme that ensures consistent estimation of target parameters from adaptively collected RL data. We then apply this framework to specific RL algorithms later in the section. Recall that the weighted GMM estimator $\hat{\theta}_n$ minimizes the empirical loss given by:
\begin{equation}
\label{eq:gmm_loss}
     \calL_n(\theta):=\bn{\frac{1}{n}\sum_{i=1}^n H_im (Z_i;\theta)}_A^2  =\bn{\frac{1}{n}\sum_{i=1}^nH_i(\beta_i+J_i\theta)}_A^2.
\end{equation}
By the conditional moment equation \eqref{eq:moment_condition} and  weight $H_{i,j}$ adapting to $\calF_{i,j}^+$, we have the true $\theta^*$ satisfies $H_i(\bar\beta_i+\bar J_i\theta^*)=\zero$ and thus is a minimizer of the norm of averaged conditional moments:
\begin{equation}
\label{eq:conditional_moment}
 \calL(\theta):=\bn{\frac{1}{n}\sum_{i=1}^nH_i(\bar\beta_i+\bar J_i\theta)}_A^2,
\end{equation}
with $\calL(\theta^*)=0$.
To establish that $\hat{\theta}_n$ is consistent for $\theta^*$, it suffices to show two conditions:  (i) for every $\theta\in\Theta$, the GMM loss $ \calL_n(\theta)$ converges to the expected loss $ \calL(\theta)$; and   (ii) the  loss $\calL(\theta)$ has a unique minimizer. The following weighting property ensures both conditions are satisfied.

\begin{property}[$(\alpha_1, \alpha_2)$-Regularizing Weights] 
\label{property:weight_regularizing}
We set $H_{i,0}=1$ for  $i\in[1:n]$. Given a matrix $A$, we use $\lambda_{\min}(A)$ to denote its smallest eigenvalue and $\Tr(A)$ to denote its trace (sum of diagonal entries).
For each $j\in[1:L]$, the weights $\{H_{i,j}\}_{i=1}^n$ are adapted to filtration $\{\calF_{i,j}^+\}_{i=1}^n$ and  satisfy that:
\begin{enumerate}[(a)]
\item $\lambda_{\min}\bp{\bc{\frac{1}{n}\sum_{i=1}^n \bbE_{i,j}\bb{H_{i,j} \Cov^+_{i,j}(\Psi_{i,j},\Phi_{i,j})}}^\top \bc{\frac{1}{n}\sum_{i=1}^n \bbE_{i,j}\bb{H_{i,j} \Cov^+_{i,j}(\Psi_{i,j},\Phi_{i,j})}}
    }  \geq ~ c_1^2\cdot  n^{-\alpha_1}$;
    
     \item $ \frac{1}{n}\sum_{i=1}^n \Tr\bp{H_{i,j}H_{i,j}^\top}\leq c_1^{-1}\cdot  n^{\alpha_1}$;
    \item $ \frac{1}{n} \sum_{i=1}^n \Tr\bp{\bbE[H_{i,j} \Var_{i,j}^+(\Psi_{i,j}) H_{i,j}^\top]} \leq c_1^{-1}n^{\alpha_2}$.
\end{enumerate}
for  universal  constants $c_1>0$,  $\alpha_1\in[0,\frac{1}{L+1})$, and  $\alpha_2 \in[0, \frac{1}{L-1}]$.\footnote{Boundedness of $\Psi$ yields $\Tr(H_{i,j}\Var_{i,j}^+(\Psi_{i,j})H_{i,j}) \lesssim \Tr(H_{i,j}H_{i,j})$. As a result, we have $ \frac{1}{n} \sum_{i=1}^n \Tr\bp{\bbE[H_{i,j} \Var_{i,j}^+(\Psi_{i,j}) H_{i,j}^\top]}\lesssim\frac{1}{n}\sum_{i=1}^n \Tr\bp{H_{i,j}H_{i,j}^\top}$. To make Property~\ref{property:weight_regularizing}(c) nontrivial at the presence of  Property~\ref{property:weight_regularizing}(b), we consider $\alpha_2\leq \alpha_1$.}
\end{property}

In particular, Properties \ref{property:weight_regularizing}(a) \&  \ref{property:weight_regularizing}(b) ensure that $ \calL(\theta)$ has a unique minimizer with high probability, and Property \ref{property:weight_regularizing}(c) ensures that $\calL_n(\theta)$ uniformly converges to $ \calL(\theta)$ for any bounded $\theta\in\Theta$.

\begin{theorem}
\label{thm:consistency}
    Suppose Assumptions~\ref{assump:exogeneity}~\&~\ref{assump:linear_blip_function}
    hold and the weights $\{H_i\}$ satisfy the $(\alpha_1, \alpha_2)$-regularizing Property \ref{property:weight_regularizing}. Also assume that $\calS,\calT,\calY$  are bounded. The AM-GMM estimator $\hat{\theta}_n$ specified in \eqref{eq:gmm_estimator_linear} achieves the following consistency result for $j\in[0:L]$: 
\begin{equation}
 \bbE\bb{\bn{\hat{\theta}_{n,j}-\theta_{j}^*}_2}=O\bp{n^{\frac{(L-j+1)(\alpha_1+\alpha_2)-1}{2}}},\quad\mbox{and}\quad \bbE\bb{\bn{\hat{\theta}_{n,j}-\theta_{j}^*}^2_2} =O\bp{n^{(2L-2j+1)\alpha_1+\alpha_2-1}}.
\end{equation}
\end{theorem}

Consider that the  data is collected by an RL agent with a fixed amount of exploration, such as the   $\epsilon$-greedy RL algorithms with a constant $\epsilon$. Then the uniform weighting with $H_i$ being the identity matrix satisfies Property \ref{property:weight_regularizing} with $\alpha_1=\alpha_2=0$, and Theorem \ref{thm:consistency} recovers the $n^{-1/2}$ convergence rate, consistent with results in \citep{lewis2020double} for i.i.d.~data. However, our result  is stronger in the sense that we allow for adaptivity in the experiments, as opposed to the i.i.d.~settings in  \citep{lewis2020double}.

\subsection{Consistency Weights for  Decaying Exploration under Bilinear Features}
\label{sec:consistency_application}

We now provide explicit weighting schemes for common RL algorithms with polynomially decaying exploration rates. 
We particularly focus on scenarios where the instrument $\psi$ adopts the feature map $\phi$ such that $\Psi_{i,j}=\Phi_{i,j}$, and the feature map $\phi$ can be expressed as a bilinear form of $\phi_j(S_{i,1:j}, T_{i,1:j}) = \chi_{j}(S_{i,1:j}, T_{i,1:j-1})  \otimes \mu(T_{i,j})$,
with $\otimes$ being the Kronecker product.\footnote{We remind that for two vectors $v\in R^p$ and $u\in R^d$, the Kronecker product is the vector whose entries contain the product of all pairs of entries of the two vectors, i.e., $v\otimes u = \text{vec}( uv^\top) = (v_1 u_1, \ldots, v_1 u_d, \ldots, v_p u_1, \ldots, v_p u_d)$.} Such  feature maps  accommodate a wide family of categorical treatments and continuous treatments (we defer the examples to the end of this section).  
Without loss of generality, we rename $\mu(T_{i,j})$, the transformation of treatment, as $T_{i,j}$ and focus on the below form throughout:
\begin{assumption}[Bilinear Feature Map]\label{ass:bilinear} The feature map $\phi_j$ of the blip function takes the form:
\begin{equation}
    \label{eq:feature_mapping_bilinear}
    \phi_j(S_{i,1:j}, T_{i,1:j}) =(T_{i,j} - \pi^*(S_{i,1:j}, T_{i,1:j-1}))\otimes \chi_{j}(S_{i,1:j}, T_{i,1:j-1}).
\end{equation}
\end{assumption}
Let $d_\tau$ be the dimension of treatment $T_{i,j}$, and we use $X_{i,j}$ as a shorthand for $\chi_j(S_{i,1:j}, T_{i,1:j-1})$. Thus we can also write that:\footnote{Here we leverage the property that for two vectors $u\in \bbR^p$ and $v\in \bbR^d$, $v\otimes u=(I_d \otimes u) \cdot v$, where $I_d \otimes u = \text{diagonal}(u, \ldots, u)$. Hence, our feature map can equivalently be written as $\Phi_{i,j} = (I_{d_{\tau}} \otimes \Psi_{i,j}) \cdot T_{i,j}$.}
\begin{align*}
    \Phi_{i,j} = 
    (T_{i,j} - \pi^*(S_{i,1:j}, T_{i,1:j-1}))\otimes  X_{i,j}  = (I_{d_{\tau}} \otimes X_{i,j}) \cdot (T_{i,j} - \pi^*(S_{i,1:j}, T_{i,1:j-1}))
\end{align*}

To identify the  structure parameter $\theta^*$, one should expect sufficient overlap condition in the behavior policy and  the state transition dynamics, due to  co-linearities in the linear system that identifies the structural parameters. The following assumption formalizes the overlap condition.

\begin{assumption}[Overlap]
\label{assump:overlap}
Let $c>0$ be some universal constant. For each unit $i$ at stage $j$, 
\begin{enumerate}[(a)]
    \item The behavior policy  satisfies that, 
    \begin{equation}
       \label{eq:require_observational_policy}
       \Var^+_{i,j}(T_{i,j})\succeq c_j\cdot i^{-\alpha}\cdot I_{d_\tau},
    \end{equation}
    for  constants $c_j>0$ and  $\alpha\in[0,1)$.  We refer to $\alpha$ as behavior exploration rate. 
    \item The state transition  satisfies that $\|X_{i,j}\|^2_2\in[c,c^{-1}]$ and $\bbE_{i,j}[X_{i,j}X_{i,j}^\top]\succeq c\cdot I$ almost surely.  
\end{enumerate}
\end{assumption}
Assumption \ref{assump:overlap}(a) imposes an explicit rate on the decay of the amount of randomization (exploration), which is employed by the behavior policy of the RL agent as a function of the number of observations it has collected so far. Similar assumption is also required in ex post inference when using bandit data~\citep{hadad2021confidence,zhan2021off}. In particular, when $\alpha=0$, this becomes a \emph{relaxed} version (by allowing for the dependence among observations) of the commonly made ``overlap'' condition in the causal inference literature on i.i.d.~samples \citep{imbens2004nonparametric}. 
Assumption \ref{assump:overlap}(b) says that the state transition dynamic, which is independent of the behavior policy, is non-degenerate. Specifically, conditional on the past states and actions $(S_{i,1:j-1}, T_{i,1:j-1})$, the next state feature $X_{i,j}$ is not collinear; it can be achieved if the exogenous noise injected at each stage has a full-rank covariance matrix.
With these,  we instantiate the weighting choices that satisfy Property \ref{property:weight_regularizing}, yielding  the  corollary of  Theorem \ref{thm:consistency}.

\begin{corollary}
\label{cor:consistency}
Suppose Assumptions~\ref{assump:exogeneity},~\ref{assump:linear_blip_function},~\ref{ass:bilinear},~\&~\ref{assump:overlap}  hold. 
\begin{enumerate}
\item  If $H_{i,j}$ is the identity matrix (i.e. no weighting), then Property \ref{property:weight_regularizing} is satisfied with $\alpha_1=2\alpha$, $\alpha_2=0$, and we have:
 \begin{equation*}
   \bbE\bb{\bn{\hat{\theta}_{n}-\theta^*}_2} =O\bp{n^{\frac{2(L-j+1)\alpha-1}{2}}
    } .
\end{equation*}
    \item  The weighting scheme $H_{i,j} = (I_{d_\tau}\otimes X_{i,j})\Var^+_{i,j}(T_{i,j})^{-1/2} 
    (I_{d_\tau}\otimes X_{i,j})^\top $ satisfies Property \ref{property:weight_regularizing} with $\alpha_1=\alpha$ and $\alpha_2=0$, and 
    we have:
 \begin{equation*}
   \bbE\bb{\bn{\hat{\theta}_{n}-\theta^*}_2} =O\bp{n^{\frac{(L-j+1)\alpha-1}{2}}
    }.
\end{equation*}
\end{enumerate}
\end{corollary}

\begin{remark}
Corollary~\ref{cor:consistency} highlights the potential improvements enabled by a weighted GMM estimator. The weighting scheme in Corollary~\ref{cor:consistency}(2) depends on $\Var^+_{i,j}(T_{i,j})$, which can be computed exactly given the known behavior policy. Incorporating these weights ensures consistency even under faster rates of exploration decay. For example, in bandit settings, uniform weighting restricts the exploration rate $\alpha$ (as defined in the lower bound on treatment variance in \eqref{eq:require_observational_policy}) to be at most $\frac{1}{2}$, whereas adaptive weighting allows $\alpha$ to reach $1$.
\end{remark}

We conclude this section by instantiating common feature maps and RL algorithms that satisfy \eqref{eq:feature_mapping_bilinear}.

\begin{example}[Categorical Treatment] Consider the action space of the RL agent contains one control action and $K$  treatment action. Omitting the episode index $i$ for simplicity, let $T_{j}=\zero$ denote the control action being taken, and $T_{j}=e_k$ denote the $k$-th treatment being taken, at the $j$-th period of unit $i$, where $e_k$ is the one-hot vector for category $k$. Without loss of generality, the blip function can be represented as 
\begin{align}
\label{eq:categorical_treatment}
     (\theta_j^*)^\top\phi_j(S_{1:j}, T_{1:j}) = \sum_{k=1}^K (\theta_{j,k}^*)^\top \chi_{j}(S_{1:j}, T_{1:j-1}) \one[T_{j}=e_k], 
\end{align}
where $(\theta_{j,k}^*)^\top \chi_{j}(S_{1:j}, T_{1:j-1})$ characterizes the heterogeneous treatment effect of the $k$-th treatment at stage~$j$. Note that \eqref{eq:categorical_treatment} is equivalent to  \eqref{eq:feature_mapping_bilinear} by expanding the Kronecker product in \eqref{eq:feature_mapping_bilinear}, i.e.
\begin{align*}
    (\theta_j^*)^\top\phi_j(S_{1:j}, T_{1:j}) = (\theta_{j}^*)^\top (
    T_{j}\otimes    \chi_{j}(S_{1:j}, T_{1:j-1})  ). 
\end{align*}
To meet the overlap condition on  behavior policy  in Assumption \ref{assump:overlap}, it suffices to run an $\epsilon$-greedy RL algorithm,  where $\epsilon$ decays over time. At each decision time $t$, the agent chooses the treatment that was best performing on historical data for that stage with  probability $1-\epsilon_t$ with $ \epsilon_t= t^{-\alpha}$; alternatively, the agent opts for a random treatment. 
\end{example}

\begin{example}[Continuous or Binary Treatment Vector] Consider the case when the treatment is a continuous (or binary) vector in $\bbR^{d_\tau}$, with the blip function represented as 
\begin{align}
\label{eq:continuous_treatment}
      (\theta_j^*)^\top \phi_j(S_{1:j}, T_{1:j}) = \sum_{k=1}^{d_\tau} (\theta_{j,k}^*)^\top \chi_{j}(S_{1:j}, T_{1:j-1}) T_{j,k}.
\end{align}
Each coordinate can be viewed as a separate continuous treatment applied to the unit, where different treatments can be applied simultaneously to each unit.
Here, $(\theta_{j,k}^*)^\top \chi_{j}(S_{1:j}, T_{1:j-1})$ characterizes the heterogeneous marginal treatment effect from the $k$-th coordinate of the treatment vector $T_{j}$, equivalently, the $k$-th treatment applied to the unit. Similarly, Equation~\eqref{eq:continuous_treatment} falls in the functional form prescribed in Equation~\eqref{eq:feature_mapping_bilinear}.

To satisfy the overlap of behavior policy required in Assumption \ref{assump:overlap}, it suffices to add an exogenous exploration noise $\varepsilon_t\in \bbR^{d_\tau}$ to  treatment assignment at time $t$, where $\varepsilon_t$ has covariance matrix $t^{-\alpha}I_{d_\tau}$. If each of the treatments are binary, then it suffices to independently randomize the assignment of each of the simultaneous binary treatments.
\end{example}

\begin{example}[Polynomial Scalar Treatment] Consider the case when a single scalar treatment $\tau_{i,j}$ can be applied at stage $j$ for unit $i$. We can express non-linear effects of the scalar treatment by considering a fixed expansion to a set of engineered treatment features $\mu(\tau_{i,j})$, as follows:
\begin{align}
\label{eq:polynomial_treatment}
      (\theta_j^*)^\top \phi_j(S_{1:j}, T_{1:j}) = \sum_{k=1}^{d_\tau} (\theta_{j,k}^*)^\top \chi_{j}(S_{1:j}, T_{1:j-1}) \mu_k(\tau_{i.j})
\end{align}
which is equivalent to \eqref{eq:feature_mapping_bilinear} by setting $T_{i,j}$ in \eqref{eq:feature_mapping_bilinear} to  $\mu(\tau_{i,j})$. For instance, $\mu(\tau_{i,j})$ could be chosen to be a high-degree polynomial, i.e. $\mu(\tau_{i,j}) = (\tau_{i,j}, \tau_{i,j}^2, \tau_{i,j}^3, \ldots, \tau_{i,j}^{d_\tau})$. In the context of a pricing application, one can view $\tau_{i,j}$ as the price offering for some product to a buyer $i$, and the outcome as revenue (i.e. whether there was a purchase times the purchase price). Aggregate revenue would typically be some bell-shaped curve, which can be well approximated by a third or fourth degree polynomial of price (a typical choice in empirical work). In such a revenue model, the quantities $(\theta_{j,k}^*)^\top \chi(S_{1:j}, T_{1:j-1})$ correspond to the heterogeneous coefficients in this parametric revenue model, reflecting heterogeneous price elasticities based on the current state of the buyer. 
\end{example}

\section{Asymptotic Normality of AW-GMM Estimation}
\label{sec:normality}
In this section, we identify a family of weighting schemes, under which the AW-GMM estimator is asymptotically normal.  These weights perfectly stabilize the variance of the empirical influence function in the asymptotic regime. Building on this, we prove strong Gaussian approximation results and characterize the uniform convergence rate. 
These results are practically appealing, as they enable the construction of uniformly valid confidence regions for parameters of interest over a large class of RL algorithms.

Recall that the AW-GMM estimator $\hat{\theta}_n$ minimizes the empirical loss $\calL_n(\theta)$ while the true parameter $\theta^*$ minimizes the expected loss $\calL(\theta)$:
\begin{equation*}
   \hat\theta_n\in\argmin_{\theta\in\Theta} \calL_n(\theta) = \bn{\frac{1}{n}\sum_{i=1}^n H_i\,   \bp{\beta_i+J_i\theta}}_A^2\quad \mbox{and}\quad \theta^*\in\argmin_{\theta\in\Theta}  \calL(\theta)=\bn{\frac{1}{n}\sum_{i=1}^n H_i\bp{ \bar{\beta}_i+\bar{J}_i\theta}}_A^2.
\end{equation*}
The key idea is to relate the estimation error $(\hat{\theta}_n - \theta^*)$ to a martingale difference sequence (MDS), which allows us to apply martingale theory to characterize the asymptotic distribution of $\hat{\theta}_n$.

As established in Section~\ref{sec:consistency}, the estimator $\hat{\theta}_n$ is consistent for $\theta^*$. Since $\theta^*$ lies in the interior of $\Theta$, $\hat{\theta}_n$ will also lie in the interior with high probability. In this case, $\hat{\theta}_n$ satisfies the first-order condition of the empirical loss $\calL_n(\theta)$:
\begin{equation}
    \nabla_\theta \calL_n(\hat\theta_n) =\bp{\frac{1}{n}\sum_{i=1}^n H_i J_i}^\top A\bp{\frac{1}{n}\sum_{i=1}^n H_i\, (\beta_i+J_i\hat\theta_n)}=\zero.
\end{equation}
Define  $\xi_i := \beta_i+J_i\theta^* $. Then after some algebra, we have:
\begin{align}
\label{eq:link_error_mds}
 -\bp{\frac{1}{n}\sum_{i=1}^n H_i J_i}^\top A\bp{\frac{1}{n}\sum_{i=1}^n H_i \xi_i}
  =& \bp{\frac{1}{n}\sum_{i=1}^n H_i J_i}^\top A\bp{\frac{1}{n}\sum_{i=1}^n H_iJ_i(\hat\theta_n-\theta^*)}.
\end{align}
Since $\theta^*$ satisfies the conditional moment equations $\bar{\beta}_i+\bar J_i\theta^*=\zero$, the sequence $\{H_i \xi_i\}$ forms an MDS under suitable conditions on the weights $H_i$, as formally established in Section~\ref{sec:mds}. By linking $\{H_i \xi_i\}$ to the estimation error $(\hat{\theta}_n - \theta^*)$, we derive a strong Gaussian approximation for $\hat{\theta}_n$ and develop corresponding inference results in Section~\ref{sec:strong_gaussian}. We then apply these results to common RL algorithms in Section~\ref{sec:normality_application}.

\subsection{Martingale from Adaptive RL Data}
\label{sec:mds}
We now formalize the martingale difference sequence  structure, which plays a central role in our analysis. 
The true parameter $\theta^*$ satisfies the conditional moment equations, so that $  \bar \beta_i + \bar J_i\theta^* =\zero$. Define
\begin{equation}
 \xi_i := \beta_i +J_i\theta^* = (\xi_{i,0}, \xi_{i,1}, \dots, \xi_{i,L}),\quad \mbox{with}\quad \xi_{i,j}=R_{i,j}(\Psi_{i,j} - \bar{\Psi}_{i,j}), \forall j\in[0:L],
\end{equation}
where  $R_{i,j}$, defined in Lemma~\ref{lemma:identification_parameter},
captures the residual component of $Y_i$ after removing the treatment effects from $T_{i,j:L}$ and adding the effects under the evaluation policy $\pi^*_{j:L}$. The sequence $\{\xi_i\}$ forms a martingale difference sequence, as established in the lemma below.
\begin{lemma}
\label{lemma:mds-1}
Suppose Assumptions~\ref{assump:exogeneity} and~\ref{assump:linear_blip_function} hold.
Let $H_i$ be a block-diagonal adaptive weighting matrix, where the $j$-th diagonal block  $H_{i,j}$ is measurable with respect to $\calF^+_{i,j}$.
Then the sequence $\{H_i\xi_i\}_{i=1}^n$ forms a martingale difference sequence adapted to the filtration $\{\calF_i\}_{i=1}^n$.
\end{lemma}
Equation~\eqref{eq:link_error_mds} shows that the estimation error $(\hat{\theta}_n-\theta^*)$ is closely linked to the sum $\sum_{i=1}^nH_i\xi_{i}$.
To analyze the asymptotic properties of  $\hat{\theta}_n$,  it thus suffices to understand the MDS, $\{H_i\xi_i\}$. The following result highlights the key components that characterize its asymptotic behavior.

\begin{proposition}[Martingale CLT,  \cite{hall2014martingale}]
\label{prop:martingale_clt_2}
Let $\{\phi_{t}, \mathcal{F}_{t}\}_{t=1}^\top$, with $\phi_t\in \mathbb{R}$, be a square-integrable scalar martingale difference sequence. Suppose that the  two conditions below are satisfied,
\begin{itemize}
    \item[(a)] conditional variance convergence: $\sum_{t=1}^\top  \bbE[\phi_{t}^2|\calF_{t-1}]\xrightarrow{p} \eta^2$ for some a.s. finite r.v. $\eta^2$;
    \item[(b)] conditional higher-moment decay: $\sum_{t=1}^\top \bbE[\phi_{t}^4|\calF_{t-1}]\xrightarrow{p}0$.
\end{itemize}
Then, $\sum_{t=1}^\top  \phi_{t}\Rightarrow Z$, where the random variable Z has characteristic function $\bbE[\exp(-\frac{1}{2}\eta^2t^2)]$.
\end{proposition}

Proposition~\ref{prop:martingale_clt_2} illuminates how  weights $H_i$ can be designed to improve the asymptotic properties of the MDS sum $\sum_{i=1}^n H_i\xi_i$, and thus  our AW-GMM estimation. A key quantity is the sum of conditional variances, $\sum_{i=1}^n \Var_{i}(H_i\xi_{i})$: when it converges to the identity matrix, this MDS sum is asymptotically normal. 

\subsubsection{Variance and Covariance under Homoscedasticity.}
We now take a closer look at the conditional variance  $\Var_{i}(H_i\xi_{i})$, which forms the basis for our subsequent weighting strategy. 
In full generality, one could design the weights $H_i$ to approximate $\Var_i(\xi_i)^{-1/2}$, so that $\Var_i(H_i \xi_i)$ is close to the identity matrix. However, this approach would require estimating $\Var_i(\xi_i)$, which in turn depends on the state transition dynamics and is thus an impractical task, particularly in high-dimensional continuous state spaces.

To address this challenge, we introduce a homoscedasticity assumption on the outcome process. Under this assumption, the conditional variance matrix $\Var_i(\xi_i)$ becomes block-diagonal: the covariances across different stages (corresponding to the off-diagonal blocks) are zero. As a result, it suffices to design weights conditioning on the realized state $S_{i,j}$ to  stabilize the stage-specific conditional variances of $\xi_{i,j}$, thereby eliminating the need to estimate the state transition dynamics.

\begin{assumption}[Homoscedasticity of residuals with respect to current treatment]
\label{assump:homoscedasticity}
The conditional variance of the residuals $\Var(R_{i,j}\mid T_{i,j}, \calF_{i,j}^+)$ is independent of $T_{i,j}$, i.e. $\Var(R_{i,j}\mid T_{i,j}, \calF_{i,j}^+)=\Var_{i,j}^+(R_j)$.
Moreover, for the residual $R_{i,j}$ defined in Lemma \ref{lemma:identification_parameter}, there exists universal constants $\sigma^2, M^2$ such that $R_{i,j}^2\leq M^2$ and $\Var_{i,j}^+(R_{i,j}) \geq \sigma^2$ almost surely.
\end{assumption}

As shown in Appendix~\ref{appendix:proof_mds}, Assumption~\ref{assump:homoscedasticity} is equivalent to $\bbE_{i,j}^+[R_{i,j}^2 (\Psi_{i,j} - \bar{\Psi}_{i,j})]=0$, while 
Lemma~\ref{lemma:identification_parameter} states that $\bbE_{i,j}^+[R_{i,j} (\Psi_{i,j}-\bar{\Psi}_{i,j})]=0$, reflecting that residuals are uncorrelated with the current treatment conditional on past information. 
 Assumption~\ref{assump:homoscedasticity} strengthens this by requiring the same uncorrelated property for the residual variance.  It is satisfied under the additive, rank-preserving version of the structural nested mean model \cite[Chapter 14.5]{miguel2023causal}, but is more general. 
 It implies that, conditional on the past, variance of $R_{i,j}$ is independent of future treatment assignments and is fully explained by exogenous factors.
 Assumption~\ref{assump:homoscedasticity} holds, for instance, if $R_{i,j}\, \ci\, T_{i,j}\mid \calF^+_{i,j}$, which is natural  if $R_{i,j}$ equals to the potential outcome $\y(T_{i,1:j-1}, \pi^*_{j:L})$ in distribution given the past.\footnote{This equality is  stronger  than Lemma~\ref{lemma:identification_policy_value}, which requires only equality in expectation.}
Given Assumption~\ref{assump:homoscedasticity}, we layout the block diagonal structure of the conditional variance $\Var_i(H_i\xi_i)$.
\begin{lemma}
\label{lemma:mds}
    Suppose Assumptions~\ref{assump:exogeneity}, \ref{assump:linear_blip_function}~\&~\ref{assump:homoscedasticity} hold. 
    Let $H_i$ be a block-diagonal adaptive weighting matrix, where the $j$-th diagonal block $H_{i,j}$ is measurable with respect to $\calF^+_{i,j}$.
   Then the conditional variance matrix $\Var_i(H_i\xi_i)=\bbE_i[H_i\xi_i\xi_i^\top H_i^\top]$ is block-diagonal, such that:
    \begin{enumerate}
        \item[(a)] its $(j_1, j_2)$ off-diagonal block
        $\Cov_i(H_{i,j_1}\xi_{i,j_1}, \xi_{i,j_2}H_{i,j_2}^\top)=0$,  for any $j_1\neq j_2\in[0:L]$;
        \item[(b)] its $j$-th diagonal block $ \Var_{i}(H_i\xi_{i,j}) =\bbE_i\bb{\bbE^+_{i,j}\bb{ R_{i,j}^2 } \cdot H_{i,j}\,\Var^+_{i,j}(\Psi_{i,j})\,
    H_{i,j}^\top}$, for any $j\in[0:L]$.
    \end{enumerate}
\end{lemma}
This lemma inspires us to decompose the weight $H_{i,j}$ into  two components: one to stabilize  $\Var_{i,j}^+(\Psi_{i,j})$ that is known exactly, as it depends on the known behavior policy, and another to stabilize $\bbE^+_{i,j}\bb{ R_{i,j}^2 }$ that can be estimated from data.

\subsection{Strong Gaussian Approximation}
\label{sec:strong_gaussian}
We now describe the construction of weights that ensure asymptotic normality of the estimator. This builds on the connection between the MDS $\{\xi_i\}$ and the estimation error $(\hat{\theta}_n - \theta^*)$, as established in Equation~\eqref{eq:link_error_mds}:
\begin{align}
\label{eq:link_mds_error_2}
 B_n^\top A\bp{\frac{1}{n}\sum_{i=1}^n H_i\xi_i}
  =& B_n^\top A B_n(\hat\theta_n-\theta^*),\quad \mbox{where}\quad B_n:=-\frac{1}{n}\sum_{i=1}^n H_i J_i.
\end{align}
To establish the asymptotic normality of $\hat{\theta}_n$, it suffices to construct weights such that $\sum_{i=1}^n H_i \xi_i$ is asymptotically normal. Proposition~\ref{prop:martingale_clt_2} indicates that this requires stabilizing the conditional variance $\Var_i(H_i \xi_i)$. Lemma~\ref{lemma:mds} further shows that $\Var_i(H_i \xi_i)$ is block-diagonal, so it is enough to stabilize each stage-specific variance $\Var_i(H_{i,j} \xi_{i,j})$ such that $\sum_{i=1}^n \Var_i(H_{i,j} \xi_{i,j}) \to I$ for each stage $j$.
Lemma \ref{lemma:mds} shows that:
\begin{align}
 \label{eq:residual_definition}
\Var_{i}(H_{i,j}\xi_{i,j})=\bbE_{i} \bb{H_{i,j}\,(F_{i,j} \cdot\Var^+_{i,j}(\Psi_{i,j}))\,
    H_{i,j}^\top},\quad\mbox{where}\quad  F_{i,j}:=\bbE^+_{i,j}[R_{i,j}^2].
\end{align}
We thus can split $H_{i,j}$ into two parts:
\[
H_{i,j} = \hat{F}_{i,j}^{-\frac{1}{2}}\cdot W_{i,j},
\]
where the first part $\hat{F}_{i,j}^{-\frac{1}{2}}$ standardizes  $F_{i,j}$ and is estimated using data available up to $\calF_{i,j}^+$. Section \ref{sec:estimate_f} provides further discussion on estimating $\hat{F}_{i,j}$; the second part $W_{i,j}$ is designed to satisfy a stabilizing property below that controls the time-varying variance of $\Var^+_{i,j}(\Psi_{i,j})$.

\begin{property}[$(\alpha_1, \alpha_3)$-Stabilizing Weights] 
\label{property:weight_stabilizing}
Set $W_{i,0}=1$ for all $i\in[1:n]$.
Given any $j\in[1:L]$, the weights $\{W_{i,j}\}_{i=1}^n$ are adapted to the filtration $\{\calF_{i,j}^+\}_{i=1}^n$ and satisfy the following for a universal constant $c_2>0$ independent of the behavior policy:
\begin{enumerate}[(a)]
    \item $\lambda_{\min}\bp{\bc{\frac{1}{n}\sum_{i=1}^n \bbE_{i,j}\bb{W_{i,j} \Cov^+_{i,j}(\Psi_{i,j},\Phi_{i,j})}}^\top \bc{\frac{1}{n}\sum_{i=1}^n \bbE_{i,j}\bb{W_{i,j} \Cov^+_{i,j}(\Psi_{i,j},\Phi_{i,j})}}
    }  \geq ~ c_2^2\cdot  n^{-\alpha_1}$;
    
    \item $\frac{1}{n}\sum_{i=1}^n\Tr\bp{\bbE\bb{  W_{i,j}^\top W_{i,j}}}  \leq c_2^{-1} \cdot n^{\alpha_1}$;
    
    \item
    $
        \Tr\bp{ W_{i,j}\Var^+_{i,j}(\Psi_{i,j})W_{i,j}^\top}\leq c_2^{-1},\\
              \frac{1}{n}
    \sum_{i=1}^n  \bbE
    \left\|\bbE_{i,j}[W_{i,j}\Var^+_{i,j}(\Psi_{i,j})W_{i,j}^\top] - \bbE[W_{i,j}\Var_{i,j}^+(\Psi_{i,j})W_{i,j}^\top]
    \right\|_{2} \leq c_2^{-1} \cdot n^{-\alpha_3},\\
    \lambda_{\min}\bp{\frac{1}{n} \bbE\bb{\sum_{i=1}^n W_{i,j}\Var^+_{i,j}(\Psi_{i,j})W_{i,j}^\top}} \geq c_2.
$
\end{enumerate}
for a universal positive constant $c_2$, with $\alpha_1\in[0,\frac{1}{L+1}), \alpha_3>0$. 
\end{property}
Property \ref{property:weight_stabilizing} looks
similar to Property \ref{property:weight_regularizing} but is stronger. In particular, to achieve asymptotic normality, Property \ref{property:weight_stabilizing}(c) further requires that after weight application, the resulted MDS variance is asymptotically standardized and stabilized around its expectation.

\begin{theorem}[Strong Gaussian Approximation]
\label{thm:be_feasible_2}
 Suppose Assumptions~\ref{assump:exogeneity}, \ref{assump:linear_blip_function}, and \ref{assump:homoscedasticity} hold.
 Suppose the sets $\calS,\calT,\calY$  are bounded. 
Suppose the weights $H_{i,j}$ decompose as $H_{i,j} = \hat{F}_{i,j}^{-1/2} W_{i,j}$, where $W_{i,j}$ satisfies Property~\ref{property:weight_stabilizing}. The term $\hat{F}_{i,j}$ is an estimator of $F_{i,j}$, which is defined in Equation~\eqref{eq:residual_definition} and is bounded in $[\sigma^2, M^2]$ under Assumption~\ref{assump:homoscedasticity}. $\hat{F}_{i,j}$ is constructed using information available up to $\calF^+_{i,j}$, and we assume that it  satisfies the following  condition:
\begin{equation}
\label{eq:f_convergence}
\frac{1}{n} \sum_{i=1}^n \mathbb{E}\left[ \left| \hat{F}_{i,j} - F_{i,j} \right| \right] = O(n^{-\gamma_F}) \mbox{ and }\hat{F}_{i,j} \in [\sigma^2, M^2].
\end{equation}
Let $\hat{\theta}_n$ be the AW-GMM estimator defined in \eqref{eq:gmm_estimator_linear}. Define $\widehat{\Xi}_n$ as a block-diagonal matrix with the $j$-th block equal to $n^{-1}\sum_{i=1}^n W_{i,j}\Var^+_{i,j}(\Psi_{i,j})W_{i,j}^\top$. Then there exists a Gaussian random variable $Z_\xi \sim \mathcal{N}(0,  \widehat{\Xi}_n)$ such that for any convex set $C \in \mathcal{C}$ (the set of convex subsets of $\Theta$),
\begin{align}
       \left| \bbP\bp{ \sqrt{n} (\hat{\theta}_n -\theta^*)\in C}-\bbP\bp{     (B_n^\top A B_n)^{\dagger}B_n^\top A Z_\xi\in C}\right| 
    = O  \Big( L^2\cdot n^{-\min\big(\frac{1-\alpha_1}{12}, \frac{1-(L+1)\alpha_1}{2},  \frac{\alpha_3}{5}, \frac{\gamma_F}{5}\big)}\Big),
\end{align}
where for any matrix $M$, we use $M^\dagger$ to denote its Moore–Penrose pseudo-inverse.

\end{theorem}

\begin{remark}
    Theorem~\ref{thm:be_feasible_2} not only establishes the asymptotic normality of $\hat{\theta}_n$, but also provides an upper bound on the uniform convergence rate of its strong Gaussian approximation. This result is stronger than the central limit theorem for $\hat{\theta}_n$, which can be achieved under a weaker condition: when $\hat{F}_{i,j}$  estimates $F_{i,j}$ consistently (the $L_1$ convergence  is not required as in \eqref{eq:f_convergence}), one can invoke Proposition~\ref{prop:martingale_clt_2} to show that,\footnote{We omit such a proof for conciseness and since it is a weaker result than our uniform convergence.}
\begin{equation}
    \label{eq:clt_feasible} 
 \sqrt{n} B_n^\top A B_n(\hat{\theta}_n-\theta^*) \Rightarrow \calN(0,B_n^\top A \widehat{\Xi}_nA B_n ).
\end{equation}
 However, the normalizing matrix $B_n^\top A B_n$ may not  necessarily converge, particularly in the adaptive settings where the behavior policy is evolving over time; we thus cannot rely on \eqref{eq:clt_feasible}  to construct confidence intervals for any arbitrary, 
 {data-independent,} projection of $\theta^*$ (for example, the  treatment effect $\theta^*_j$ at stage $j$), unless we use the uniform convergence results provided in Theorem \ref{thm:be_feasible_2}.
\end{remark}

\subsubsection{Post-RL Inference.}

Note that both $B_n$ and $\widehat{\Xi}_n$ are measurable from the data.  Theorem~\ref{thm:be_feasible_2} implies  the following Gaussian approximation, 
\begin{align}
\label{eq:strong_gaussian_approx}
\sqrt{n}  (\hat{\theta}_n-\theta^*) \approx \calN\bp{ 0,M_n
},\quad \mbox{where}\quad M_n:=(B_n^\top A B_n)^{\dagger}B_n^\top A\widehat{\Xi}_n
 A B_n(B_n^\top A B_n)^{\dagger}.
\end{align}
This approximation enables the construction of confidence regions,  formalized by the two corollaries below.

\begin{corollary}[Uniform Confidence Intervals]
\label{cor:uniform_ci}
    Suppose the conditions in Theorem \ref{thm:be_feasible_2} hold.
   For a confidence level $a\in (0, 1)$, for any projection $\ell\in\Theta$,  define the confidence interval as
    \[
CI_a=\bigg[\ell^\top \hat{\theta}_n\pm 
n^{-\frac{1}{2}} q_{\ell,\frac{a}{2}}
\bigg],
\]
where $q_{\ell,\frac{a}{2}}$ is the $\frac{a}{2}$-quantile of $ \calN\bp{0,
\ell^\top M_n \ell
}$. It holds that:
\[
\sup_{\ell\in\bbR^{1+dL}}\ba{\bbP(\ell^\top \theta\in CI_a) - (1-a)}=O\Big( L^2\cdot n^{-\min\bp{ \frac{1-\alpha_1}{12}, \frac{1-(L+1)\alpha_1}{2},\frac{\alpha_3}{5},\frac{\gamma_F}{5} }}\Big).
\]
\end{corollary}

\begin{corollary}[Simultaneous Confidence Band]
\label{cor:uniform_cb}
 Suppose the conditions in Theorem \ref{thm:be_feasible_2} hold. 
For a confidence level $a\in (0, 1)$, define the confidence band as
 \[
    CB_a = \bb{\hat{\theta}_n \pm n^{-\frac{1}{2}} q_{\infty,1-a}\cdot \hat{d}_n^{-\frac{1}{2}}},
\]
where $\hat{d}_n$ is the matrix vector that concatenates the diagonal blocks of $M_n$, and   $q_{\infty,1-a}$ is the $(1-a)$-quantile of $\|\calN(0,(\hat{D}_n^{\dagger})^{\frac{1}{2}}
M_n
(\hat{D}_n^{\dagger})^{\frac{1}{2}})\|_\infty$ for   $\hat{D}_n:=\diag\{\hat{d}_n\}$.\footnote{
{By $\|N(0,A)\|_{\infty}$, we denote the random variable that corresponds to the maximum absolute value of any entry in a vector that is distributed according to $N(0,A)$.}} It holds that,
\[
\left|\bbP(\theta\in CB_a)-( 1-a)\right|= O\Big( L^2\cdot n^{-\min\bp{ \frac{1-\alpha_1}{12}, \frac{1-(L+1)\alpha_1}{2}, \frac{\alpha_3}{5}, \frac{\gamma_F}{5} }}\Big).
\] 
\end{corollary}

\subsubsection{Estimating $F_{i,j}$.}
\label{sec:estimate_f}
We now provide a generic framework for estimating $F_{i,j}$ in a way that satisfies the condition in Equation~\eqref{eq:f_convergence}. Recall that $F_{i,j} := \mathbb{E}^+_{i,j}[R_{i,j}^2]$ is, by definition, a function of $(S_{i,1:j}, T_{i,1:j-1})$. If we further assume that this function is invariant to the behavior policy, then it can be consistently estimated using the realized data.
This condition holds, for instance, when the residual $R_{i,j}$ corresponds to the counterfactual outcome $Y(T_{i,1:j-1}, \pi^*_{j:L-1})$, which arises in our application to the Markovian model discussed in Section~\ref{sec:partial_linear_model}. In such cases, we denote this function by $f_j$:

\begin{assumption}\label{assump:invariant_fj}
The conditional expectation function $\mathbb{E}[R_{i,j}^2 \mid S_{i,1:j} = s_{1:j},\ T_{i,1:j-1} = \tau_{1:j-1}]$
is invariant to the episode index $i$ and is thus denoted by $f_j$ as a function of $(s_{1:j}, \tau_{1:j-1})$:
 \begin{equation}
\label{eq:define_f}
f_j(s_{1:j}, \tau_{1:j-1}):=\bbE[R^2_{i,j}\mid S_{i, 1:j}=s_{1:j}, T_{i,1:j-1}=\tau_{1:j-1}].    
\end{equation}
\end{assumption}

Let $\hat{f}_{i,j}$ be a predictor  of $f_j$, fitted on the  previous $i-1$ episodes and satisfying the estimation rate $i^{-\gamma_F}$ for $\gamma_F\in(0,1)$:
\begin{equation}
  \label{eq:estimate_f}
 \|\hat{f}_{i,j}- f_j\|_{1,\infty} = O(i^{-\gamma_F}), 
\end{equation}
where we introduce the norm $\|\cdot\|_{1,\infty}$ for any  $f$ defined on domain $\calS^{j_1}\times\calT^{j_2}\rightarrow \mathbb{R}$:
$\|f\|_{1,\infty} := \bbE_{s_1\sim P_S}\big[\sup_{(s_{2:j_1},\tau_{1:j_2})\in\calS^{j_1-1}\times \calT^{j_2}} |f(s_{1:j_1}, \tau_{1:j_2})|\big].$

Then
 we have Condition \eqref{eq:f_convergence} satisfied:
\begin{align}
 & \frac{1}{n}\sum_{i=1}^n\bbE[|\hat F_{i,j}-F_{i,j}|]=\frac{1}{n}\sum_{i=1}^n \bbE[|\hat{f}_{i,j}(S_{i,1:j}, T_{1,j-1})- f_j(S_{i,1:j}, T_{1,j-1})|] \leq \frac{1}{n}\sum_{i=1}^n  \|\hat{f}_{i,j}- f_j\|_{1,\infty}
 \nonumber\\
 &= O\bp{ \frac{1}{n}\sum_{i=1}^n i^{-\gamma_F}} 
\leq  O\bp{ \Big(\frac{1}{n}\sum_{i=1}^{n} i^{-1} \Big)^{\gamma_F}} =
 O((\log n/n)^{\gamma_F}),\label{eq:f_cumulative_convergence}
\end{align}
where the last inequality is due to that $x^{-\gamma_F}$ is a concave function with $\gamma_F\in(0,1)$.

Estimating $f_j$ to achieve the rate in \eqref{eq:estimate_f} can be addressed by online learning algorithms \citep{rakhlin2014online,daskalakis2022fast}, and in particular \citeauthor{daskalakis2022fast} propose a learning algorithm to achieve fast rates of convergence in nonparametric online regression. Later in Section~\ref{sec:partial_linear_model} we  show how to estimate $f_{j}$ for high-dimensional Markovian models via Lasso (simpler than proposed   algorithm in \cite{daskalakis2022fast}) and establish the guaranteed convergence rate   in \eqref{eq:estimate_f}.

\subsection{Normality Weights for  Decaying Exploration under Bilinear Features}
\label{sec:normality_application}
We now instantiate  weight construction on asymptotic normality for common RL algorithms. 
Similar to Section \ref{sec:consistency_application}, we adopt features as instruments and consider the bilinear feature map in Assumption~\ref{ass:bilinear}:
\begin{equation*}
    \Phi_{i,j} = (T_{i,j} - \pi^*(S_{i,1:j}, T_{i,1:j-1})) \otimes X_{i,j} 
\end{equation*}
where recall that  $X_{i,j}$ serves as a shorthand for $\chi_j(S_{i,1:j}, T_{i,1:j-1})$.
Under the overlap condition specified in Assumption \ref{assump:overlap}, we provide weighting choices in the following corollary to achieve strong Gaussian approximation in Theorem \ref{thm:be_feasible_2}.

\begin{corollary}
\label{cor:normality}
 Suppose Assumptions~\ref{assump:exogeneity},~\ref{assump:linear_blip_function},~\ref{ass:bilinear},~\ref{assump:overlap},~\&~\ref{assump:homoscedasticity} hold.
  Define  the random matrix: 
    \begin{equation}
    V_{i,j}:=\bbE_{i,j}[X_{i,j}X_{i,j}^\top]. 
    \end{equation}
 Let weights $H_{i,j} := \hat{F}_{i,j}^{-1/2}W_{i,j}$, where $\hat{F}_{i,j}$ satisfies Condition \eqref{eq:f_convergence} and 
 \begin{equation}
 \label{eq:w_normality}
      W_{i,j}:=(I_{d_\tau}\otimes\hat{V}_{i,j}^{-1/2}) (I_{d_\tau}\otimes X_{i,j})\Var^+_{i,j}(T_{i,j})^{-1/2} 
    (I_{d_\tau}\otimes X_{i,j})^\top\cdot\|X_{i,j}\|_2^{-2},
 \end{equation}
 where $\|X_{i,j}\|_2$ denotes the $\ell_2$ norm of the vector $X_{i,j}$;
 $ \hat{V}_{i,j}$ is positive definite adapted to $\calF_{i,j}$, satisfies
 $c^{-1}\cdot I \succeq\hat{V}_{i,j}\succeq c\cdot I$ and approximates $V_{i,j}$ (where $c$ and $V_{i,j}$ are introduced  in Assumption \ref{assump:overlap}) with
 \begin{equation}
       \frac{1}{n}
    \sum_{i=1}^n \bbE\bb{\left\|\hat{V}_{i,j}- V_{i,j}
    \right\|_{2}} = O(n^{-\gamma_V}).
    \label{eq:v_convergence}
 \end{equation}
Then, this $W_{i,j}$ satisfies Property \ref{property:weight_stabilizing} with $\alpha_1=\alpha$ and $\alpha_3=\gamma_V$, and thus we have the uniform Gaussian approximation rate in Theorem \ref{thm:be_feasible_2} to be $O(L^{2}n^{-\min(\frac{1-\alpha}{12},\frac{1-(L+1)\alpha}{2},\frac{\gamma_F}{5},\frac{\gamma_V}{5})})$.
\end{corollary}

The  choice of $W_{i,j}$ in \eqref{eq:w_normality} requires estimating $V_{i,j}=\bbE_{i,j}[X_{i,j}X_{i,j}^\top]$. 
When $j=1$, we have $V_{i,1}=\bbE\bb{X_{i,1}X_{i,1}^\top}$, with $\{X_{i,1}\}$ being i.i.d.; thus we can estimate $V_{i,1}$   consistently via the sample-mean estimator and achieve $O(n^{-1/2})$ estimation rate. We hence focus on estimating $V_{i,j}$ for stage $j\geq 2$,  where $V_{i,j}$ involves the  expectation over $S_{i,j}$ conditional on  $(S_{i,1:j-1}, T_{i,1:j-1})$. The  randomness in $S_{i,j}$ arises from the  state transition process, which is independent of the behavior policy. As a result, we can express $V_{i,j}$ as the output  of a function $\nu_j(\cdot)$, with input $(S_{i,1:j-1}, T_{i,1:j-1})$. This function $\nu_j$  depends only  on the stage index $j$ but not on the episode index $i$, i.e.,
 \begin{equation}
\label{eq:define_v}
\nu_j(s_{1:j-1}, \tau_{1:j-1}):=\bbE[X_{i,j}X_{i,j}^\top\mid S_{i, 1:j-1}=s_{1:j-1}, T_{i,1:j-1}=\tau_{1:j-1}]. 
\end{equation}

Estimating $\nu_j$  can be similarly addressed by the online learning literature as we have discussed for estimating $f_j$ in Section \ref{sec:estimate_f}, which shall be instantiated with Lasso estimator for the high-dimensional Markovian models in the next section. Below we provide guarantees for generic estimators.

Let $\hat{\nu}_{i,j}$ be a generic predictor  of $\nu_j$, fitted on the  previous $i-1$ episodes and satisfying the estimation rate $i^{-\gamma_V}$, with $\gamma_V\in(0,1)$:
\begin{equation}
  \label{eq:estimate_v}
\|\hat{\nu}_{i,j}- \nu_j\|_{1,\infty} = O(i^{-\gamma_V}).
\end{equation}
Above, we reload the norm notation $\|\cdot\|_{1,\infty}$ for any  $v$ defined on domain $\calS^{j_1}\times\calT^{j_2}\rightarrow \mathbb{R}^{d_1\times d_2}$:
$\|v\|_{1,\infty} := \bbE_{s_1\sim P_S}\big[\sup_{(s_{2:j_1},\tau_{1:j_2})\in\calS^{j_1-1}\times \calT^{j_2}} \|v(s_{1:j_1}, \tau_{1:j_2})\|_{2}\big]$, where $\|\cdot\|_{2}$ denotes the matrix operator norm induced by vector $2$-norm. Then we  have Condition \eqref{eq:v_convergence} satisfied:
\begin{align}
\label{eq:nu_cumulative_convergence}
 &\frac{1}{n}
    \sum_{i=1}^n \bbE\bb{\left\|\hat{V}_{i,j}- V_{i,j}
    \right\|_{2}}=\frac{1}{n}\sum_{i=1}^n \bbE[\|\hat{\nu}_{i,j}(S_{i,1:j-1}, T_{1,j-1})- \nu_j(S_{i,1:j-1}, T_{1,j-1})\|_2] \nonumber\\
  &  \leq \frac{1}{n}\sum_{i=1}^n \|\hat{\nu}_{i,j}- \nu_j\|_{1,\infty}=
 O((\log n/n)^{\gamma_V}).
\end{align}

To this end, we provide a full theory with the generic nuisance estimators.
\begin{corollary}  
\label{cor:full_normality}
Suppose Assumptions~\ref{assump:exogeneity},~\ref{assump:linear_blip_function},~\ref{ass:bilinear},~\ref{assump:overlap},~\ref{assump:homoscedasticity},~\&~\ref{assump:invariant_fj} hold.  
Define functions $f_j$ and $\nu_j$ as in  Equations~\eqref{eq:define_f} and \eqref{eq:define_v} respectively.
Suppose $\hat{f}_{i,j}$ and $\hat{\nu}_{i,j}$ are fitted using the first $i-1$ episodes and satisfy
 $ \|\hat{f}_{i,j}- f_j\|_{1,\infty} = O(i^{-\gamma_F})$ and $\|\hat{\nu}_{i,j}- \nu_j\|_{1,\infty} = O(i^{-\gamma_V})$ respectively.
 Set the weights $H_{i,j}$ as in Corollary \ref{cor:normality}, where we use $\hat{F}_{i,j}:=\hat{f}_{i,j}(S_{i,1:j}, T_{i,1:j-1})$ and $\hat{V}_{i,j}:=\hat{\nu}_{i,j}(S_{i,1:j-1}, T_{i,1:j})$. Then the uniform Gaussian approximation rate in Theorem \ref{thm:be_feasible_2} holds with rate $\widetilde O(L^{2}n^{-\min(\frac{1-\alpha}{12},\frac{1-(L+1)\alpha}{2},\frac{\gamma_F}{5},\frac{\gamma_V}{5})})$, where we use $\widetilde O(\cdot)$ to omit logarithm terms.
\end{corollary}
This result is a direct application of Corollary \ref{cor:normality}, Equations \eqref{eq:f_cumulative_convergence} and \eqref{eq:nu_cumulative_convergence}; we omit its proof for brevity.

\section{Application to High-dimensional Markovian Models}

\label{sec:partial_linear_model}

In this section, we instantiate the generic inference framework we present in Section \ref{sec:normality} to high-dimensional Markovian models. In particular, we consider the  baseline policy of ``no treatment'' as the evaluation policy, i.e.~$0_{1:L}$, where the treatment $0_j$ represents a pre-existing status quo treatment that would have been applied at stage $j$ in the absence of the RL experimentation process.
Estimating the value of a no-treatment policy is highly valuable in practical RL scenarios. It enables verification, with statistical confidence, that the deployed RL policy resulted in a statistically significant higher reward compared to consistently applying the status quo treatment at each stage.
This allows to test an alternative policy, which introduces a new innovation, product feature, or treatment.

We consider high-dimensional Markovian models that satisfy the high-level assumptions on the data-generating process (DGP) introduced in previous sections. Our focus is on a specific instantiation of the estimation framework in which the features $\Phi_{i,j}$ are used as instruments $\Psi_{i,j}$, and the norm matrix   $A$ in the GMM estimation is set to the identity. While our analysis centers on this setting, extensions to more general cases are straightforward. We introduce concrete statistical estimation algorithms for weight construction and provide the corresponding strong Gaussian approximation rates. Finally, we conduct simulations to empirically validate the inferential ability and robustness of our method in this setting.

\subsection{Data Generating Process}
We start by formalizing the DGP for high-dimensional Markovian models.
Let $\calS, \calT$ be the bounded state and treatment spaces. 
The RL agent sequentially rolls out one episode per unit. Each episode $i$ consists of a length-$L$ trajectory of high-dimensional states, where only a sparse subset of state components influences the final outcome. Specifically, the blip functions depend only on low-dimensional sub-vectors of the state. To facilitate the weight construction discussed in Sections~\ref{sec:consistency_application} and~\ref{sec:normality_application}, we assume the blip function at each stage $j$ takes a bilinear form and is written as $T_{i,j}\otimes\chi(S_{i,j,\Omega})$ for Lipschitz and bounded $\chi(\cdot)$, where $\Omega$ denotes the set of relevant state coordinates. We use $d_\Omega$ and $d_s$ to denote the dimensions of $S_{i,j,\Omega}$ and $S_{i,j}$, respectively.

For each episode $i$, the behavior policy $\pi^{\text{obs}}_i$ is decided based  on the previous $i-1$ episodes. This policy $\pi^{\text{obs}}_i$ is assumed to be known and assigns the treatment via $T_{i,j} := \pi^{\text{obs}}_i(S_{i,j}, \zeta_{i,j})$, where $\zeta_{i,j}$ is an i.i.d.~bounded noise term.
We consider the behavior policy satisfies Assumption \ref{assump:overlap}(a) such that:
\begin{equation}
      \Var^+_{i,j}(T_{i,j}) \succsim i^{-\alpha}\cdot I_{d_\tau},
      \label{eq:plmm_behavior_policy_rate}
\end{equation}
for an exploration rate $\alpha$.
State transitions are governed by unknown matrices $(A_j, B_j, M_j)$, which capture the effects of the current treatment, the current state, and the initial state, respectively, on the next state. Let $\eta_{i,j}$ denote the i.i.d.~mean-zero bounded noise associated with the state transitions.
The final outcome $Y_i$ is modeled as a sparse linear function of the last treatment, the last state, and the initial state, perturbed by an i.i.d.~bounded noise term $\epsilon_i$.
 Model \ref{algo:plmdgp} summarizes the DGP.

\begin{model}
\DontPrintSemicolon 
\KwIn{Time horizon $L$; behavior policy $\pi^{\text{obs}}_i$}
\KwOut{Episodic data $(S_{i,1}, T_{i,1}, \ldots, S_{i,L}, T_{i,L}, Y_i)$}
Observe $S_{i,1} \stackrel{\text{i.i.d.}}{\sim} P_S$

\For{$j=\{1,\dots, L-1\}$}{
    Assign $T_{i,j} \gets \pi^{\text{obs}}_i(S_{i,j}, \zeta_{i,j})$, for $\zeta_{i,j} \stackrel{\text{i.i.d.}}{\sim} P_{\zeta,j}$
  
  Observe $S_{i,j+1} \gets A_{j+1}^\top( T_j\otimes\chi(S_{i,j, \Omega}))  + B_{j+1}^\top S_{i,j} +M_{j+1}^\top S_{i,1} + \eta_{i,j+1}$, for $\eta_{i,j+1}  \stackrel{\text{i.i.d.}}{\sim} P_{\eta,j+1}$}

Assign $T_{i,L} \gets \pi^{\text{obs}}_i({S}_{i,L}, \zeta_{i,L})$,  for $\zeta_{i,L}  \stackrel{\text{i.i.d.}}{\sim} P_{\zeta, L}$

Observe $\y_i \gets \alpha^\top (T_{i,L}\otimes\chi(S_{i,L,\Omega}))) + \beta^\top  S_{i,L} + \kappa_L^\top S_{i,1} + \epsilon_i$, for  $\epsilon_i  \stackrel{\text{i.i.d.}}{\sim} P_{\epsilon}$

\caption{{\sc PLM DGP}: high-dimensional Markovian Data Generating Process}
\label{algo:plmdgp}
\end{model}

\begin{lemma}
\label{lemma:partial_linear_model}
Consider the no-treatment $0_{1:L}$ policy as the evaluation policy. The high-dimensional Markovian model specified in Model \ref{algo:plmdgp} satisfies Assumptions \ref{assump:exogeneity}, \ref{assump:linear_blip_function},   \ref{ass:bilinear}, \& \ref{assump:homoscedasticity} .
\end{lemma}
Given the  DGP outlined in Model \ref{algo:plmdgp} ,  we can  expand the outcome $\y$ recursively: 
\begin{align}
\label{eq:final_unroll_plm}
    &\quad \quad \y_i = \sum_{j'=j}^L\theta_{j'}^\top (  T_{i,j'}\otimes\chi(S_{i,j',\Omega})) + \beta_j^\top S_{i,j} + \kappa_j^\top S_{i,1} + 
    \epsilon_{i,j}\\
    \mbox{where}\quad & \theta_L^*:=\alpha,\quad \beta_L:=\beta; \quad \mbox{For} ~j=L,\cdots,2: \nonumber\\
   & \quad \beta_{j-1} := B_{j}\beta_j,\quad \theta_{j-1}^*:= A_{j}\beta_j, \quad\kappa_{j-1} := \kappa_j + M_{j}\beta_j, \quad \epsilon_{i,j}=  \sum_{j'=j}^{L-1}\beta_{j'+1}^\top \eta_{i,j'} + \epsilon_i.\nonumber
\end{align}
Our goal is to estimate the structural parameters $\{\theta_0^*, \theta_1^*,\dots, \theta_L^*\}$, where 
$\theta_0^*:=\bbE[Y(0_{1:L})]=\beta_1^\top\bbE[S_{i,1}] + \kappa_1^\top\bbE[ S_{i,1}]$ represents the expected outcome under the baseline  policy,
and $\theta_{1:L}^*$ reflect the stage-wise dynamic treatment effects in  Model \ref{algo:plmdgp}  outlined in \eqref{eq:final_unroll_plm}.

\subsection{Constructing the Weights}
\label{sec:estimation_plm}

We now provide details on the  weight construction for  Model \ref{algo:plmdgp}.
With the bilinear blip functions, 
 Corollary~\ref{cor:consistency} provides analytical weight construction to achieve consistency.
To achieve the strong Gaussian approximation, Corollary~\ref{cor:full_normality} shows that we need to estimate functions $f_j$ and
$\nu_j$
to satisfy Conditions \eqref{eq:estimate_f}  \& \eqref{eq:estimate_v} respectively. Recall that in general settings,  $f_{j}$ and $\nu_{j}$ can be estimated  by  online learning algorithms, while here under the   high-dimensional Markovian models,   this estimation procedure can be greatly simplified.

\subsubsection{Estimating $f_{j}$.} 
We first show that Model \ref{algo:plmdgp} satisfies the below homoscedasticity property, which satisfies Assumptions \ref{assump:homoscedasticity} \& \ref{assump:invariant_fj}  and makes estimating $f_{j}$  a regular estimation problem.
 \begin{lemma}[Homoscedasticity]
\label{lemma:homoscedasticity_plm}
Consider the no-treatment $0_{1:L}$ policy as the evaluation policy.  For  Model~\ref{algo:plmdgp}, 
the residual $R_{i,j}$, when conditioning on the $(S_{i,1:j}, T_{1:j-1})$,  is independent from  the behavior policy and regardless of the realization of $(S_{i,1:j}, T_{1:j-1})$, has the same variance, denoted as $\sigma_j^2$.
\end{lemma}
With this property, we can decompose $F_{i,j}$ into two parts: 
\begin{equation}
    \label{eq:decomposition_of_f}
    F_{i,j}:=\bbE^+_{i,j}\bb{R_{i,j}^2}=  G_{i,j}^2 + \sigma_j^2, \quad\mbox{with}\quad G_{i,j}:=\bbE^+_{i,j}\bb{R_{i,j}}\quad \mbox{and}\quad \sigma_j^2=\sigma_{i,j}^2:=\Var^+_{i,j}(R_{i,j}).
\end{equation}
Note that $G_{i,j}$, when conditioning on $(S_{i, 1:j}, T_{i,1:j-1})$, does not  depend   on the behavior policy. Thus we can view $G_{i,j}$ as a function $g_j$ of $(S_{i,1:j}, T_{i,1:j-1})$, which only depends on the stage index $j$ but not the specific episode index $i$, that is,
 \begin{equation}
\label{eq:g_hdmm}
 g_j(S_{i,1:j}, T_{i,1:j-1}):=G_{i,j} = \bbE[R_{i,j}\mid S_{i, 1:j}, T_{i,1:j-1}]=  S_{i,j}^\top\beta_j  +  S_{i,1}^\top \kappa_j,
\end{equation}
where the last equation is due to that $R_{i,j}=\beta_j^\top S_{i,j} + \kappa_j^\top S_{i,1} + \epsilon_{i,j}$ by Equation~\eqref{eq:final_unroll_plm}.
We  can estimate $g_j$ by regressing $\{R_{i,j}\}$ on $\{(S_{i,1}, S_{i,j})\}$ to get $\hat{g}_{i,j}$;   we then estimate 
$\sigma_{j}^2$ by its sample variance $\hat\sigma_{i,j}^2$ up till episode~$i$; 
finally we set $\hat{f}_{i,j}(\cdot):=\Clip_{[\sigma^2, M^2]}\bp{\hat{g}_{i,j}(\cdot)^2 + \hat\sigma_{i,j}^2}$, where $(\sigma^2,M^2)$ denotes the specified range for $F_{i,j}$.
Here, for any real numbers $a < b$ and any $x \in \mathbb{R}$, the clipping operator is defined as:
\begin{equation}
    \label{eq:clip_scalar}
     \text{Clip}_{[a,b]}\bp{x}:=\max(\min(x,b), a).
\end{equation}
 Algorithm \ref{algo:f} summarizes the estimation process. Note that  
we cannot observe the true residual $R_{i,j} =Y_{i}-\sum_{j'>j}\Phi_{i,j'}^\top\theta^*_j $, which requires the knowledge of the true structure parameter $\theta^*$. However, we can  estimate it using an approximated $\hat{\theta}$ (for example, the consistent estimate from Section \ref{sec:consistency}). With that, we obtain  the estimation rate of $\hat{f}_{i,j}$ as $\|\hat f_{i,j}-f_j\|_{1,\infty}=O(i^{-\frac{1-L\alpha}{2}})$; see details in  Appendix~\ref{appendix:proof_hdmm_gaussian}.

\setcounter{algocf}{0}
\begin{algorithm}
\KwIn{The first $i-1$ episodes $\{(S_{i',1}, T_{i',1}, \ldots, S_{i',L}, T_{i',L}, Y_{i'})\}_{i'=1}^{i-1}$; stage index $j$; estimated $\hat{\theta}_{1:L}$ ;   specified range $[\sigma^2, M^2]$ for $F_{i,j}$.}
\KwOut{Estimated $\hat{f}_{i,j}$.}

Set approximated residual as $\hat{R}_{i',j}=Y_{i'} - \sum_{j'=j}^L(T_{i',j'}\otimes \chi(S_{i',j',\Omega}))^\top\hat{\theta}_{j'}$ for $i'\in [1:i-1]$.

Estimate $g_{j}$, defined in \eqref{eq:g_hdmm}, by regressing $\{\hat{R}_{i',j}\}_{i'=1}^{i-1}$ on $\{(S_{i',1}, S_{i',j})\}_{i'=1}^{i-1}$ using Lasso regression:
\begin{equation}
         (\hat{\beta}_{i,j}, \hat{\kappa}_{i,j})= \argmin_{\tilde{\beta}, \tilde{\kappa}} \frac{1}{i-1}\sum_{i'=1}^{i-1} \bp{\hat{R}_{i',j}-\tilde{\beta}^\top S_{i',j} -  \tilde{\kappa}^\top S_{i',1}}^2 + \lambda_g (\|\tilde{\beta}\|_1+ \|\tilde{\kappa}\|_1)\label{eq:lasso_g}.
\end{equation}

Set the estimate as $\hat{g}_{i,j}(s_1, s_j)=s_j^\top \hat{\beta}_{i,j} + s_1^\top  \hat{\kappa}_{i,j}$.

Calculate $
    \hat{\sigma}_{i,j}^2 := \frac{1}{i-1}\sum_{i'=1}^{i-1}\bp{ \hat{R}_{i',j}- \hat{g}_{i,j}(S_{i',1}, S_{i',j}  )}^2.
$

Set $\hat{f}_{i,j}(\cdot)=\text{Clip}_{[\sigma^2, M^2]}\bp{
\hat{g}_{i,j}(\cdot)^2+\hat{\sigma}_{i,j}^2}$ with clipping operator defined in \eqref{eq:clip_scalar}.

\caption{Estimating $f_j$ under Model \ref{algo:plmdgp}}
\label{algo:f}
\end{algorithm}

\subsubsection{Estimating $\nu_{j}$.} We now estimate 
\[\nu_j(s_{1:j-1},\tau_{1:j-1}):=\bbE\bb{\chi(S_{i,j,\Omega})\chi(S_{i,j,\Omega})^\top\mid S_{i,1:j-1}=s_{1:j-1}, T_{i,1:j-1}=\tau_{1:j-1}}.\]
Note that when $j=1$, we have $\nu_{1}\equiv\bbE\bb{\chi(S_{i,1,\Omega})\chi(S_{i,1,\Omega})^\top}$, with $\{S_{i,1}\}$ being i.i.d.; thus we can estimate $\nu_{1}$   consistently via the sample-mean estimator and achieve $O(n^{-1/2})$ estimation rate. We hence focus on estimation of $\nu_{j}$ for stage $j\geq 2$.

Let $h_{j}(S_{i,1:j-1}, T_{i,1:j-1})$ denote the conditional expectation of $S_{i,j}$ when restricted to coordinates $\Omega$~(since only those coordinates enter into $\psi_j$ and contribute to $V_j$):
\begin{align}
& h_{j}(S_{i,1:j-1}, T_{i,1:j-1}):=\bbE_{i,j}[S_{i,j,\Omega}]=A_{j,\Omega}^\top\bp{T_{i,j-1}\otimes\chi(S_{i,j-1, \Omega})} +B_{j-1,\Omega}^\top S_{i,j-1} + M_{j,\Omega}^\top S_{i,1},\label{eq:h_hdmm}
\end{align}
Then $\nu_j$ can be induced by $h_j$: 
$
\nu_j(\cdot)
=~ \bbE_{i,j}\bb{\chi\bp{h_{j}(\cdot) + \eta_{i,j,\Omega}}\chi\bp{h_{j}(\cdot) + \eta_{i,j,\Omega}}^\top},
$
for i.i.d.~exogenous noise  $\eta_{i,j,\Omega}$ during the state transition.
We therefore first  regress $\{S_{i,j,\Omega}\}$ on $\{S_{i,1}, S_{i,j-1},T_{i,j-1}\}$ to get estimate $\hat{h}_{i,j}$, based on which we  obtain estimate $\hat \nu_{i,j}$, as summarized in  Algorithm \ref{algo:v}.
Note that in Algorithm~\ref{algo:v}, we also apply the clipping operator when obtaining $\hat{\nu}_{i,j}$ to ensure its eigenvalues lie within the specified range for the true matrix $V_{i,j}$. 
Here, For any real numbers $a < b$ and a symmetric matrix $X = U \Lambda U^\top \in \mathbb{R}^{d \times d}$, where $U$ is an orthogonal matrix and $\Lambda = \text{diag}(\lambda_1, \dots, \lambda_d)$ is the diagonal matrix of eigenvalues, the clipping operator is defined as:
\begin{equation}
\label{eq:clip_matrix}
     \text{Clip}_{[a,b]}(X):= U\diag\{
  \text{Clip}_{[a,b]}(\lambda_1),\dots, 
    \text{Clip}_{[a,b]}(\lambda_d)U^\top
 \},
\end{equation}
 where recall $\text{Clip}_{[a,b]}(\lambda) := \max(\min(\lambda, b), a)$ for any scalar $\lambda \in \mathbb{R}$ as defined \eqref{eq:clip_scalar}.
The estimation rate of $\hat{\nu}_{i,j}$ satisfies $\|\hat\nu_{i,j}-\nu_j\|_{1,\infty}=O(i^{\frac{-\gamma(1-2\alpha)}{2}})$ (proof deferred to Appendix \ref{appendix:proof_hdmm_gaussian}).

\begin{algorithm}
\KwIn{Data $\{S_{i',1}, T_{i',1}, \ldots, S_{i',L}, T_{i',L}, Y_{i'}\}_{i'=1}^i$; stage index $j$; specified range $[c,c^{-1}]$ for eigenvalues of  $V_{i,j}$.}

\KwOut{Estimated $\hat{\nu}_{i,j}$.}

Estimate $h_{j}$, defined in \eqref{eq:h_hdmm}, by regressing  $\{S_{i',j, \Omega}\}_{i'=1}^i$ on $\{(S_{i',1}, S_{i',j-1}, T_{i',j-1})\}_{i'=1}^{i-1}$ using Lasso regression:
\begin{align}
(\hat{A}_{i,j,k}, \hat{B}_{i,j,k}, \hat{M}_{i,j,k}) &= \argmin_{(\tilde{A}_k, \tilde{B}_k, \tilde{M}_k)} \frac{1}{i-1}\sum_{i'=1}^{i-1} \bp{S_{i',j,k}- \tilde{A}_k^\top \Phi_{i',j-1}- \tilde{B}_k^\top S_{i',j-1} - \tilde{M}_k^\top S_{i',1})}^2\nonumber \\
&\quad\quad\quad\quad\quad\quad\quad\quad + \lambda_k ( \|\tilde{A}_k\|_1 +\|\tilde{B}_k\|_1 +\|\tilde{M}_k\|_1  ), \quad\quad \mbox{for}\quad k\in\Omega.\label{eq:lasso_h}
\end{align}
Set the estimate  $\hat{h}_{i,j}(s_1, s_{j-1},\tau_{j-1})=\hat{A}_{i,j,\Omega}^\top\bp{\tau_{j-1}\otimes\chi(s_{j-1, \Omega})} +\hat B_{i,j,\Omega}^\top s_{j-1} +\hat{M}_{i,j,\Omega}^\top s_1$.

Define $\hat\nu_{i,j}(\cdot)=\Clip_{[c,c^{-1}]}\big(\frac{1}{i-1}\sum_{i'=1}^{i-1}  \chi\big(\hat{h}_{i,j}(\cdot) + S_{i',j,\Omega}- \hat{h}_{i,j}(S_{i',1:j-1},  T_{i,1:j-1})\big)\chi\big(\hat{h}_{i,j}(\cdot) + S_{i',j,\Omega}- \hat{h}_{i,j}(S_{i',1:j-1},  T_{i,1:j-1})\big)^\top\big)$,  with clipping operator defined in \eqref{eq:clip_matrix}.
\caption{Estimating $\nu_{j}$ under  Model \ref{algo:plmdgp}}
\label{algo:v}
\end{algorithm}

\subsubsection{Putting everything together.} 
\label{sec:combine_all_plmm}
With the estimated $\hat{f}_{i,j}$ and $\hat{\nu}_{i,j}$ in the previous sections, we are ready to construct weights to achieve strong Gaussian approximation in Corollary \ref{cor:full_normality}. We summarize the full steps in 
Algorithm \ref{algo:plmdgp_weights} and provide the  strong Gaussian approximation rate  in Corollary \ref{cor:plmm_normality_rate}.

\begin{algorithm}
\DontPrintSemicolon 
\KwIn{Data $\calD=\{S_{i,1}, T_{i,1}, \ldots, S_{i,L}, T_{i,L}, Y_i\}_{i=1}^n$; true  $\theta^*$ range $[-U,U]$; true $F_{i,j}$ range $[\sigma^2, M^2]$.}
\KwOut{Consistent Estimation $\hat{\theta}^{(C)}_{n}$ and Asymptotically Normal Estimation $\hat{\theta}_n^{(N)}$}

\textsc{//Part I: Construct consistent estimation}

Set $H_{i,0}^{(C)}:=1,\forall i=1,\dots, n$.

\For{$i=\{1,\dots, n\},  j=\{1,\dots, L\}$}{
Set $H_{i,j}^{(C)}:= (I_{d_\tau}\otimes \chi(S_{i,j,\Omega}))\Var^+_{i,j}(T_{i,j})^{-1/2} 
    (I_{d_\tau}\otimes \chi(S_{i,j,\Omega}))^\top$.
}
\For{$i=\{1,\dots, n\},  j=\{1,\dots, L\}$}{

Let $\tilde{\theta}^{(C)}_{i,j}$ solve the AW-GMM estimator \eqref{eq:gmm_estimator_linear} using 
 data $\{Z_{i'}\}_{i'=1}^{i-1}$ and weights $\{H_{i',j}^{(C)}\}_{i'=1}^{i-1}$. 
 
}

Set the consistent estimation $\hat{\theta}^{(C)}_{n}$ as $\tilde{\theta}^{(C)}_{n,L}$.

 \hrulefill

\textsc{//Part II: Construct asymptotically normal estimation}

\For{$i=\{1,\dots, n\}$}{
Set $\hat{\sigma}_{i,0}^2 := \frac{1}{n_{i,j}^F}\sum_{i'\in\calI_{i,j}^F}\bp{Y_{i'}-\sum_{j=1}^L  
(T_{i',j}\otimes\chi(S_{i',j,\Omega}))^\top
\tilde{\theta}_{i,j}^{(C)}}^2$.

Set $H_{i,0}^{(N)}=\hat{\sigma}_{i,0}^{-1}$.
}

\For{$i=\{1,\dots, n\}$, $j=\{1,\dots, L\}$}{

Obtain $\hat{f}_{i,j}$ via Algorithm \ref{algo:f}$(\calD, i,j,\tilde{\theta}_{i,j}^{(C)}, [\sigma^2, M^2])$.

Obtain $\hat{\nu}_{i,j}$ via Algorithm \ref{algo:v}$(\calD, i,j)$.

Set $H_{i,j}^{(N)}=\hat{f}_{i,j}(S_{i,1:j}, T_{i,1:j-1})^{-1/2}(I_{d_\tau}\otimes\hat{\nu}_{i,j}(S_{i,1:j-1}, T_{i,1:j-1})^{-1/2}) (I_{d_\tau}\otimes \chi(S_{i,j,\Omega}))\Var^+_{i,j}(T_{i,j})^{-1/2} 
    (I_{d_\tau}\otimes \chi(S_{i,j,\Omega}))^\top\cdot\|\chi(S_{i,j,\Omega})\|_2^{-2}$.
}

Set the  asymptotically normal estimation $\hat{\theta}^{(N)}_{n}$ as the solution to the AW-GMM estimator \eqref{eq:gmm_estimator_linear} using the full data $\calD$ and weights $\{H_{i,j}^{(N)}\}_{i=1}^n$.

\caption{Post RL Estimation and Inference under  Model \ref{algo:plmdgp}}
\label{algo:plmdgp_weights}
\end{algorithm}

\begin{corollary}
\label{cor:plmm_normality_rate}
Consider  Model \ref{algo:plmdgp} and  the no-treatment $0_{1:L}$ policy as the evaluation policy. Suppose that the nuisance components $g_j$ and $h_{j}$ are  estimated via Lasso regression as in Lemma \ref{lemma:estimation_rate_plmm}. Under Assumption \ref{assump:overlap} with behavior exploration rate $\alpha\in[0,\frac{1}{L+1})$, 
the AW-GMM estimation $\hat{\theta}_n^{(N)}$  given by  Algorithm \ref{algo:plmdgp_weights} is asymptotically normal with strong Gaussian approximation rate of $\widetilde O\big(L^2 n^{-\min\big(\frac{1-\alpha}{12}, \frac{1-L\alpha}{10}, \frac{1-(L+1)\alpha}{2}\big)}\big)$.
\end{corollary}

\begin{figure}[t]
  \centering
    \includegraphics[width=\textwidth]{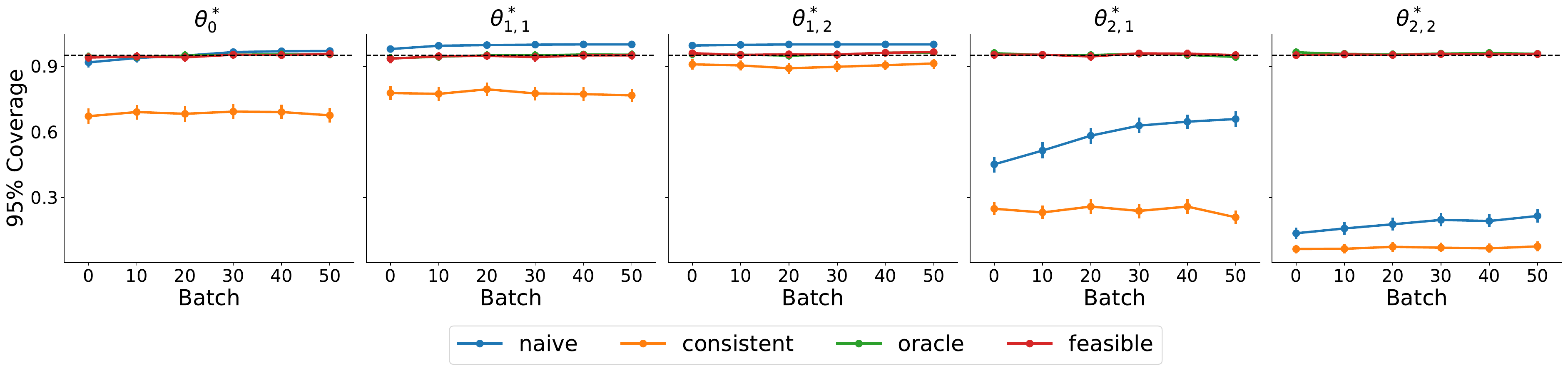}
    \includegraphics[width=\textwidth]{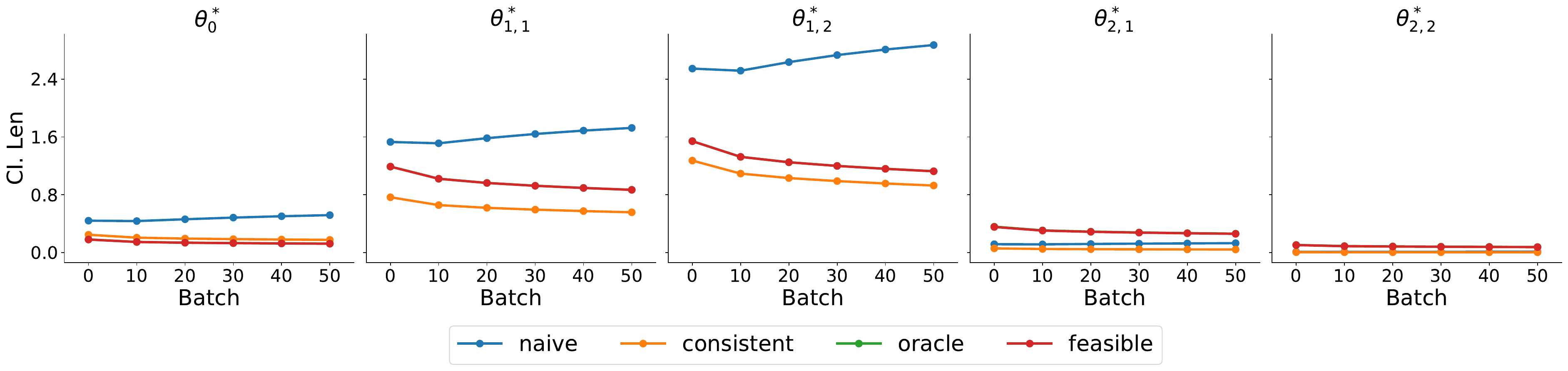}
\caption{
Inference results of  AW-GMM Estimations with different weights across varying sample size. Error bars are $95\%$ confidence intervals derived from $10^3$ simulations.
Results under \oracle~weights are shown in dashed line to indicate that oracle weighting   requires knowledge of ground truth structure parameters and thus cannot be applied in practice.
AW-GMM Estimations with \textsc{Oracle} and \textsc{Feasible} weights meet nominal  coverage, while the \naive~and \consistent~are either under- or over-coverage. 
}
    \label{fig:cov}
\end{figure}

\begin{figure}[t]
    \centering
    \includegraphics[width=\textwidth]{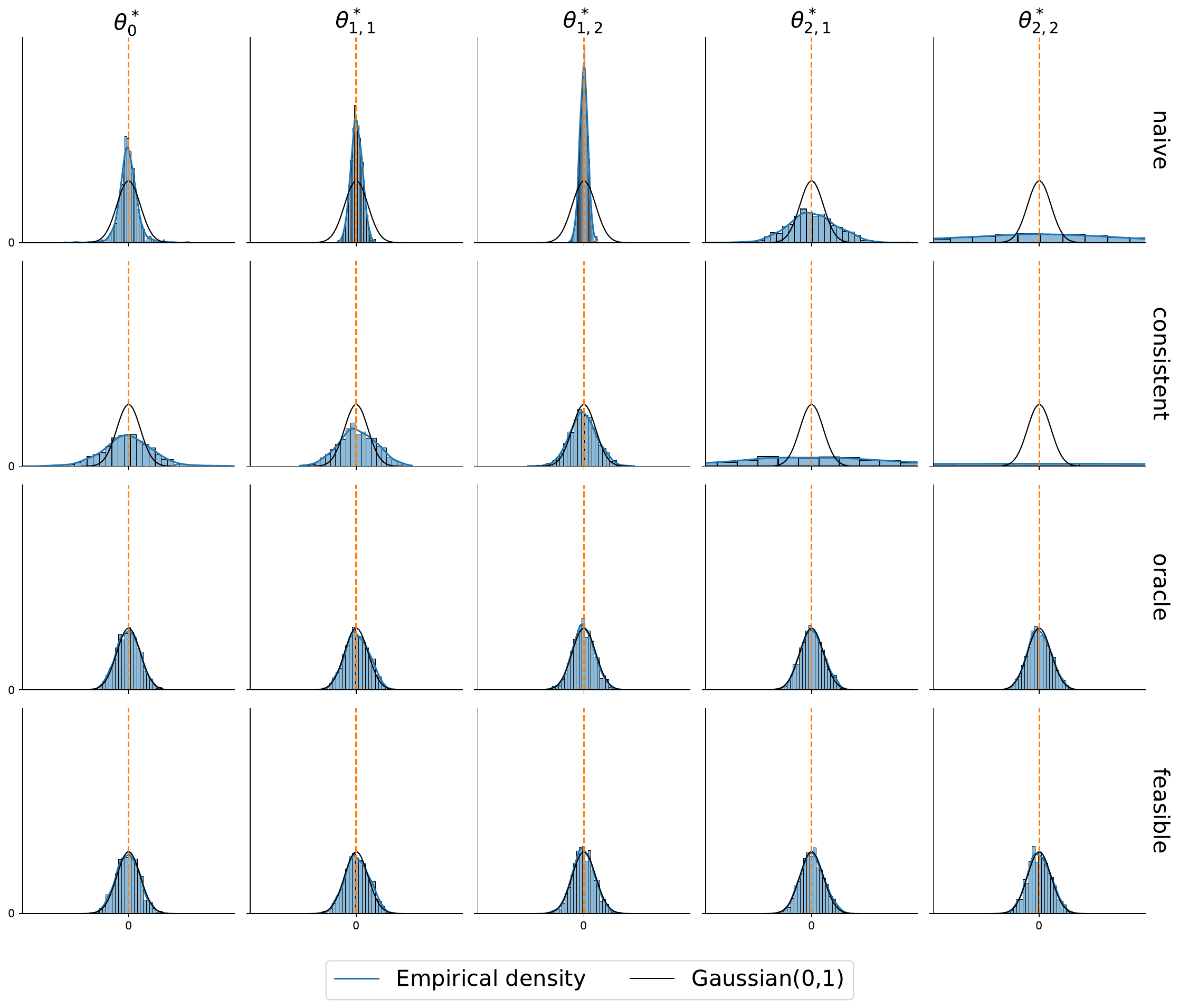}
    \caption{Histogram of studentized statistics from Gaussian approximation \eqref{eq:strong_gaussian_approx} at sample szie $n=5\times 10^3$.  Numbers are aggregated from  $10^3$ simulations. AW-GMM Estimations with \textsc{Oracle} and \textsc{Feasible} weights are asymptotically normal.}
    \label{fig:tstat}
\end{figure}

\begin{figure}[t]
  \centering
    \includegraphics[width=\textwidth]{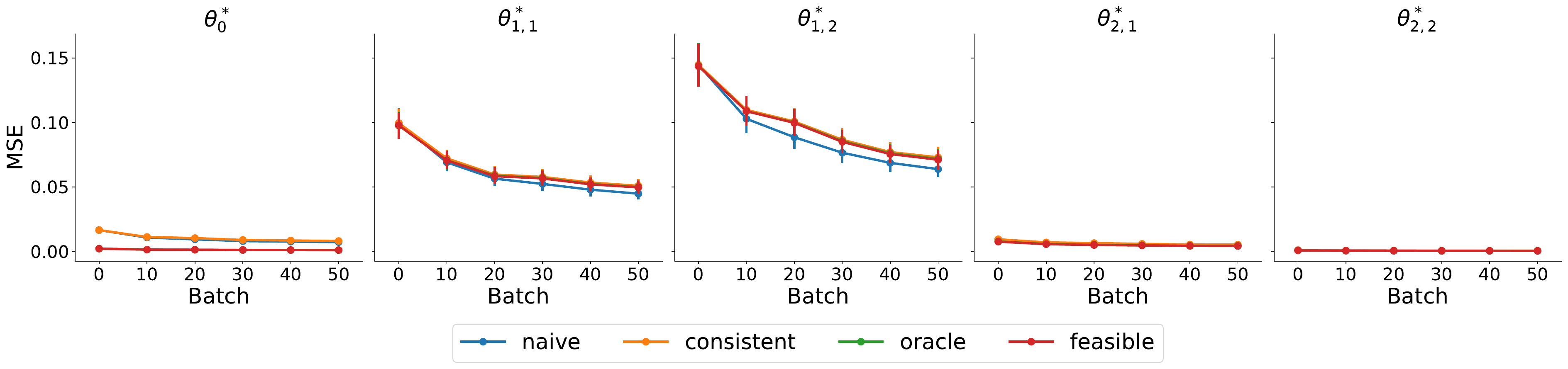}
     \includegraphics[width=\textwidth]{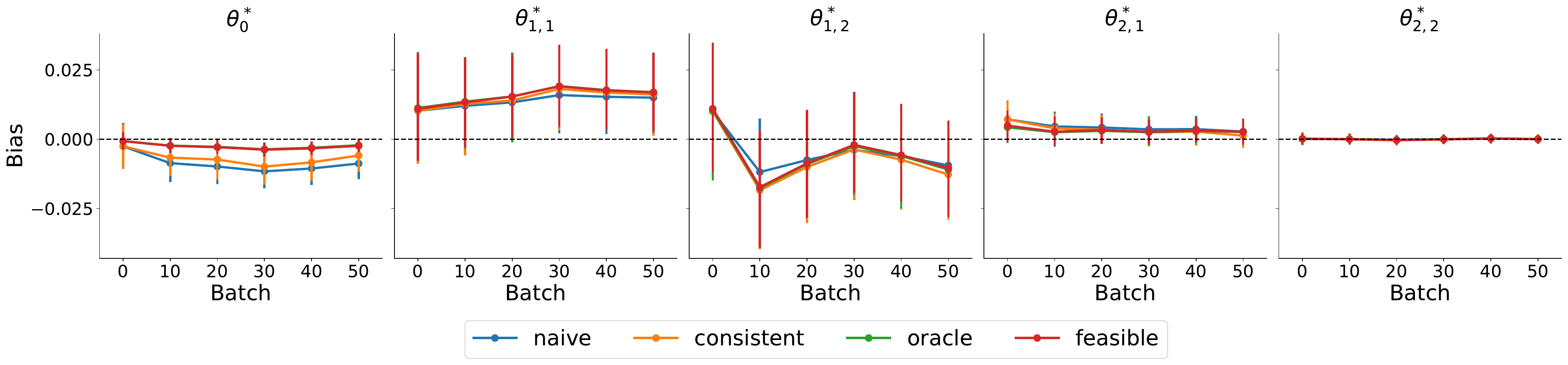}
\caption{Estimation results of  AW-GMM Estimations with different weights across varying sample size. Error bars are $95\%$ confidence intervals derived from $10^3$ simulations. Results under \oracle~weights are shown in dashed line to indicate that oracle weighting   requires knowledge of ground truth structure parameters and thus cannot be applied in practice. AW-GMM Estimations with \oracle~and \feasible~weights provide more accurate estimations for policy value $\theta_0^*$.}
    \label{fig:estimation}
\end{figure}

\begin{figure}[t]
  \centering
    \includegraphics[width=\textwidth]{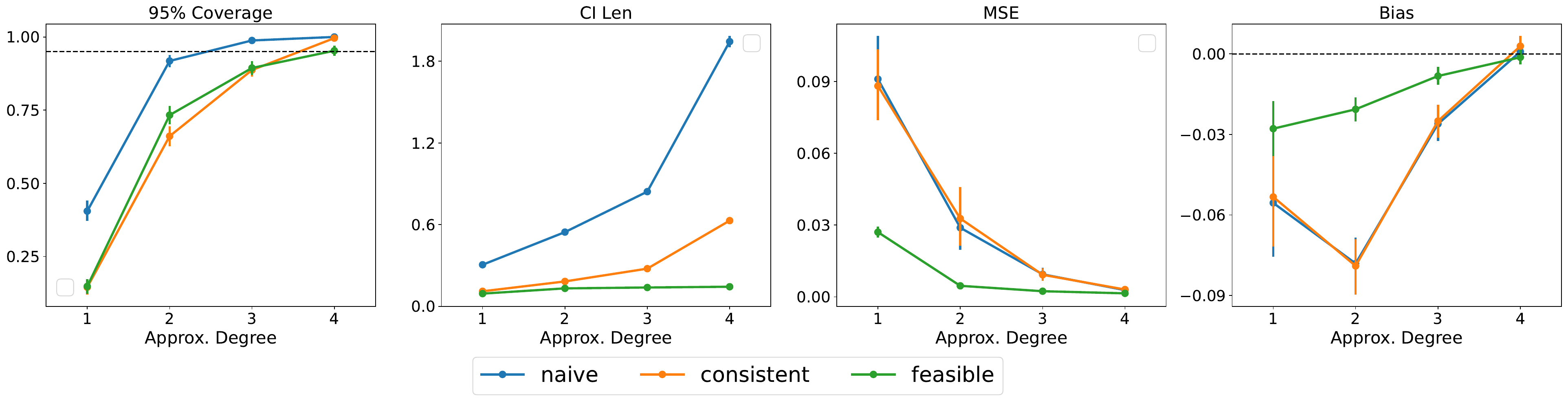}
\caption{Estimation and inference results of AW-GMM Estimations with different weights under  mis-specification at sample size $n=5\times 10^3$. We use polynomial approximations from degrees 1 to 5 for exponential feature mappings.
Error bars are $95\%$ confidence intervals derived from $10^3$ simulations.
AW-GMM Estimations under \textsc{Feasible} weights achieve nominal coverage and tight confidence intervals for degrees above one, consistently offering more accurate estimations with lower MSE and bias across all approximation degrees.}
    \label{fig:misspecifcation}
\end{figure}

\subsection{Numerical Experiments}
We finally present empirical evidence supporting our method's effectiveness in scenarios with both correct and mis-specified feature mapping function $\phi_j$. Our findings illustrate that applying weights enhances outcomes in both cases: under correct specification, it empirically confirms our method's consistency and asymptotic normality across all structural parameters. In cases of mis-specification, it notably improves the accuracy of estimating the evaluation policy value.

We study a two-stage high-dimensional Markovian model (with $L=2$), as outlined in Model \ref{algo:plmdgp}, focusing on binary treatment scenarios. Data collection is performed by an $\epsilon$-greedy episodic RL agent, which sequentially rolls out an episode for each unit. The agent's behavior policy undergoes batch updates, with each batch including $100$ units; the exploration amount $\epsilon$ decays at a polynomial rate such that for any given batch index $b$, the exploration $\epsilon$ is set to be $b^{-0.5}/2$, satisfying Assumption \ref{assump:overlap}.  Throughout the experiment, a total of $5000$ units are collected.

We consider high dimensional states and low dimensional features, with sparse linear models for both the state transition and the final outcome.
Specifically, the state space is in $\bbR^{10}$ with $S_{i,j}= (S_{i,j,1},\dots, S_{i,j,10})$; only  the first coordinate is informative, with others being noises. 
We focus on  GMM-estimations using four different weighting schemes:

\begin{itemize}
    \item \naive: Standard GMM estimation with no weights applied.
    \item \consistent:  AW-GMM estimation $\hat{\theta}^{(C)}_{n}$ under consistency weights as  in  Algorithm \ref{algo:plmdgp_weights}.
    \item \oracle: AW-GMM estimation with oracle weights, using  ground truth $F_{i,j}$ and $V_{i,j}$ in Corollary \ref{cor:normality}.
    \item \feasible: AW-GMM estimation $\hat{\theta}_n^{(N)}$ under asymptotic normality weights as  in  Algorithm~\ref{algo:plmdgp_weights}.
\end{itemize}
\medskip

Note that the \oracle~weights are infeasible, as they require knowledge of the true data-generating process. Conversely, the \naive, \consistent, and \feasible~weights can either be directly computed or estimated from the data. Below we show that \feasible~weighting scheme performs comparably to the \oracle~and significantly outperforms the other two  in  achieving asymptotic normality for post-RL inference; the \feasible~also demonstrates robustness under misspecification.

\subsubsection{Estimation and Inference Validity.}
We first consider cases with correctly specified feature mapping. 
Define the true feature mapping as $\phi_j(S_{i,1:j}, T_{i,1:j})=T_{i,j}\cdot (S_{i,j,1}, S_{i,j,1}^2)\in\bbR^2$. 
Our  estimand,  $\theta^*=(\theta_0^*, \theta_{1,1}^*, \theta_{1,2}^*, \theta_{2,1}^*, \theta_{2,2}^*)$, includes the evaluation policy value $\theta_0^*$  and structure parameters $(\theta_{1,1}^*, \theta_{1,2}^*)$ and $(\theta_{2,1}^*, \theta_{2,2}^*)$, indicating the effect of the treatment at each stage. 
Figure \ref{fig:cov} shows that 
 estimations with \oracle~and \feasible~weights achieve nominal coverage, contrasting with the \naive~no-weighting or \consistent~weights, which either have low coverage or are overly conservative. In particular, the \naive~estimator's confidence intervals for the policy value $\theta_0^*$ and the first-stage treatment effect $(\theta_{1,1}^*,\theta_{1,2}^*)$ even widen with increased sample size. 
Figure \ref{fig:tstat} further shows that the studentized statistics of the AW-GMM estimators, derived from Eq.\eqref{eq:strong_gaussian_approx}, under \oracle~and \feasible~weights conform to asymptotic normality, unlike those from other weighting schemes.
Furthermore, Figure \ref{fig:estimation} shows that while all estimators yield similar results for the stage-wise treatment effect estimation, those with \oracle~and \feasible~weights achieve higher accuracy in estimating the policy value $\theta_0^*$ with smaller MSE and bias.

\subsubsection{Robustness under Misspecification.} Transitioning to cases of mis-specification, we consider a true feature mapping defined as  $\phi_j(S_{i,1:j}, T_{i,1:j})=\exp(T_{i,j}S_{i,j,1}/2) -1$, unknown to the AW-GMM estimators.
These estimators then use polynomial approximations with degree $d\in[1:5]$ instead of the true $\phi_j$, i.e., we use $\hat{\phi}^{(d)}_j(S_{i,1:j}, T_{i,1:j})=T_{i,j}(S_{i,j,1}, \dots, S_{i,j,1}^d)\in\bbR^d$  as the approximated feature mapping for each degree $d$. 
Under this scenario, the stage-wise structural parameters lose their causal interpretation, yet estimating the evaluation policy value $\theta_0^*$ remains relevant. Figure \ref{fig:misspecifcation} shows that AW-GMM estimations under \feasible~weights outperform those with \naive~no-weighting or \consistent~weights, demonstrating better coverage, narrower confidence intervals, and reduced MSE and bias. 
Interestingly, with \feasible~weights, improvements in inference validity (coverage) and efficiency (confidence interval length) plateau for approximation degrees $d \geq 3$, while estimation quality (MSE and bias) slightly deteriorates at $d=5$. 
The decrease in estimation accuracy at higher approximation degrees is due to the requirement to estimate more structural parameters from the same sample size, which instead complicates the estimation process.
This observation highlights a  balance between inference and estimation: without knowing the precise feature mapping, selecting an approximation with appropriate complexity is crucial in optimizing both estimation accuracy and inferential robustness.



\bibliographystyle{informs2014}
\bibliography{reference}

\newpage

\begin{APPENDICES}
\section{Proofs of Main Lemmas}
\subsection{Proof of Lemma \ref{lemma:identification_policy_value}}
\label{appendix:proof_identification_policy_value}
We follow the proof pattern for Lemma 6 in \cite{lewis2020double}. Note that for the observed $\y$, we always have $\y\equiv \y(T_{1:L})$.  For any stage $j$, we have
\begin{align*}
    \theta_j^\top \phi_j(S_{1:j}, T_{1:j})\stackrel{(i)}{=} \gamma(S_{1:j}, T_{1:j}) \stackrel{(ii)}{=}\bbE\bb{ \y(T_{1:j}, \pi^*_{j+1:L}) - \y(T_{1:j-1}, 0, \pi^*_{j+1:L}) \mid S_{1:j}, T_{1:j}} ,
\end{align*}
where (i) is by  Assumption \ref{assump:linear_blip_function}, (ii) is by the definition of blip function. Moreover, note that:
\begin{align*}
    \theta_j^\top \phi_j(S_{1:j}, T_{1:j-1}, \pi^*(S_{1:j}, T_{1:j-1})) =~& \gamma(S_{1:j}, T_{1:j-1}, \pi^*(S_{1:j}, T_{1:j-1}))\\
    =~& \bbE\bb{ \y(T_{1:j}, \pi^*_{j+1:L}) - \y(T_{1:j-1}, 0, \pi^*_{j+1:L}) \mid S_{1:j}, T_{1:j-1}, T_j=\pi^*(S_{1:j}, T_{1:j-1})}\\
    =~& \bbE\bb{ \y(T_{1:j-1}, \pi^*_{j:L}) - \y(T_{1:j-1}, 0, \pi^*_{j+1:L}) \mid S_{1:j}, T_{1:j-1}, T_j=\pi^*(S_{1:j}, T_{1:j-1})}\\
    \stackrel{(iii)}{=}~& \bbE\bb{ \y(T_{1:j-1}, \pi^*_{j:L}) - \y(T_{1:j-1}, 0, \pi^*_{j+1:L}) \mid S_{1:j}, T_{1:j-1}, T_j},
\end{align*}
where (iii) follows since the counterfactual outcomes are independent of the value of $T_j$ conditional on $S_{1:j}, T_{1:j-1}$ by Sequential Conditional Exogeneity Assumption~\ref{assump:exogeneity}. Subtracting the two equalities, and by the definition of $\Phi_j$, we derive that:
\begin{align}
\bbE\bb{ \y(T_{1:j}, \pi^*_{j+1:L}) - \y(T_{1:j-1}, \pi^*_{j:L}) \mid S_{1:j}, T_{1:j}} = \theta_j^\top \Phi_j
\end{align}
Then we have
\begin{align*}
    \bbE\bb{ \y-\y(T_{1:j-1}, \pi^*_{j:L})|S_{1:j}, T_{1:j} } =~& \bbE\bb{ \y(T_{1:L})-\y(T_{1:j-1}, \pi^*_{j:L}) \mid S_{1:j}, T_{1:j} }\\
    \stackrel{(i)}{=}~& \sum_{j'=j}^L\bbE\bb{ \y(T_{1:j'}, \pi_{l'+1:L})-\y(T_{1:j'-1}, \pi_{l':L}) \mid S_{1:j}, T_{1:j} }\\
     =~& \sum_{j'=j}^L\bbE\bb{ \bbE\bb{\y(T_{1:j'}, \pi_{l'+1:L})-\y(T_{1:j'-1}, \pi_{l':L}) \mid X_{1:j'}, T_{1:j'} } \mid S_{1:j}, T_{1:j} }\\
     =~&\sum_{j'=j}^L\bbE\bb{ \theta_{j'}^\top \Phi_{j'} \mid S_{1:j}, T_{1:j} },
\end{align*}
where (i) uses  a telescoping sum. 
Rearranging the above, we have
\begin{align*}
    \bbE\bb{ \y(T_{1:j-1}, \pi^*_{j:L}) \mid S_{1:j}, T_{1:j} }  = \bbE\bb{ \y -\sum_{j'=j}^L\theta_{j'}^\top \Phi_{j'} \mid S_{1:j}, T_{1:j} }.
\end{align*}

\subsection{Proof of Lemma \ref{lemma:identification_parameter}}
\label{appendix:proof_identification_parameter}
We  adapt  the proof pattern of   Lemma 7 in  \cite{lewis2020double} to the RL data.
\begin{align*}
    \bbE_{i,j}^+\bb{R_{i,j}\, \bp{\Psi_{i,j} - \bar{\Psi}_{i,j}}} =~& \bbE_{i,j}^+\bb{ \bbE\bb{R_{i,j} \mid \calF_{i,j}^+, T_{i,j}}\, \bp{\Psi_{i,j} - \bar{\Psi}_{i,j}}} 
\end{align*}
Moreover, note that by Lemma~\ref{lemma:identification_policy_value} and the definition of $R_{i,j}$, we have:
\begin{align*}
    \bbE\bb{R_{i,j} \mid  \calF_{i,j}^+, T_{i,j}} = \bbE\bb{\y(T_{1:j-1}, \pi^*_{j:L}) \mid \calF_{i,j}^+, T_{i,j}}
\end{align*}
Moreover, by Assumption~\ref{assump:exogeneity}, we have:
\begin{align*}
    \bbE\bb{\y(T_{1:j-1}, \pi^*_{j:L}) \mid  \calF_{i,j}^+, T_{i,j}} = \bbE\bb{\y(T_{1:j-1}, \pi^*_{j:L}) \mid  \calF_{i,j}^+} = \bbE_{i,j}^+\bb{\y(T_{1:j-1}, \pi^*_{j:L})}
\end{align*}
Combining the last three equations, we conclude that:
\begin{align*}
\bbE_{i,j}^+\bb{R_{i,j}\, \bp{\Psi_{i,j} - \bar{\Psi}_{i,j}}}
=~& \bbE_{i,j}^+\bb{ \bbE_{i,j}^+\bb{\y(T_{1:j-1}, \pi^*_{j:L}) }\, \bp{\Psi_{i,j} - \bar{\Psi}_{i,j}}} \\
=~& \bbE_{i,j}^+\bb{\y(T_{1:j-1}, \pi^*_{j:L}) }\, \bbE_{i,j}^+\bb{\Psi_{i,j} - \bar{\Psi}_{i,j}} = 0
\end{align*}
where the last equality holds, since by definition $\bar{\Psi}_{i,j}:=\bbE_{i,j}^+[\Psi_{i,j}]$.

\subsection{Proof of Lemma \ref{lemma:mds-1}}
\label{appendix:proof_mds-1}

We  show that $\sum_{i=1}^n H_i\xi_{i}$ is a sum of  martingale difference sequence adapted to filtration $\{\calF_i\}$. 
By Lemma~\ref{lemma:identification_parameter} and the definition of $\xi_{i,j}$, we have for any $j\in \{0,\dots,L\}$
\begin{align*}
    \bbE_{i,j}^+\bb{H_{i,j} \xi_{i,j} } = \bbE_{i,j}^+\bb{ H_{i,j}\, R_{i,j}\, \bp{\Psi_{i,j} - \bar{\Psi}_{i,j}}} \stackrel{(i)}{=} H_{i,j}\, \bbE_{i,j}^+\bb{ R_{i,j}\, \bp{\Psi_{i,j} - \bar{\Psi}_{i,j}}} \stackrel{(ii)}{=} 0
\end{align*}
where $(i)$ follows by construction that $H_{i,j}$ is measurable with respect to $\calF^+_{i,j}$ and $(ii)$ follows by Lemma~\ref{lemma:identification_parameter}. Therefore we have:
\begin{align*}
    \bbE_i\bb{H_{i,j}\xi_{i,j} } =  \bbE_i     \bb{\bbE_{i,j}^+\bb{H_{i,j}\xi_{i,j} }}=0,\quad \forall j\in[0:L],
\end{align*}
implying $ \bbE_i\bb{H_{i}\xi_{i} }=0$.

\subsection{Proof of Lemma \ref{lemma:mds}}
\label{appendix:proof_mds}

We first prove that $\bbE[R_{i,j}^2\mid T_{i,j}, \calF_{i,j}^+]$ does not depend on $T_{i,j}$ such that
$\bbE[R_{i,j}^2\mid T_{i,j}, \calF_{i,j}^+]=\bbE[R_{i,j}^2\mid  \calF_{i,j}^+]=\bbE_{i,j}^+[R_{i,j}^2]$.
Define $\Bar{R}_{i,j}:=\bbE[R_{i,j}\mid T_{i,j}, \calF^+_{i,j}]$. We have:
\begin{align*}
    \bbE[R_{i,j}^2\mid T_{i,j}, \calF_{i,j}^+] = \bbE[(R_{i,j}-\Bar{R}_{i,j})^2\mid T_{i,j}, \calF_{i,j}^+]  + \Bar{R}_{i,j}^2 = \Var(R_{i,j}\mid T_{i,j}, \calF_{i,j}^+) + \Bar{R}_{i,j}^2.
\end{align*}
By Assumption \ref{assump:homoscedasticity}, $\Var(R_{i,j}\mid T_{i,j}, \calF_{i,j}^+) \equiv\Var(R_{i,j}\mid \calF_{i,j}^+)$ does not depend on $T_{i,j}$. Moreover, Appendix \ref{appendix:proof_identification_parameter} proves that $\Bar{R}_{i,j}$ is independent of $T_{i,j}$ and thus $\Bar{R}_{i,j}=\bbE_{i,j}^+[R_{i,j}]$. We therefore have
\begin{equation}
    \bbE[R_{i,j}^2\mid T_{i,j}, \calF_{i,j}^+]\equiv \bbE_{i,j}^+[R_{i,j}^2], \label{eq:residual_independence}
\end{equation}
which does not depend on $T_{i,j}$.

\paragraph{Part (a).} Consider the conditional covariance. Without loss of generality, assume $j_1<j_2$. Since $H_{i,j}$ are measurable with respect to $\calF^+_{i,j}$, we can write:
\begin{align*}
   \bbE_{i,j_2}^+\bb{H_{i,j_1}\xi_{i,j_1}\xi_{i,j_2}^\top H_{i,j_2}^\top }
    =~& H_{i,j_1} \bp{ \Psi_{i,j_1}-\bar{\Psi}_{i,j_1} }
    \underbrace{\bbE_{i,j_2}^+\bb{ R_{i,j_1}\,R_{i,j_2} \bp{ \Psi_{i,j_2}-\bar{\Psi}_{i,j_2}}^\top }}_{(I)} H_{i,j_2}^\top
\end{align*}
Note that:
\begin{align*}
    R_{i,j_1} = R_{i,j_2} - \sum_{j'=j_1}^{j_2-1}\Phi_{i,j'}^\top\theta_{j'}^*
\end{align*}
Thus:
\begin{align*}
    (I) =~& \bbE_{i,j_2}^+\bb{ R_{i,j_2}^2 \bp{ \Psi_{i,j_2}-\bar{\Psi}_{i,j_2}}^\top } - \sum_{j'=j_1}^{j_2-1}\Phi_{i,j'}^\top\theta_{j'}^* \bbE_{i,j_2}^+\bb{ R_{i,j_2} \bp{ \Psi_{i,j_2}-\bar{\Psi}_{i,j_2}}^\top }\\
    =~& \bbE_{i,j_2}^+\bb{ R_{i,j_2}^2 \bp{ \Psi_{i,j_2}-\bar{\Psi}_{i,j_2}}^\top }  \tag{by Lemma~\ref{lemma:identification_parameter}}\\
      =~& \bbE_{i,j_2}^+\bb{\bbE[ R_{i,j_2}^2 \mid T_{i,j_2}, \calF_{i,j_2}^+]\bp{ \Psi_{i,j_2}-\bar{\Psi}_{i,j_2}}^\top }  \\
     =~& \bbE_{i,j_2}^+\bb{\bbE_{i,j_2}^+[ R_{i,j_2}^2]\bp{ \Psi_{i,j_2}-\bar{\Psi}_{i,j_2}}^\top }  \tag{by \eqref{eq:residual_independence}}\\
     =~& \bbE_{i,j_2}^+\bb{ R_{i,j_2}^2}\bbE_{i,j_2}^+\bb{\bp{ \Psi_{i,j_2}-\bar{\Psi}_{i,j_2}}^\top }   = \zero.
\end{align*}
Therefore, we have:
\begin{equation*}
\Cov_i(H_{i,j_1}\xi_{i,j_1},H_{i,j_2} \xi_{i,j_2})=\bbE_i[H_{i,j_1}\xi_{i,j_1} \xi_{i,j_2}^\top H_{i,j_2}^\top] = \bbE_i\bb{\bbE_{i,j_2}^+[H_{i,j_1}\xi_{i,j_1} \xi_{i,j_2}^\top H_{i,j_2}^\top]}= \zero.
\end{equation*}

\paragraph{Part (b).} By the definition of $\xi_{i,j}$ and the fact that the weights $H_{i,j}$ are measurable in $\calF_{i,j}^+$:
\begin{align*}
 \bbE_{i,j}^+\bb{H_{i,j_1} \xi_{i,j}\xi_{i,j}^\top H_{i,j_2}^\top} =~& H_{i,j}\, \bbE_{i,j}^+\bb{R_{i,j}^2 \bp{\Psi_{i,j} - \bar{\Psi}_{i,j}}\, \bp{\Psi_{i,j} - \bar{\Psi}_{i,j}}^\top}\, H_{i,j}^\top\\
 =~& H_{i,j}\, \bbE_{i,j}^+\bb{\bbE[R_{i,j}^2\mid T_{i,j},\calF_{i,j}^+] \bp{\Psi_{i,j} - \bar{\Psi}_{i,j}}\, \bp{\Psi_{i,j} - \bar{\Psi}_{i,j}}^\top}\, H_{i,j}^\top\\
 =~& H_{i,j}\, \bbE_{i,j}^+\bb{\bbE_{i,j}^+[R_{i,j}^2] \bp{\Psi_{i,j} - \bar{\Psi}_{i,j}}\, \bp{\Psi_{i,j} - \bar{\Psi}_{i,j}}^\top}\, H_{i,j}^\top \tag{by \eqref{eq:residual_independence}}\\
=~& H_{i,j}\, \bbE_{i,j}^+\bb{R_{i,j}^2 }\, \bbE_{i,j}^+ \bb{\bp{\Psi_{i,j} - \bar{\Psi}_{i,j}}\, \bp{\Psi_{i,j} - \bar{\Psi}_{i,j}}^\top}\, H_{i,j}^\top \\
    =~& \bbE_{i,j}^+\bb{R_{i,j}^2 }  H_{i,j} \,\Var_{i,j}^+\bp{\Psi_{i,j}}\, H_{i,j}^\top 
\end{align*}
Therefore,
\begin{align*}
        \Var_{i}(H_{i,j}\xi_{i,j})=\bbE_i[H_{i,j}\xi_{i,j}\xi_{i,j}^\top H_{i,j}^\top]=\bbE_i\bb{ \bbE_{i,j}^+[H_{i,j}\xi_{i,j}\xi_{i,j}^\top H_{i,j}^\top]} = \bbE_i\bb{\bbE^+_{i,j}\bb{ R_{i,j}^2 } \cdot H_{i,j}\,\Var^+_{i,j}(\Psi_{i,j})\,
    H_{i,j}^\top}.
\end{align*}

\section{Proof of Theorem \ref{thm:consistency}}
\label{appendix:consistency}

We hereby show that GMM estimator $\hat{\theta}_n\in\argmin_{\theta\in\Theta}\calL_n(\theta)$ converges to the true parameter $\theta^*\in\argmin_{\theta\in\Theta}\calL(\theta)$, where
\begin{align*}
      \calL_n(\theta) &= \bn{m_n(\theta)
      }_A^2, \quad \mbox{for}\quad m_n(\theta):=\frac{1}{n}\sum_{i=1}^n H_i\, \bp{\beta_i+J_i\theta},\\
      \calL(\theta)&=\bn{\bar{m}_n(\theta)
      }_A^2, \quad \mbox{for}\quad \bar{m}_n(\theta):=\frac{1}{n}\sum_{i=1}^n H_i\bp{ \bar{\beta}_i+\bar{J}_i\theta}.
\end{align*}
In Appendix \ref{appendix:thm_1_a}, we show that  uniformly across $\theta\in\Theta$, we have:
\begin{equation}
\label{eq:empirical_loss_uniform_convergence}
   \bbE\bb{I^2} = O(L^2n^{\alpha_2-1}),\quad \mbox{for}\quad I:=
   \max_{\theta\in\Theta}\|m_n(\theta)-\bar{m}_n(\theta)\|_A.
\end{equation}
Then for $\hat\theta_n$ that minimizes the loss $\calL_n(\theta) = \|m_n(\theta)\|_A^2$, we have:
\begin{align*}
    \|\bar{m}_n(\hat\theta_n)\|_A &= \|m_n(\hat\theta_n)\|_A + \|\bar{m}_n(\hat\theta_n)\|_A - \|m_n(\hat\theta_n)\|_A \\
    &\leq \|m_n(\hat\theta_n)\|_A + \|\bar{m}_n(\hat\theta_n)-m_n(\hat\theta_n)\|_A \tag{by triangular inequality}\\
    &\leq \|m_n(\theta^*)\|_A + \|\bar{m}_n(\hat\theta_n)-m_n(\hat\theta_n)\|_A \tag{$\hat\theta_n$ minimizes $\calL_n(\theta)$}\\
    &\leq \|\bar m_n(\theta^*)\|_A +\|\bar{m}_n(\theta^*)-m_n(\theta^*)\|_A+ \|\bar{m}_n(\hat\theta_n)-m_n(\hat\theta_n)\|_A \tag{by triangular inequality}\\
    &=\|\bar{m}_n(\theta^*)-m_n(\theta^*)\|_A+ \|\bar{m}_n(\hat\theta_n)-m_n(\hat\theta_n)\|_A \tag{by moment equation $\bar m_n(\theta^*)=0$}\\
    &\leq 2I.
\end{align*}
Therefore, 
\begin{equation}
\label{eq:moment_convergence}
    \bbE[ \|\bar{m}_n(\hat\theta_n)\|_A^2]\leq 4\bbE[I^2] \stackrel{\mbox{by \eqref{eq:empirical_loss_uniform_convergence}}}{=} O(L^2n^{\alpha_2-1}).
\end{equation}
On the other hand, we have:
\begin{align}
\label{eq:normalized_error_rate}
    \bbE\bb{
    \bn{\frac{1}{n}\sum_{i=1}^n H_i \bar{J}_i(\hat{\theta}_n-\theta^*)}_A^2
    }& = \bbE\bb{
    \|\bar{m}_n(\hat\theta_n) - \bar{m}_n(\theta^*) \|_A^2
    }\\
    &= \bbE\bb{
    \|\bar{m}_n(\hat\theta_n) \|_A^2
    }\tag{by moment equation $\bar m_n(\theta^*)=0$}\nonumber\\
    & = O(L^2n^{\alpha_2-1})\tag{by \eqref{eq:empirical_loss_uniform_convergence}}.\nonumber
\end{align}
With \eqref{eq:normalized_error_rate}, Appendix \ref{appendix:thm_1_b} applies the induction method to invert $\frac{1}{n}\sum_{i=1}^n H_i \bar{J}_i$ and  show that for each $j\in[0,L]$, it holds that:
\begin{align*}
      \bbE\bb{\bn{\hat{\theta}_{n,j}-\theta_{j}^*}_2} =O\bp{n^{\frac{(L-j+1)(\alpha_1+\alpha_2)-1}{2}}}\quad \mbox{and}\quad \bbE\bb{\bn{\hat{\theta}_{n,j}-\theta_{j}^*}^2_2} =O\bp{n^{(2L-2j+1)\alpha_1+\alpha_2-1}}
\end{align*}
Concluding the proof.

\subsection{Uniform Convergence}
\label{appendix:thm_1_a}
We now show:
\begin{equation}
   \bbE\bb{I^2} = O(L^2n^{\alpha_2-1}),\quad \mbox{for}\quad I:=
   \max_{\theta\in\Theta}\|m_n(\theta)-\bar{m}_n(\theta)\|_A.\tag{\ref{eq:empirical_loss_uniform_convergence}}
\end{equation}
We have:
\begin{align*}
   I&=
   \max_{\theta\in\Theta}\|m_n(\theta)-\bar{m}_n(\theta)\|_A\\
   &=\max_{\theta\in\Theta}\bn{
   \frac{1}{n}\sum_{i=1}^nH_i(\beta_i-\bar \beta_i) + \frac{1}{n}\sum_{i=1}^nH_i(J_i-\bar J_i) \theta
   }_A\\
    &\lesssim\max_{\theta\in\Theta}\bn{
   \frac{1}{n}\sum_{i=1}^nH_i(\beta_i-\bar \beta_i) + \frac{1}{n}\sum_{i=1}^nH_i(J_i-\bar J_i) \theta
   }_2 \tag{$\lambda_{\max}(A)=O(1)$}\\
   &\leq \max_{\theta\in\Theta} \bn{
   \frac{1}{n}\sum_{i=1}^nH_i(\beta_i-\bar \beta_i)}_2 + \bn{\frac{1}{n}\sum_{i=1}^nH_i(J_i-\bar J_i) \theta
   }_2\tag{triangular inequality}\\
    &\lesssim \underbrace{\bn{
   \frac{1}{n}\sum_{i=1}^nH_i(\beta_i-\bar \beta_i)}_2}_{C} + \underbrace{\bn{\frac{1}{n}\sum_{i=1}^nH_i(J_i-\bar J_i) 
   }_2}_{D},
\end{align*}
where for a symmatric matrix $M$, we use $\lambda_{\max}(M)$ to denote its largest eigenvalue and use $\|M\|_2$ to denote its operator norm induced by vector 2-norm. 
Appendix \ref{appendix:term_c} shows that $\bbE[C^2]=O(Ln^{\alpha_2-1})$, and Appendix \ref{appendix:term_d} shows that $\bbE[D^2]=O(L^2n^{\alpha_2-1})$. Therefore,
\[
\bbE[I^2]\lesssim \bbE[C^2]+ \bbE[D^2] = O(L^2n^{\alpha_2-1}),
\]
proving \eqref{eq:empirical_loss_uniform_convergence}.

\subsubsection{Bounding term C.}
\label{appendix:term_c}
We  show:
\begin{equation}
\label{eq:beta_convergence}
       \bbE\bb{C^2}=\bbE\bb{ \bn{\frac{1}{n}\sum_{i=1}^n H_i\, \beta_i - \frac{1}{n}\sum_{i=1}^n H_i\bar{\beta}_i}_2^2}  = O(Ln^{\alpha_2-1}).
\end{equation}
Note that $\{H_i(\beta_i-\bar \beta_i)\}$ is a martingale difference sequence adapted to $\{\calF_i\}$. 
This is because each $j$-th component 
\[
\bbE_i[H_{i,j}(\beta_{i,j}-\bar \beta_{i,j})] = \bbE_i \bbE_{i,j}^+[H_{i,j}(\beta_{i,j}-\bar \beta_{i,j})]  = \bbE_i \bb{H_{i,j}\bbE_{i,j}^+[(\beta_{i,j}-\bar \beta_{i,j})]}=0,
\]
by the definition $\bar \beta_{i,j}=\bbE_{i,j}^+[\beta_{i,j}]$ and  $H_{i,j}$ adapted to $\calF_{i,j}^+$.
We have
\begin{align*}
   \bbE\bb{ \bn{\frac{1}{n}\sum_{i=1}^n H_i\, \beta_i - \frac{1}{n}\sum_{i=1}^n H_i\bar{\beta}_i}_2^2 } &\stackrel{(i)}{=} \frac{1}{n^2}\sum_{i=1}^n\bbE\bb{ (\beta_i - \bar{\beta}_i)^\top H_i^\top H_i(\beta_i - \bar{\beta}_i)}\\
   & = \sum_{j=0}^L\frac{1}{n^2}\sum_{i=1}^n \bbE\bb{ (\beta_{i,j} - \bar{\beta}_{i,j})^\top H_{i,j}^\top H_{i,j}(\beta_{i,j} - \bar{\beta}_{i,j})}\\
   & = \sum_{j=0}^L\frac{1}{n^2}\sum_{i=1}^n \Tr\bp{\bbE\bb{ H_{i,j}(\beta_{i,j} - \bar{\beta}_{i,j})(\beta_{i,j} - \bar{\beta}_{i,j})^\top H_{i,j}^\top }},
\end{align*}
where (i) is due to that $\{H_i(\beta_i-\bar \beta_i)\}$ is a martingale difference sequence.
Now consider each diagonal entry $j$. We have:
\begin{align*}
     &\frac{1}{n^2}\sum_{i=1}^n\Tr\bp{\bbE\bb{ H_{i,j}(\beta_{i,j} - \bar{\beta}_{i,j})(\beta_{i,j} - \bar{\beta}_{i,j})^\top H_{i,j}^\top}}\\
  = & \frac{1}{n^2}\sum_{i=1}^n\Tr\bp{\bbE\bb{ H_{i,j}\Var_{i,j}^+(Y_i(\Psi_{i,j}-\bar{\Psi}_{i,j})) H_{i,j}^\top} }  = O(n^{-2}) \sum_{i=1}^n\Tr\bp{\bbE\bb{ H_{i,j}\Var_{i,j}^+(\Psi_{i,j}) H_{i,j}^\top}}   = O(n^{\alpha_2-1}),
\end{align*}
where the last equality is by Property \ref{property:weight_regularizing}(c). We thus have \eqref{eq:beta_convergence} hold.

\subsubsection{Bounding term D.}
\label{appendix:term_d}
We next show:
\begin{equation}
\label{eq:coeff_convergence}
      \bbE\bb{ \bn{ \frac{1}{n}\sum_{i=1}^n H_i\, J_i - \frac{1}{n}\sum_{i=1}^n H_i\bar{ J}_i}_2^2 }= O(L^2n^{\alpha_2-1}).
\end{equation}
Similarly,  $\{H_i(J_i-\bar J_i)\}$ is a martingale difference sequence adapted to $\{\calF_i\}$.
For any given $\theta\in\Theta$ with $\Theta$ being bounded, we have
\begin{align}
     &\bbE\bb{ \bn{\frac{1}{n}\sum_{i=1}^n H_i\, (J_i -\bar{J}_i)\theta}_2^2 } \leq \bbE\bb{ \bn{\frac{1}{n}\sum_{i=1}^n H_i\, (J_i -\bar{J}_i)}_{2}^2 \cdot \|\theta^*\|_2^2 }\nonumber\\
\stackrel{(i)}{\lesssim} &  \bbE\bb{ \bn{\frac{1}{n}\sum_{i=1}^n H_i\, (J_i -\bar{J}_i)}_{Frob}^2  } =\bbE\bb{\Tr\bp{\bp{\frac{1}{n}\sum_{i=1}^n H_i\, (J_i -\bar{J}_i)} \cdot \bp{\frac{1}{n}\sum_{i=1}^n H_i\, (J_i -\bar{J}_i)}^\top}}\nonumber\\
= &\Tr\bp{\bbE\bb{\bp{\frac{1}{n}\sum_{i=1}^n H_i\, (J_i -\bar{J}_i)} \cdot \bp{\frac{1}{n}\sum_{i=1}^n H_i\, (J_i -\bar{J}_i)}^\top}}\nonumber\\
\stackrel{(ii)}{=}& \frac{1}{n^2}\sum_{i=1}^n \Tr\bp{\bbE\bb{H_i\, (J_i -\bar{J}_i)(J_i -\bar{J}_i)^\top H_i^\top }}\nonumber \\
\stackrel{(iii)}{=}&\frac{1}{n^2}\sum_{i=1}^n \sum_{j=0}^L \sum_{k\geq j}\Tr\bp{\bbE\bb{H_{i,j}\, (J_{i,j,k} -\bar{J}_{i,j,k})(J_{i,j,k} -\bar{J}_{i,j,k})^\top H_{i,j}^\top }}:=\frac{1}{n^2}\sum_{i=1}^n \sum_{j=0}^L \sum_{k\geq j}\bbE\bb{\Lambda_{i,j,k}}\label{eq:jacobian_jk}
\end{align}
where $\Lambda_{i,j,k}:=\Tr\bp{\bbE_{i,j}^+\bb{H_{i,j}\, (J_{i,j,k} -\bar{J}_{i,j,k})(J_{i,j,k} -\bar{J}_{i,j,k})^\top H_{i,j}^\top }}$; (i) uses the boundedness of $\theta$ and that a matrix $\ell_2$ norm is bounded by its Frobenius norm; (ii) uses that $\{H_i\, (J_i -\bar{J}_i)\}_{i=1}^n$ is a martingale difference sequence; (iii) is due to that $H_i$ is a block-diagonal matrix and $J_i$ is a block upper-triangular matrix.
Now let's look into \eqref{eq:jacobian_jk} and analyze $\Lambda_{i,j,k}$.
\begin{align*}
    \Lambda_{i,j,k}:=&\Tr\bp{\bbE_{i,j}^+\bb{H_{i,j}\, (J_{i,j,k} -\bar{J}_{i,j,k})(J_{i,j,k} -\bar{J}_{i,j,k})^\top H_{i,j}^\top }}\\
    = &  \Tr( H_{i,j}\bbE_{i,j}^+[  ((\Psi_{i,j} -\bar{\Psi}_{i,j})\Phi_{i,k}^\top - \bbE_{i,j}^+[(\Psi_{i,j} -\bar{\Psi}_{i,j})\Phi_{i,k}^\top])((\Psi_{i,j} -\bar{\Psi}_{i,j})\Phi_{i,k}^\top - \bbE_{i,j}^+[(\Psi_{i,j} -\bar{\Psi}_{i,j})\Phi_{i,k}^\top])^\top ]H_{i,j}^\top)\\
\stackrel{(i)}{\lesssim} & \Tr(H_{i,j}\Var_{i,j}^+(\Psi_{i,j} -\bar{\Psi}_{i,j})H_{i,j}^\top),
\end{align*}
where (i) is due to Lemma \ref{lemma:var_inequality}, boundedness of $\|\Phi_{i,k}\|_2$, and that if positive semi-definite matrices $A\preceq B$, then $\Tr(CAC^\top)\leq \Tr(CBC^\top)$ for any matrix $C$.

Continuing \eqref{eq:jacobian_jk}, we have:
\begin{align*}
    \bbE\bb{ \bn{\frac{1}{n}\sum_{i=1}^n H_i\, (J_i -\bar{J}_i)\theta}_2^2 } &\lesssim \frac{1}{n^2}\sum_{i=1}^n \sum_{j=0}^L \sum_{k\geq j}\bbE\bb{\Tr(H_{i,j}\Var_{i,j}^+(\Psi_{i,j} -\bar{\Psi}_{i,j})H_{i,j}^\top)}\stackrel{(i)}{=}O(L^2n^{\alpha_2-1}),
\end{align*}
where (i) is by Property \ref{property:weight_regularizing}(c). Thus we have \eqref{eq:coeff_convergence} hold.

\subsection{Induction Step}
\label{appendix:thm_1_b}
We now bound the estimation error $(\hat\theta_n - \theta^*)$. Equation
 \eqref{eq:normalized_error_rate} shows that the normalized error satisfies
$ \bbE[\|\bar B_n (\hat{\theta}_n-\theta^*)  
    \|_{2}^2]=O(n^{\alpha_2-1})$,
for $ \Bar B_{n} := -n^{-1}\sum_{i=1}^nH_i \bar{J}_i$.
Note that $\Bar{B}_n$ is an upper block-triangular matrix. We will show that, for each diagonal block $\Bar{B}_{n,j,j}$, the matrix $\Bar{B}_{n,j,j}^\top \Bar{B}_{n,j,j}$ is invertible with high probability. This allows us to invert $\Bar{B}_{n,j,j}^\top \Bar{B}_{n,j,j}$ to bound the estimation error $(\hat{\theta}_{n,j} - \theta_j^*)$ in a backwards manner.

In particular, we apply induction method to prove the result. We start by introducing a few notations. With $\Bar B_{n} = -n^{-1}\sum_{i=1}^nH_i \bar{J}_i$, define $\Bar B_{n,j,j'}=n^{-1}\sum_{i=1}^n H_{i,j}\bbE_{i,j}^+\bb{(\Psi_{i,j}-\bar{\Psi}_{i,j})\Phi_{i,j'}^\top}$ as its $(j,j')$ block.
Similarly, define $ B^0_{n,j,j'}=n^{-1}\sum_{i=1}^n \bbE_{i,j}\bb{H_{i,j}(\Psi_{i,j}-\bar{\Psi}_{i,j})\Phi_{i,j'}^\top}$. Define their difference $\delta_{n,j,j'}:=B^0_{n,j,j'}-\Bar B_{n,j,j'}$. Appendix \ref{appendix:regularity_Bn} shows that:
\begin{equation}
\label{eq:bn_regularity}
\|\Bar B_{n,j,j'}\|^2_{Frob}=O\bp{n^{\alpha_1}},\quad \|B^0_{n,j,j'}\|^2_{Frob}=O\bp{n^{\alpha_2}}, \quad     \bbE\bb{\bn{\delta_{n,j,j'}}_{Frob}^2}=O\bp{n^{\alpha_2-1}} 
\end{equation}

By Property \ref{property:weight_regularizing}(a), we have $\forall l\in [0:L]$,  $(B^0_{n,j,j})^\top  B^0_{n,j,j} \succeq c_1^2 n^{-\alpha_1}I$, where   $c_1$ is a constant introduced in Property \ref{property:weight_regularizing}. The regularity of $\bar B_{n,j,j}$ inherits  from that of $B^0_{n,j,j}$.
Define the   event $\calE_n$ as:
\begin{align*}
    \calE_n = \left\{\Bar B_{n,j,j}^\top \Bar B_{n,j,j} \succeq \frac{c_1^2}{4} n^{-\alpha_1}I, \  \forall l\in\{0,\dots, L\}\right\},
\end{align*}
Appendix \ref{appendix:regularity_Bn} shows that $\calE_n $ happens  with high probability:
\begin{align}
\label{eq:en_high_probability}
 &1-\bbP(\calE_n)=O(Ln^{\alpha_1+\alpha_2-1}).
\end{align}
We now apply the induction step to bound the estimation error $\hat\theta_n - \theta^*$.

\paragraph{Base case: $j=L$.} We have:
\begin{align*}
    \bbE\bb{\|\hat{\theta}_{n,L}-\theta^*_L\|_2^2}& \lesssim  n^{\alpha_1}\bbE\bb{\bn{\Bar B_{n,L,L}(\hat{\theta}_{n,L}-\theta_L^*)}^2_2\one(\calE_n) }  + \bbP(\calE_n^c)\\
    &\leq  n^{\alpha_1}\bbE\bb{\bn{\Bar B_{n,L,L}(\hat{\theta}_{n,L}-\theta_L^*)}^2_2 }  + \bbP(\calE_n^c)\\
        &\stackrel{(i)}{\leq} n^{\alpha_1}\bbE\bb{\bn{\Bar B_{n}(\hat{\theta}_{n}-\theta^*)}^2_2}  + \bbP(\calE_n^c)\\
    &=  n^{\alpha_1} \bbE\bb{\bn{\frac{1}{n}\sum_{i=1}^n H_i\bar{J}_i \bp{\hat{\theta}_n-\theta^*}  
    }_{2}^2} + \bbP(\calE_n^c)\stackrel{(ii)}{=} O(n^{\alpha_1+\alpha_2-1}),
\end{align*}
where (i) is because $B_n$ is a block upper triangular matrix and 
(ii) is by \eqref{eq:normalized_error_rate} and \eqref{eq:en_high_probability}. This result also yield the $L^1$-norm convergence:
\begin{equation*}
     \bbE\bb{\|\hat{\theta}_{n,L}-\theta^*_L\|_2}\leq   \bbE\bb{\|\hat{\theta}_{n,L}-\theta^*_L\|_2^2}^{1/2} =  O(n^{\frac{\alpha_1+\alpha_2-1}{2}}).
\end{equation*}

\paragraph{Induction step.} Now consider $j<L$ recursively. Assume the induction hypothesis holds that for $j'=j+1, \dots, L$, such that:
\begin{align}
\label{eq:induction}
  \bbE\bb{\bn{\hat{\theta}_{n,j'}-\theta_{j'}^*}_2} =O\bp{n^{\frac{(L-j'+1)(\alpha_1+\alpha_2)-1}{2}}}\quad \mbox{and}\quad \bbE\bb{\bn{\hat{\theta}_{n,j'}-\theta_{j'}^*}^2_2} =O\bp{n^{(2L-2j'+1)\alpha_1+\alpha_2-1}}.    
\end{align}
Let's first prove the $L^2$-norm convergence. We have
\begin{align*}
  &  \bbE\bb{\|\hat{\theta}_{n,j}-\theta^*_j\|_2^2} \lesssim  n^{\alpha_1}\bbE\bb{\bn{\Bar B_{n,j,j}(\hat{\theta}_{n,j}-\theta^*_j)}^2_2\one(\calE_n) }  + \bbP(\calE_n^c)\leq   n^{\alpha_1}\bbE\bb{\bn{\Bar B_{n,j,j}(\hat{\theta}_{n,j}-\theta^*_j)}^2_2}  + \bbP(\calE_n^c)\\
   &=  n^{\alpha_1}\bbE\bb{\bn{\Bar B_{n,j,j}(\hat{\theta}_{n,j}-\theta^*_j) + \sum_{j'>j}\Bar B_{n,j,j'}(\hat{\theta}_{n,j'}-\theta^*_{j'}) 
   - \sum_{j'>j}\Bar B_{n,j,j'}(\hat{\theta}_{n,j'}-\theta^*_{j'})
   }^2_2 } + \bbP(\calE_n^c)\\
    &\stackrel{(i)}{\leq} 2L n^{\alpha_1}\bbE\bb{\bn{\Bar B_{n,j,j}(\hat{\theta}_{n,j}-\theta^*_j) + \sum_{j'>j}\Bar B_{n,j,j'}(\hat{\theta}_{n,j'}-\theta^*_{j'}) }^2_2 } + 2L
    n^{\alpha_1}
    \sum_{j'>j} \bbE\bb{\bn{\Bar B_{n,j,j'}(\hat{\theta}_{n,j'}-\theta^*_{j'})}^2_2 }+ \bbP(\calE_n^c)\\
    &\stackrel{(ii)}{\leq} 2L n^{\alpha_1}\bbE\bb{\bn{\Bar B_{n}(\hat{\theta}_{n}-\theta^*)}^2_2 }+
    2L
    n^{\alpha_1}
    \sum_{j'>j} \bbE\bb{\bn{\Bar B_{n,j,j'}(\hat{\theta}_{n,j'}-\theta^*_{j'})}^2_2 }+ \bbP(\calE_n^c)\\
    &=  2L n^{\alpha_1} \bbE\bb{\bn{\frac{1}{n}\sum_{i=1}^n H_i\bar{J}_i \bp{\hat{\theta}_n-\theta^*}  
    }_{2}^2} +
    2L
    n^{\alpha_1}
    \sum_{j'>j} \bbE\bb{\bn{\Bar B_{n,j,j'}(\hat{\theta}_{n,j'}-\theta^*_{j'})}^2_2 }+ \bbP(\calE_n^c)\\
    &\leq  2L n^{\alpha_1} \bbE\bb{\bn{\frac{1}{n}\sum_{i=1}^n H_i\bar{J}_i \bp{\hat{\theta}_n-\theta^*}  
    }_{2}^2} +
    2L
    n^{\alpha_1}
    \sum_{j'>j} \bbE\bb{\bn{\Bar B_{n,j,j'}}_{Frob}^2\bn{(\hat{\theta}_{n,j'}-\theta^*_{j'})}^2_2 }+ \bbP(\calE_n^c)\\
    &\stackrel{(iii)}{=} 2LO(n^{\alpha_2-1}) + 2L\sum_{j'>j} O\bp{n^{\alpha_1+\alpha_1+(2L-2j'+1)\alpha_1+\alpha_2-1}} + O(n^{\alpha_1+\alpha_2-1}) = O\bp{n^{(2L-2j+1)\alpha_1+\alpha_2-1}},
\end{align*}
where (i) is by triangular inequality and Cauchy-Swarchz inequality, (ii) is because $\Bar B_n$ is a block upper triangular matrix and (iii) is by \eqref{eq:normalized_error_rate}, \eqref{eq:bn_regularity}, \eqref{eq:induction}, and \eqref{eq:en_high_probability}.

We then look into the $L^1$-norm convergence by having:
\begin{align*}
  &  \bbE\bb{\|\hat{\theta}_{n,j}-\theta^*_j\|_2} \lesssim  n^{\frac{\alpha_1}{2}}\bbE\bb{\bn{\Bar B_{n,j,j}(\hat{\theta}_{n,j}-\theta^*_j)}_2\one(\calE_n) }  + \bbP(\calE_n^c)\leq   n^{\frac{\alpha_1}{2}}\bbE\bb{\bn{\Bar B_{n,j,j}(\hat{\theta}_{n,j}-\theta^*_j)}_2}  + \bbP(\calE_n^c)\\
    &\leq  n^{\frac{\alpha_1}{2}}\bbE[\|\Bar B_{n,j,j}(\hat{\theta}_{n,j}-\theta^*_j) + \sum_{j'>j}\Bar B_{n,j,j'}(\hat{\theta}_{n,j'}-\theta^*_{j'}) \|_2 ] + 
    n^{\frac{\alpha_1}{2}}
    \sum_{j'>j} \bbE\bb{\bn{\Bar B_{n,j,j'}(\hat{\theta}_{n,j'}-\theta^*_{j'})}_2 }+ \bbP(\calE_n^c)\\
    &\stackrel{(i)}{\leq} n^{\frac{\alpha_1}{2}}\bbE\bb{\bn{\Bar B_{n}(\hat{\theta}_{n}-\theta^*)}_2 }+
    n^{\frac{\alpha_1}{2}}
    \sum_{j'>j} \bbE\bb{\bn{\Bar B_{n,j,j'}(\hat{\theta}_{n,j'}-\theta^*_{j'})}_2 }+ \bbP(\calE_n^c)\\
    &=  n^{\frac{\alpha_1}{2}} \bbE\bb{\bn{\frac{1}{n}\sum_{i=1}^n H_i\bar{J}_i \bp{\hat{\theta}_n-\theta^*}  
    }_{2}} +
    n^{\frac{\alpha_1}{2}}
    \sum_{j'>j} \bbE\bb{\bn{\Bar B_{n,j,j'}(\hat{\theta}_{n,j'}-\theta^*_{j'})}^2_2 }+ \bbP(\calE_n^c)\\
     &\stackrel{(ii)}{\leq}   O(n^{\frac{\alpha_1+\alpha_2-1}{2}}) +
    n^{\frac{\alpha_1}{2}}
    \sum_{j'>j} \bbE\bb{\bn{B^0_{n,j,j'}(\hat{\theta}_{n,j'}-\theta^*_{j'})}_2 }+
    n^{\frac{\alpha_1}{2}}
    \sum_{j'>j} \bbE\bb{\bn{\delta_{n,j,j'}(\hat{\theta}_{n,j'}-\theta^*_{j'})}_2 } + O(n^{\alpha_1+\alpha_2-1}) \\
    &\stackrel{(iii)}{\leq}  O(n^{\frac{\alpha_1+\alpha_2-1}{2}})  +
    n^{\frac{\alpha_1}{2}}
    \sum_{j'>j} \bbE\bb{\bn{B^0_{n,j,j'}}_{Frob}\bn{(\hat{\theta}_{n,j'}-\theta^*_{j'})}_2 }+
    n^{\frac{\alpha_1}{2}}
    \sum_{j'>j} \bbE\bb{\bn{\delta_{n,j,j'}}^2_{Frob}}^{\frac{1}{2}}\bbE\bb{\bn{(\hat{\theta}_{n,j'}-\theta^*_{j'})}^2_2 }^{\frac{1}{2}} \\
   &\stackrel{(iv)}{=} O(n^{\frac{\alpha_1+\alpha_2-1}{2}})  +     n^{\frac{\alpha_1}{2}}\sum_{j'>j} O(n^{\frac{\alpha_2}{2}+\frac{(L-j'+1)(\alpha_1+\alpha_2)-1}{2}}) + 
        n^{\frac{\alpha_1}{2}}
        \sum_{j'>j} O(n^{\frac{\alpha_2-1}{2} + \frac{(2L-2j'+1)\alpha_1+\alpha_2-1}{2}})\\
    &= O(n^{\frac{\alpha_1+\alpha_2-1}{2}})  +  O(n^{\frac{(L-j+1)(\alpha_1+\alpha_2)-1}{2}}) + O(n^{\frac{(2L-2j)\alpha_1+2\alpha_2-2}{2}})
    \\
    &\stackrel{(v)}{=}O(n^{\frac{(L-j+1)(\alpha_1+\alpha_2)-1}{2}}).
\end{align*}
where (i) uses that $\bar B_n$ is a block upper triangular matrix, (ii) uses \eqref{eq:normalized_error_rate}, triangular inequality, and \eqref{eq:en_high_probability}, (iii) uses Cauchy-Schwartz inequality and that a matrix $\ell_2$ norm is bounded by its Frobenius norm, (iv) uses the induction assumption \eqref{eq:induction}, and (v) uses that $\alpha_2\leq \alpha_1\leq 1/L$ as specified in Property \ref{property:weight_regularizing}.

\subsubsection{Asymptotic regularity of $\Bar{B}_n$.}
\label{appendix:regularity_Bn}
Note that $\Bar{B}_n$ is block upper triangular, where the column sizes of blocks  correspond to the decomposition of $\theta=(\theta_0, \theta_1, \dots, \theta_L)$. For any $j\leq j'$ we have:
 \begin{align*}
  \Bar B_{n,j,j'}=n^{-1}\sum_{i=1}^n H_{i,j}\bbE_{i,j}^+\bb{(\Psi_{i,j}-\bar{\Psi}_{i,j})\Phi_{i,j'}^\top} .
\end{align*}
Define the matrix $B_n^0$ as follows: for each of its $(j,j')$-th block with $j\leq j'$, define
\begin{align*}
        B^0_{n,j,j'}=n^{-1}\sum_{i=1}^n \bbE_{i,j}\bb{H_{i,j}(\Psi_{i,j}-\bar{\Psi}_{i,j})\Phi_{i,j'}^\top} ,\mbox{ and }\delta_{n,j,j'}=B^0_{n,j,j'}-\Bar B_{n,j,j'}.
\end{align*}
First, notice that the singular values of the $j$-th diagonal block $B^0_{n,j,j}$ are lower-bounded when Property \ref{property:weight_regularizing}(a) holds. 
For the $0$-th diagonal block corresponding to $\theta_0^*$, we have   $B^0_{n,0,0}= n^{-1}\sum_{i=1}^n \bbE_{i,0}[H_{i,0}] =1\geq c_1 n^{-\frac{\alpha_1}{2}}$ for $H_{i,0}=1$ and large $n$.

Next we shall show that for $j\leq j'$, we have
  \begin{align*}
  \bbE\bb{\bn{\delta_{n,j,j'}}_{Frob}^2}=O\bp{n^{\alpha_2-1}},  \, \|B^0_{n,j,j'}\|^2_{Frob}=O\bp{n^{\alpha_2}}, \, \|\Bar B_{n,j,j'}\|^2_{Frob}=O\bp{n^{\alpha_1}},
  \end{align*}
  and that the event $\calE_n$ happens with high probability:
  \begin{align*}
      1 - \bbP(\calE_n)=O(Ln^{\alpha_1+\alpha_2-1}).
  \end{align*}

\paragraph{Asymptotic neglibility of $\delta_{n,j,j'}$.}  For a matrix $A$, we use $\Tr(A)$ to denote its trace, i.e. the sum of its diagonal entries. With $j\leq j'$, by the definition of the Frobenius norm, we have
\begin{align*}
    \bbE[\|\delta_{n,j,j'}\|_{Frob}^2 ]=~& \bbE\bb{\Tr\bp{\delta_{n,j,j'}\, \delta_{n,j,j'}^\top}}\\
    =~& n^{-2}\bbE\Big[\Tr\Big(\bc{\sum_{i=1}^n \bp{H_{i,j}\bbE_{i,j}^+[(\Psi_{i,j}-\bar{\Psi}_{i,j})\Phi_{i,j'}^\top] -\bbE_{i,j}[H_{i,j}(\Psi_{i,j}-\bar{\Psi}_{i,j})\Phi_{i,j'}^\top ]}}\\
  &\quad\quad  \times\bc{\sum_{i=1}^n \bp{H_{i,j}\bbE_{i,j}^+[(\Psi_{i,j}-\bar{\Psi}_{i,j})\Phi_{i,j'}^\top] -\bbE_{i,j}[H_{i,j}(\Psi_{i,j}-\bar{\Psi}_{i,j})\Phi_{i,j'}^\top ]}}^\top  \Big)\Big]\\
  \stackrel{(i)}{=}&
  n^{-2}\sum_{i=1}^n\Tr\Big( \bbE\Big[\bp{H_{i,j}\bbE_{i,j}^+[(\Psi_{i,j}-\bar{\Psi}_{i,j})\Phi_{i,j'}^\top] -\bbE_{i,j}[H_{i,j}(\Psi_{i,j}-\bar{\Psi}_{i,j})\Phi_{i,j'}^\top ]}\\
  &\quad \quad \times \bp{H_{i,j}\bbE_{i,j}^+[(\Psi_{i,j}-\bar{\Psi}_{i,j})\Phi_{i,j'}^\top] -\bbE_{i,j}[H_{i,j}(\Psi_{i,j}-\bar{\Psi}_{i,j})\Phi_{i,j'}^\top ]}^\top  
  \Big]\Big)\\
    \stackrel{(ii)}{\leq} &  n^{-2}\sum_{i=1}^n\Tr \bp{
\bbE\bb{ H_{i,j} \bbE^+_{i,j}[(\Psi_{i,j}-\bar{\Psi}_{i,j})\Phi_{i,j'}^\top] \bbE^+_{i,j}[\Phi_{i,j'}(\Psi_{i,j}-\bar{\Psi}_{i,j})^\top]  H_{i,j}^\top }}\\
  \stackrel{(iii)}{\lesssim} & n^{-2}d\sum_{i=1}^n\Tr \bp{
\bbE\bb{ H_{i,j} \bbE^+_{i,j}\bb{(\Psi_{i,j}-\bar{\Psi}_{i,j})(\Psi_{i,j}-\bar{\Psi}_{i,j})^\top } H_{i,j}^\top }
  } \stackrel{(iv)}{=} O(n^{\alpha_2-1})
\end{align*}
where (i) uses the fact  that $\delta_{n,j,j'}$ is a sum of martingale difference sequence, (ii) uses Lemmas \ref{lemma:trace_inequality_3} \& \ref{lemma:trace_inequality_4}, (iii)  uses that $\|\Phi_{i,j'}\|_2$ is  bounded, Lemma  \ref{lemma:matrix_inequality} and Lemma  \ref{lemma:trace_inequality_3}, and (iv) uses Property \ref{property:weight_regularizing}(c).

\paragraph{Uniform bound of $B^0_{n,j,j'}$.} With $j\leq j'$, we have
\begin{align*}
    \|B^0_{n,j,j'}\|_{Frob}^2=&n^{-2}\bn{
    \sum_{i=1}^n \bbE_{i,j}[H_{i,j}(\Psi_{i,j}-\bar{\Psi}_{i,j})\Phi_{i,j'}^\top ]}^2_{Frob}
    \\
  \stackrel{(i)}{\leq}  &n^{-2}
    \bp{\sum_{i=1}^n \bn{\bbE_{i,j}[H_{i,j}(\Psi_{i,j}-\bar{\Psi}_{i,j})\Phi_{i,j'}^\top ]}_{Frob}}^2
   \\
  \stackrel{(ii)}{\leq}  &n^{-1}
    \sum_{i=1}^n \bn{\bbE_{i,j}[H_{i,j}(\Psi_{i,j}-\bar{\Psi}_{i,j})\Phi_{i,j'}^\top ]}^2_{Frob}
   \\
  =  &n^{-1}
    \sum_{i=1}^n \Tr\bp{\bbE_{i,j}[H_{i,j}(\Psi_{i,j}-\bar{\Psi}_{i,j})\Phi_{i,j'}^\top ]\bbE_{i,j}[\Phi_{i,j'}(\Psi_{i,j}-\bar{\Psi}_{i,j})^\top H_{i,j}^\top] }
   \\
\stackrel{(iii)}{\lesssim}  &n^{-1}
    \sum_{i=1}^n \Tr\bp{\bbE_{i,j}[H_{i,j}(\Psi_{i,j}-\bar{\Psi}_{i,j})(\Psi_{i,j}-\bar{\Psi}_{i,j})^\top H_{i,j}^\top] }
 \\
    =  &
    n^{-1}\sum_{i=1}^n \Tr\bp{\bbE_{i,j}[H_{i,j}\Var_{i,j}^+(\Psi_{i,j})
    H_{i,j}^\top] } = O(n^{\alpha_2}).
\end{align*}
where (i) is by triangular inequality, (ii) is by Cauchy-Schwartz inequality, (iii) is by the assumption that $\|\Phi_{i,j'}\|_2$ is bounded,  Lemma \ref{lemma:matrix_inequality} and Lemma \ref{lemma:trace_inequality_2}.

\paragraph{Uniform bound of $\Bar B_{n,j,j'}$.} With $j\leq j'$, we have
\begin{align*}
    \|\Bar B_{n,j,j'}\|_{Frob}^2=&n^{-2}\bn{
    \sum_{i=1}^n H_{i,j}\bbE_{i,j}^+[(\Psi_{i,j}-\bar{\Psi}_{i,j})\Phi_{i,j'}^\top] }^2_{Frob}
  \stackrel{(i)}{\leq}  n^{-1}
    \sum_{i=1}^n \bn{H_{i,j}\bbE_{i,j}^+[(\Psi_{i,j}-\bar{\Psi}_{i,j})\Phi_{i,j'}^\top] }^2_{Frob}\\
     \leq &n^{-1}
    \sum_{i=1}^n \Tr\bp{H_{i,j}\bbE_{i,j}^+[(\Psi_{i,j}-\bar{\Psi}_{i,j})\Phi_{i,j'}^\top]\bbE_{i,j}^+[ \Phi_{i,j'}(\Psi_{i,j}-\bar{\Psi}_{i,j})^\top] H_{i,j}^\top }
   \\
\stackrel{(ii)}{\lesssim}  &n^{-1}
    \sum_{i=1}^n \Tr\bp{H_{i,j}H_{i,j}^\top } = O(n^{\alpha_1}).
\end{align*}
where (i) is by triangular inequality and Cauchy-Schwartz inequality, (ii) is due to the fact that $\Phi_{i,j'}, \Psi_{i,j},\bar{\Psi}_{i,j}$ are bounded such that 
\[
H_{i,j}\bbE_{i,j}^+[(\Psi_{i,j}-\bar{\Psi}_{i,j})\Phi_{i,j'}^\top] \bbE_{i,j}^+[\Phi_{i,j'}(\Psi_{i,j}-\bar{\Psi}_{i,j})^\top] H_{i,j}^\top\precsim H_{i,j}H_{i,j}^\top;
\]
then with Lemma \ref{lemma:trace_inequality_2} we have (ii).

\paragraph{High probability event $\calE_n$.} 
We finally show that $\bbP(\calE^c_n)=O(Ln^{\alpha_1+\alpha_2-1})$. When $\calE_n$ does not happen, there exists a $j\in\{0,1,\dots, L\}$ and an  
eigenvector $x$ of $\Bar B_{n,j,j}^\top \Bar B_{n,j,j}$, with $\|x\|_2=1$ such that $\Bar B_{n,j,j}^\top \Bar B_{n,j,j} x = \lambda x$ and $\lambda<c_1^2n^{-\alpha_1}/4$. In that case:
\begin{equation}
    \label{eq:ct_ex}
    x^\top \Bar B_{n,j,j}^\top \Bar B_{n,j,j} x < \frac{c_1^2 n^{-\alpha_1}}{4}.
\end{equation}
Since $\delta_{n,j,j}=\Bar B_{n,j,j}-B_{n,j,j}^0$,  we have that:
\begin{align}
    \frac{c^2n^{-\alpha_1}}{4} >  x^\top \Bar B_{n,j,j}^\top \Bar B_{n,j,j} x \stackrel{(i)}{\geq}  \frac{1}{2} x^\top (B_{n,j,j}^0)^\top B_{n,j,j}^0 x - 2\,\|\delta_{n,j,j}\|_{Frob}^2 \stackrel{(ii)}{\geq} \frac{c_1^2n^{-\alpha_1}}{2} -2\,\|\delta_{n,j,j}\|_{Frob}^2
\end{align}
where (i) is  by Lemma \ref{lem:lower-eigenvalue}, and (ii) is because the minimum  eigenvalue value of $(B_{n,j,j}^0)^\top B_{n,j,j}^0$ is at least $c_1^2n^{-\alpha}$ by Property \ref{property:weight_regularizing}(a) and the vector $x$ is unit-norm. Rearranging yields:
\begin{align}
    \|\delta_{n,j,j}\|_{Frob}^2 \geq \frac{c_1^2n^{-\alpha_1}}{8}
\end{align}
Thus we have, 
by Markov's inequality:
\begin{align*}
    \bbP(\calE^c_n)&\leq \bbP\bp{ \exists j, \|\delta_{n,j,j}\|_{Frob}^2 \geq \frac{c_1^2n^{-\alpha_1}}{8}}\leq \sum_{j=0}^L\bbP\bp{ \|\delta_{n,j,j}\|_{Frob}^2 \geq \frac{c_1^2n^{-\alpha_1}}{8}}\\
    &\leq \sum_{j=0}^L\frac{8n^{\alpha_1}}{c_1^2}\bbE\bb{\|\delta_{n,j,j}\|_{Frob}^2}=O(Ln^{\alpha_1+\alpha_2-1}).
\end{align*}

\section{Proof of Theorem \ref{thm:be_feasible_2}}
\label{appendix:be_feasible}

For notation convenience, we write 
$F_{i,j}:=f_i(S_{i,1:j}, T_{i,1:j-1})$ and $\hat{F}_{i,j}:=\hat{f}_{i,j}(S_{i,1:j}, T_{i,1:j-1})$. Also, we let 
$F_i:=(F_{i,1},\dots, F_{i,L})$, and $\hat{F}_i=(\hat{F}_{i,1}, \dots, \hat{F}_{i,L})$. 
Define the variance estimate:  $\widehat{\Xi}_n=n^{-1}\sum_{i=1}^n \diag\bc{W_{i,j}\Var^+_{i,j}(\Psi_{i,j})W_{i,j}^\top}$.

Let $\hat{\theta}_n$ be the AW-GMM estimator defined in \eqref{eq:gmm_estimator_linear}. With assumptions in Theorem \ref{thm:be_feasible_2}, the weights $H_{i,j}$ satisfy Property \ref{property:weight_regularizing} with constants $(\alpha_1, 0)$.
We can then invoke Theorem \ref{thm:consistency} and get:
\begin{equation}
    \bbE[\|\hat{\theta}_{n,j}-\theta^*_j\|_2]=O\bp{n^{\frac{(L-j+1)\alpha_1-1}{2}}}, \forall j\in[0:L].
\end{equation}
Thus by Markov inequality, we have for any positive $\delta$:
\begin{equation*}
    \bbP\bp{\|\hat{\theta}_{n,j}-\theta^*_j\|_2\geq \delta}\leq O\bp{n^{\frac{(L-j+1)\alpha_1-1}{2}}}\Big/\delta.
\end{equation*}
 With $\theta^*$ lying in the interior of $\Theta$ and $\Theta$ being bounded,  for $n$ larger than a constant, with probability at least $1-O(n^{\frac{(L+1)\alpha_1-1}{2}})$, $\hat{\theta}_n$  lies in the interior of $\Theta$. Define $\calE_{int}$ as the event that $\hat{\theta}_n$  lies in the interior of $\Theta$. Then for  $n$ larger than a constant, $\bbP(\calE_{int}^c)=O(n^{\frac{(L+1)\alpha_1-1}{2}})$.

Now let's establish the strong Gaussian approximation when event  $\calE_{int}$ happens.
With $\hat{\theta}_n$ being in the interior of $\Theta$, we can invoke the first-order condition of the GMM solution:
\begin{align}
\tag{\ref{eq:link_mds_error_2}}
 B_n^\top A\bp{\frac{1}{n}\sum_{i=1}^n H_i \xi_i}
  =& B_n^\top A B_n(\hat\theta_n-\theta^*),\quad \mbox{where}\quad B_n:=-\frac{1}{n}\sum_{i=1}^n H_i J_i.
\end{align}
Appendix \ref{appendix:uniform_mds} shows that the sum of MDS $\frac{1}{\sqrt{n}}\sum_{i=1}^n   H_i\xi_{i}$ can be approximated uniformly by a  Gaussian random variable $Z_\xi\sim  \calN\bp{0, \widehat{\Xi}_n}$, such that
\begin{align}
\label{eq:high_prob_gaussian_approximation}
   \sup_{C\in\calC}\left| \bbP\bp{n^{-\frac{1}{2}}\sum_{i=1}^n   H_i\xi_{i}\in C}-\bbP\bp{  Z_\xi \in C}\right|  
    = O  \bp{ L^2\cdot n^{-\min(\frac{1-\alpha_1}{12}, \frac{\alpha_3}{5}, \frac{\gamma}{5})}},
\end{align}
where  $\calC$ is the set of all convex subsets in $\bbR^{1+dL}$.

Therefore for any convex set $\bar C$, 
\begin{equation}
    \bbP\bp{ \bc{\sqrt{n}B_n^\top A B_n (\hat{\theta}_n -\theta)\in \bar  C}\cap \calE_{int}}=\bbP\bp{\bc{n^{-\frac{1}{2}}\sum_{i=1}^n B_n^\top A H_i\xi_{i}\in  \bar C}\cap \calE_{int}} \label{eq:nor_eq_1}.
\end{equation}
Now for the convex set $\bar{C}$, define the induced set $\tilde{C}$ as follows:
\begin{equation}
    \tilde{C} = \{x:B_n^\top A x \in \bar  C\}. \label{eq:nor_eq_2}
\end{equation}
This definition yields:
\begin{align}
    \bbP\bp{n^{-\frac{1}{2}}\sum_{i=1}^n  H_i \xi_{i}\in \tilde{C}} = \bbP\bp{n^{-\frac{1}{2}}\sum_{i=1}^n B_n^\top A H_i\xi_{i}\in \bar C} \quad \mbox{and}\quad 
     \bbP\bp{Z_\xi\in \tilde{C}} = \bbP\bp{ B_n^\top A Z_\xi\in \bar C}. \label{eq:nor_eq_3}
\end{align}
Meanwhile, it's straightforward to see that $\tilde{C}$ is also a convex set; thus by the strong Gaussian approximation in \eqref{eq:high_prob_gaussian_approximation}, we have
\begin{align}
    \left| \bbP\bp{n^{-\frac{1}{2}}\sum_{i=1}^n   H_i \xi_{i}\in \tilde{C}}-\bbP\bp{ Z_\xi\in \tilde{C}}\right|  
    = O  \bp{ L^2\cdot n^{-\min(\frac{1-\alpha_1}{12}, \frac{\alpha_3}{5}, \frac{\gamma}{5})}} \label{eq:nor_eq_4}
\end{align}
Combining \eqref{eq:nor_eq_1}, \eqref{eq:nor_eq_3}, \eqref{eq:nor_eq_4}, we have for any convex set $\bar{C}\in\calC$:
\begin{align}
    &\left| \bbP\bp{ \sqrt{n}B_n^\top A B_n (\hat{\theta}_n -\theta)\in \bar C }-\bbP\bp{ B_n^\top A Z_\xi\in \bar C}\right| \label{eq:normalized_error_gaussian_approximation}\\
    \leq &\left| \bbP\bp{ \bc{\sqrt{n}B_n^\top A B_n (\hat{\theta}_n -\theta)\in \bar C}\cap\calE_{int}}-\bbP\bp{\bc{ B_n^\top A Z_\xi\in \bar C} \cap\calE_{int}}\right| + 2\bbP(\calE_{int}^c) \nonumber\\
    =& \left| \bbP\bp{ \bc{n^{-\frac{1}{2}}\sum_{i=1}^n  B_n^\top A  H_i\xi_i \in \bar C}\cap\calE_{int}}-\bbP\bp{\bc{  B_n^\top AZ_\xi\in \bar C} \cap\calE_{int}}\right|+ 2\bbP(\calE_{int}^c)\nonumber\\
     \leq & \left| \bbP\bp{ n^{-\frac{1}{2}}\sum_{i=1}^n   B_n^\top A H_i\xi_i \in \bar C}-\bbP\bp{ B_n^\top A Z_\xi\in \bar C}\right|+ 4\bbP(\calE_{int}^c)\nonumber\\
     \leq & \left| \bbP\bp{ n^{-\frac{1}{2}}\sum_{i=1}^n  H_i\xi_i \in \tilde C}-\bbP\bp{ Z_\xi\in \tilde C}\right|+ 4\bbP(\calE_{int}^c)
     \tag{by definition of $\tilde C$ in \eqref{eq:nor_eq_2}}
     \nonumber\\
    &=  O  \bp{ L^2\cdot n^{-\min(\frac{1-\alpha_1}{12}, \frac{\alpha_3}{5}, \frac{\gamma}{5})}}+ O\bp{n^{\frac{(L+1)\alpha_1-1}{2}}} =  O  \bp{ L^2\cdot n^{-\min(\frac{1-\alpha_1}{12}, \frac{1-(L+1)\alpha_1}{2}, \frac{\alpha_3}{5}, \frac{\gamma}{5})}} .\nonumber
\end{align}

We finally show that the normalizing matrix $B_n^\top AB_n$
is invertible with high probability.
 Define   event $\calE^M_n$ for  $B_n=-n^{-1}\sum_{i=1}^n H_iJ_i$:
\begin{align*}
      \calE^M_n = \left\{B_{n,j,j}^\top B_{n,j,j} \succeq \frac{c_1^2}{4M^2} n^{-\alpha_1} I, \  \forall l\in[0: L]\right\}.
\end{align*}
 Appendix \ref{appendix:regularity_B_normal} shows that $\calE^M_n$ happens with probability at least  $1-O(Ln^{\alpha_1-1})$. When  $\calE^M_n$ happens, each $B_{n,j,j}$ has full column rank. Since $B_n$ is a block upper triangular matrix, $B_n$ also has full column rank. As a result, $B_n^\top A B_n$, which has the same rank as $B_n$ (see Lemma \ref{lemma:col_rank_matrix}), is thus invertible.

Now for any convex set $C\in\calC$, define the induced set
 \begin{equation}
 \label{eq:induce_c_bar}
  \bar C:=\{x: (B_n^\top A B_n)^\dagger x\in C\},   
 \end{equation}
 which is also a convex set by construction. We have:
 \begin{align*}
      &\left| \bbP\bp{ \sqrt{n}(\hat{\theta}_n -\theta)\in C }-\bbP\bp{(B_n^\top A B_n)^\dagger B_n^\top A Z_\xi\in C}\right| \\
    \leq &\left| \bbP\bp{ \bc{\sqrt{n}(\hat{\theta}_n -\theta)\in  C}\cap  \calE^M_n}-\bbP\bp{\bc{(B_n^\top A B_n)^\dagger B_n^\top A Z_\xi\in  C} \cap \calE^M_n}\right| + 2\bbP((\calE^M_n)^c) \\
    = &\left| \bbP\bp{ \bc{\sqrt{n}(B_n^\top A B_n)^{-1}B_n^\top A B_n(\hat{\theta}_n -\theta)\in C}\cap  \calE^M_n}-\bbP\bp{\bc{(B_n^\top A B_n)^\dagger B_n^\top A Z_\xi\in C} \cap \calE^M_n}\right|\\
    &\quad \quad + 2\bbP((\calE^M_n)^c) 
    \tag{When $\calE^M_n$ happens, $B_n^\top A B_n$ is invertible. }
    \\
     = &\left| \bbP\bp{ \bc{\sqrt{n}(B_n^\top A B_n)^{\dagger}B_n^\top A B_n(\hat{\theta}_n -\theta)\in C}\cap  \calE^M_n}-\bbP\bp{\bc{(B_n^\top A B_n)^{\dagger}B_n^\top A Z_\xi\in C} \cap \calE^M_n}\right| + 2\bbP((\calE^M_n)^c) 
        \tag{When $\calE^M_n$ happens, $B_n^\top A B_n$ is invertible and $(B_n^\top A B_n)^{-1}=(B_n^\top A B_n)^\dagger$. }\\
    = & \left| \bbP\bp{ \bc{\sqrt{n}B_n^\top A B_n(\hat{\theta}_n -\theta)\in\bar C}\cap  \calE^M_n}-\bbP\bp{\bc{ B_n^\top A Z_\xi\in \bar C} \cap \calE^M_n}\right| + 2\bbP((\calE^M_n)^c) 
    \tag{by definition of $\bar C$ in \eqref{eq:induce_c_bar}}
    \\
       = & \left| \bbP\bp{ \sqrt{n}B_n^\top A B_n(\hat{\theta}_n -\theta)\in\bar C}-\bbP\bp{B_n^\top A Z_\xi\in \bar C}\right| + 4\bbP((\calE^M_n)^c)  \\
    = & O  \bp{ L^2\cdot n^{-\min(\frac{1-\alpha_1}{12}, \frac{1-(L+1)\alpha_1}{2}, \frac{\alpha_3}{5}, \frac{\gamma}{5})}} + O(Ln^{\alpha_1-1})\tag{by \eqref{eq:normalized_error_gaussian_approximation}}\\
    =& O  \bp{ L^2\cdot n^{-\min(\frac{1-\alpha_1}{12}, \frac{1-(L+1)\alpha_1}{2}, \frac{\alpha_3}{5}, \frac{\gamma}{5})}}.
 \end{align*}
 This concludes our proof for Theorem \ref{thm:be_feasible_2}.

\subsection{Strong Gaussian Approximation of $n^{-\frac{1}{2}}\sum_{i=1}^n  H_i\xi_i$}
\label{appendix:uniform_mds}

We show the strong Gaussian approximation results of $n^{-\frac{1}{2}}\sum_{i=1}^n  H_i \xi_i$. To do it, we follow three steps:
\begin{enumerate}
    \item Connect $n^{-\frac{1}{2}}\sum_{i=1}^n  H_i\xi_i$ with $\calN(0, \Sigma_n)$, where recall that 
    \[
    \Sigma_n = \diag\bc{\Sigma_{n,0}, \Sigma_{n,1}, \dots, \Sigma_{n,L}}, \quad \mbox{where} \quad \Sigma_{n,j}=n^{-1}\sum_{i=1}^n\bbE\bb{\frac{F_{i,j}}{\hat{F}_{i,j}} 
 W_{i,j}\Var^+_{i,j}(\Psi_{i,j})W_{i,j}^\top
 }.
    \]
    \item Connect $\calN(0, \Sigma_n)$ with $\calN(0, \Xi_n)$, where recall that 
    \[
    \Xi_n=\diag\bc{
    \Xi_{n,0}, \Xi_{n,1},\dots \Xi_{n,L}
   },\quad \mbox{where}\quad \Xi_{n,j}= n^{-1}\bbE\bb{\sum_{i=1}^n W_{i,j}\Var^+_{i,j}(\Psi_{i,j})W_{i,j}^\top},
    \]
    and by Property \ref{property:weight_regularizing}(b), we have
    \[
    \Xi_n \succeq C_2 \cdot I.
    \]
    \item Connect $\calN(0, \Xi_n)$ with $\calN(0, \widehat{\Xi}_n)$.\footnote{For any almost surely positive semi-definite random matrix $\Sigma$, we use $\calN(0,\Sigma)$
to define  the distribution of $X:=\Sigma^{1/2}Z$ with $Z\sim \calN(0,1)$ independent of $\Sigma$.} Here we define $\widehat{\Xi}_n$ as 
    \[
  \widehat{\Xi}_n:= 
  \diag\bc{
      \widehat\Xi_{n,0}, \  \widehat\Xi_{n,1},\dots   \widehat\Xi_{n,L}
   },\quad \mbox{where}\quad   \widehat\Xi_{n,j}=
  n^{-1}\diag\bc{\sum_{i=1}^n W_{i,j}\Var^+_{i,j}(\Psi_{i,j})W_{i,j}^\top}.
    \]

\end{enumerate}
Collectively, we shall be able to connect $n^{-1}\sum_{i=1}^n   H_i\xi_i$ with $\calN(0, \widehat{\Xi}_n)$, which achieves the  inferential result we target.

\medskip

\subsubsection{Step 1: Connect $n^{-\frac{1}{2}}\sum_{i=1}^n   H_i\xi_i$  with $\calN(0, \Sigma_n)$.} 
Recall that  $\{  H_i\xi_i\}_{i=1}^n$ is a  martingale difference sequence. To characterize its asymptotic behavior, we shall leverage the following proposition.
\begin{restatable}[Strong Gaussian Approximation for Martingale Vectors, \cite{cattaneo2022yurinskii}]{proposition}{cattaneo}
\label{prop:uniform_clt}
Let $\{\varphi_t\}_{t=1}^T $ be $\mathbb{R}^d$-valued squared integrable martingale difference sequence adapted to $\{\calG_t\}_{t=1}^T$. Define $v_t:=\Var_{t}(\varphi_t) = \bbE[\varphi_t \varphi_t^\top\mid \calG_{t-1}]$. Define $S_T:=\sum_{t=1}^T \varphi_t$, and let $\Sigma_T:=\frac{1}{T}\sum_{t=1}^T\bbE[v_t]$ and $\Omega_T:=\sum_{t=1}^T (v_t-\bbE[v_t])$. Then  there exists a $W_T \sim \calN(0, \Sigma_T)$ such that
\begin{equation*}
    \sup_{C\in\calC}|\bbP(T^{-1/2}S_T\in C)-\bbP(W_T\in C)| \lesssim \inf_{\eta>0} \bc{
    \bp{\frac{\beta_2d}{\eta^3}}^{1/3} +  \bp{\frac{\sqrt{d\bbE[\|\Omega_T\|_2]}}{\eta}}^{2/3} +\Delta_2(\calC, \eta)
    },
\end{equation*}
where:
\begin{itemize}
    \item $\calC$ is a class of measurable subsets of $\bbR^d$.
    \item $\Delta_p(\calC, \eta)$ defines the Gaussian perimetric (anti-concentration) quantity \[
\Delta_2(\calC, \eta) = \sup_{C\in\calC}\bc{\bbP(W_T\in C_2^\eta\setminus C)\vee \bbP(W_T\in  C\setminus C_2^{-\eta})},
    \]
    with $C_2^\eta=\bc{x\in\bbR^d:\|x-C\|_2\leq \eta}, C^{-\eta}_2=\bbR^d\setminus(\bbR^d\setminus C)_2^\eta$ and $\|x-C\|_2=\inf_{x'\in C}\|x-x'\|_2$.
    \item $\beta_2=\sum_{t=1}^T  \bbE\bb{
    \|\varphi_t\|^3_2 + \|v_t^{1/2}Z_t\|_2^3
    }$, with $\{Z_t\}_{t=1}^T$ being i.i.d.~standard Gaussian variables on $\mathbf{R}^d$ independent of $\calF_T$.
\end{itemize}
When $\Sigma_T$ is invertible, and  $\calC$ represents the set of all convex measurable subsets of $\mathbf{R}^d$, we have:
\begin{equation}
\label{eq:strong_gaussian_cattaneo}
    \sup_{C\in\calC}|\bbP(T^{-1/2}S_T\in C)-\bbP(W_T\in C)| \lesssim \inf_{\eta>0} \bc{
    \bp{\frac{\beta_2d}{\eta^3}}^{1/3} +  \bp{\frac{\sqrt{d\bbE[\|\Omega_T\|_2]}}{\eta}}^{2/3} +\eta\sqrt{\|T^{-1}\Sigma^{-1}_T\|_F}
    }.
\end{equation}
\end{restatable}

\begin{remark}
    Proposition \ref{prop:uniform_clt} is adapted from Proposition A.1 in \cite{cattaneo2022yurinskii} by setting the vector norm parameter $p=2$. 
\end{remark}

We now write out the terms in Proposition \ref{prop:uniform_clt} regarding $\{  H_i\xi_i\}_{i=1}^{n}$:
\begin{itemize}
    \item $ v_i = \bbE_{i}\bb{ H_i\xi_i\xi_i^\top  H_i^\top}$, whose $j$-th diagonal block is $\bbE_{i}\bb{\frac{F_{i,j}}{\hat{F}_{i,j}} 
W_{i,j}\Var_{i,j}^+(\Psi_{i,j})W_{i,j}^\top
 }$.\footnote{Recall that by Lemma \ref{lemma:mds} this $v_i$ is block-diagonal.}
 \item $\Sigma_{n} =n^{-1}\bbE\bb{ \sum_{i=1}^{n} v_i }$, which is a block-diagonal matrix with its $j$-th block being $ n^{-1}\sum_{i=1}^n\bbE\bb{\frac{F_{i,j}}{\hat{F}_{i,j}} 
 W_{i,j}\Var^+_{i,j}(\Psi_{i,j})W_{i,j}^\top}$.
 \item $\Omega_{n}=\sum_{i=1}^{n}(v_i-\bbE[v_i])$, which is a block-diagonal matrix with its $j$-th block being 
 \[
 \begin{split}
  &\sum_{i=1}^{n}
 \bc{\bbE_{i}\bb{\frac{F_{i,j}}{\hat{F}_{i,j}} 
W_{i,j}\Var_{i,j}^+(\Psi_{i,j})W_{i,j}^\top
 } -  \bbE\bb{\frac{F_{i,j}}{\hat{F}_{i,j}} 
W_{i,j}\Var_{i,j}^+(\Psi_{i,j})W_{i,j}^\top
 }}.
 \end{split}
 \]
\end{itemize}

\paragraph{Regularity of $\Sigma_{n}$.} We first lower bound the eigenvalues of $\Sigma_{n}$. Note that by construction, $\Sigma_{n}$ is a symmetric block-diagonal matrix, and thus we only need to look at its $j$-th diagonal block $\Sigma_{n,j,j}$. Define $\Xi_{n} := n^{-1}\diag_j\bc{\bbE\bb{\sum_{i=1}^n W_{i,j}\Var^+_{i,j}(\Psi_{i,j})W_{i,j}^\top}}$, and so $\Xi_n$ is also a block-diagonal and symmetric matrix. We have
\begin{align*}
   &\bn{ \Sigma_{n,j,j} -\Xi_{n,j,j}}_{Frob}^2  = 
   \bn{\frac{1}{n}\bp{\sum_{i=1}^n
    \bbE\bb{\frac{F_{i,j}-\hat{F}_{i,j}}{\hat{F}_{i,j}} W_{i,j}\Var^+_{i,j}(\Psi_{i,j})W_{i,j}^\top} }}_{Frob}^2
   \\
   &\leq Ld
   \bp{ \frac{1}{n}\sum_{i=1}^n
    \bbE\bb{\ba{\frac{F_{i,j}-\hat{F}_{i,j}}{\hat{F}_{i,j}}} \Tr \bc{W_{i,j}\Var^+_{i,j}(\Psi_{i,j})W_{i,j}^\top}}}^2
   \tag{by Lemma \ref{lemma:sum_scalar_trace}}\\
   & \leq Ld \bp{ \frac{1}{n\sigma^2}\sum_{i=1}^n  \bbE\bb{
    \ba{F_{i,j}-\hat{F}_{i,j}}\Tr\bc{ W_{i,j}\Var^+_{i,j}(\Psi_{i,j})W_{i,j}^\top}
    }}^2
    \tag{$\hat F_{i,j}\geq \sigma^2$ by condition \eqref{eq:f_convergence}}\\
    &\stackrel{(iii)}{\lesssim} \bp{ \frac{1}{n\sigma^2}\sum_{i=1}^n  \bbE\bb{
    \ba{F_{i,j}-\hat{F}_{i,j}}
    }}^2 
    \tag{
    $\Tr(W_{i,j}\Var^+_{i,j}(\Psi_{i,j})W_{i,j})=O(1)$ by 
     Property \ref{property:weight_stabilizing}(c)
    }
    \\
   & = O(n^{-2\gamma_F})\tag{by convergence condition in \eqref{eq:f_convergence}}.
\end{align*}
Thus we have
\begin{equation}
    \label{eq:diff_sigman_xin}
    \bn{ \Sigma_{n} -\Xi_{n}}_{Frob}^2 = O(Ln^{-2\gamma_F}).
\end{equation}
Together with Property \ref{property:weight_stabilizing}(c) that says that $\Xi_{n}\succeq c_6\cdot I$ and Lemma \ref{lem:lower-eigenvalue}, we have that 
\begin{align}
\label{eq:invertability_of_sigman}
  & \Sigma_{n}^2\succeq \bc{\frac{c_2^2}{2}-O(Ln^{-2\gamma_F})}\cdot I.
\end{align}
Therefore, for large enough $n$, $\Sigma_{n} \succeq \frac{c_2}{2}I$ is invertible. Moreover, we have:
\begin{equation}
\label{eq:bounded_sigman_inverse_frob}
    \|\Sigma_{n}^{-1}\|_{Frob}\leq \sqrt{Ld} \|\Sigma_{n}^{-1}\|_{2}\leq  \frac{2\sqrt{Ld}}{c_2}=O(1).
\end{equation}

\paragraph{Magnitude of $\Omega_{n}$.}
Note that $\Omega_{n}$ by definition is block-diagonal and symmetric.
\begin{align}
    & \frac{1}{n}\bbE\bb{\bn{\Omega_{n}}_2}     \label{eq:omega_n_bound}\\&\leq~ \max_{j=1}^L \frac{1}{n}
    \bbE\bb{\left\|
\sum_{i=1}^{n}
 \bc{\bbE_{i}\bb{\frac{F_{i,j}}{\hat{F}_{i,j}} 
W_{i,j}\Var_{i,j}^+(\Psi_{i,j})W_{i,j}^\top
 } -  \bbE\bb{\frac{F_{i,j}}{\hat{F}_{i,j}} 
W_{i,j}\Var_{i,j}^+(\Psi_{i,j})W_{i,j}^\top
 }}
    \right\|_2}    \nonumber\\
    &\leq ~\max_{j=1}^L \frac{1}{n}
    \bbE\bb{\left\|
    \sum_{i=1}^{n}\bc{\bbE_{i}\bb{\frac{F_{i,j} - \hat{F}_{i,j}}{\hat{F}_{i,j}}W_{i,j}\Var_{i,j}^+(\Psi_{i,j})W_{i,j}^\top }}
    \right\|_2}\nonumber\\
    &\quad + \max_{j=1}^L \frac{1}{n}\left\|\sum_{i=1}^{n}
   \bbE\bb{\frac{F_{i,j} - \hat{F}_{i,j}}{\hat{F}_{i,j}}W_{i,j}\Var^+_{i,j}(\Psi_{i,j})W_{i,j}^\top}
    \right\|_2    \nonumber\\
    &\quad + \max_{j=1}^L \frac{1}{n} \bbE
    \left\|
    \sum_{i=1}^{n} \bc{\bbE_{i}[W_{i,j}\Var^+_{i,j}(\Psi_{i,j})W_{i,j}^\top] - \bbE[W_{i,j}\Var^+_{i,j}(\Psi_{i,j})W_{i,j}^\top]}
    \right\|_2    \nonumber\\
    &\lesssim ~\max_{j=1}^L \frac{1}{n}
    \bbE\bb{
     \sum_{i=1}^{n} \bbE_{i}\bb{\ba{\frac{F_{i,j} - \hat{F}_{i,j}}{\hat{F}_{i,j}}} \Tr\bc{W_{i,j}\Var^+_{i,j}(\Psi_{i,j})W_{i,j}^\top }}
    }\nonumber\\
    &\quad+ \max_{j=1}^L \frac{1}{n}\sum_{i=1}^{n}
   \bbE\bb{\ba{\frac{F_{i,j} - \hat{F}_{i,j}}{\hat{F}_{i,j}}} \Tr\bc{W_{i,j}\Var^+_{i,j}(\Psi_{i,j})W_{i,j}^\top}}
   \tag{by Lemma \ref{lemma:sum_scalar_trace}}
      \nonumber \\
    &\quad + \max_{j=1}^L \frac{1}{n} \bbE
    \left\|
   \sum_{i=1}^{n} \bc{\bbE_{i}[W_{i,j}\Var^+_{i,j}(\Psi_{i,j})W_{i,j}^\top] - \bbE[W_{i,j}\Var^+_{i,j}(\Psi_{i,j})W_{i,j}^\top]}
    \right\|_2    \nonumber\\
&\lesssim  ~\max_{j=1}^L \frac{2}{n}
    \bbE\bb{
     \sum_{i=1}^{n} \bbE_{i}\bb{\ba{F_{i,j} - \hat{F}_{i,j}}}
    }   
    \tag{$\Tr\bc{W_{i,j}\Var^+_{i,j}(\Psi_{i,j})W_{i,j}^\top}=O(1)$ by Property \ref{property:weight_stabilizing}(c) and  $\hat F_{i,j}\geq \sigma^2$ by construction}
    \nonumber\\
    &\quad + \max_{j=1}^L \frac{1}{n} \bbE
    \left\|
   \sum_{i=1}^{n} \bc{\bbE_{i}[W_{i,j}\Var^+_{i,j}(\Psi_{i,j})W_{i,j}^\top] - \bbE[W_{i,j}\Var^+_{i,j}(\Psi_{i,j})W_{i,j}^\top]}
    \right\|_2
    \nonumber
    \\   
    & \leq O(n^{-\gamma_F}) + 
\max_{j=1}^L \frac{1}{n} 
   \sum_{i=1}^{n} 
   \bbE
    \left\|\bbE_{i}[W_{i,j}\Var^+_{i,j}(\Psi_{i,j})W_{i,j}^\top] - \bbE[W_{i,j}\Var^+_{i,j}(\Psi_{i,j})W_{i,j}^\top]
    \right\|_2
    \tag{by convergence of $\hat F_{i,j}$ and triangular inequality}\\
    & \leq O(n^{-\gamma_F}) + 
\max_{j=1}^L \frac{1}{n} 
   \sum_{i=1}^{n} 
   \bbE
    \left\|\bbE_{i,j}[W_{i,j}\Var^+_{i,j}(\Psi_{i,j})W_{i,j}^\top] - \bbE[W_{i,j}\Var^+_{i,j}(\Psi_{i,j})W_{i,j}^\top]
    \right\|_2
    \tag{by Jensen's inequality}\\
    &= 
     ~O(n^{-\min(\gamma_F, \alpha_3)})
     \tag{by Property \ref{property:weight_stabilizing}(c)}.
     \nonumber
 \end{align}

\paragraph{Regularity of $\beta_2$.} We move onto discussing $\beta_2$ defined in Proposition \ref{prop:uniform_clt} and have that
\begin{align}
    \beta_2 =~& \sum_{i=1}^{n} \bbE\bb{\| H_i\xi_i\|^3_2 + \|v_i^{\frac{1}{2}}Z_i\|_2^3}\nonumber\\
    \leq~&  \sum_{i=1}^{n} \sqrt{\bbE\bb{\| H_i\xi_i\|^4_2}\,\bbE\bb{\| H_i\xi_i\|^2_2}} + \bbE\bb{\left\|\bbE_{i}\bb{\diag_j\bc{\frac{F_{i,j}}{\hat{F}_{i,j}} 
W_{i,j}\Var_{i,j}^+(\Psi_{i,j})W_{i,j}^\top}
 }^{\frac{1}{2}}Z_{i}\right\|_2^3}\nonumber\\
       \stackrel{(i)}{\lesssim}~& \sum_{i=1}^{n} \sqrt{Ld'\,\bbE[\| H_i\xi_i\|^4_2]} + \bbE\bb{\|Z_{i}\|_2^3}\nonumber\\
       \leq~&  \sum_{i=1}^{n} \sqrt{Ld'\, \bbE\bb{\| H_i\xi_i\|^4_2}} + \sqrt{Ld'\bbE\bb{\| Z_{i}\|^4_2}}\nonumber\\
       \stackrel{(ii)}{\leq}~& \sum_{i=1}^{n} (Ld')^{\frac{1}{2}}\sqrt{\bbE
       \bb{  \| H_i\xi_i\|_2^4 }} + (Ld')^{\frac{3}{2}} \sqrt{\kappa_4}\label{eq:beta_2}
\end{align}
where $\kappa_4$ is the fourth-moment of a standard normal random variable, (i) uses that
\[
\frac{F_{i,j}}{\hat{F}_{i,j}} 
W_{i,j}\Var_{i,j}^+(\Psi_{i,j})W_{i,j}^\top\preceq\frac{M^2}{\sigma^2}\Tr(W_{i,j}\Var_{i,j}^+(\Psi_{i,j})W_{i,j}^\top)\cdot I \precsim  I,
\]
and that 
\[
    \bbE\bb{\| H_i\xi_{i}\|^2}=\Tr(\bbE[v_i])=\Tr\bp{\bbE\bb{\diag_j\bc{\frac{F_{i,j}}{\hat{F}_{i,j}} 
W_{i,j}\Var_{i,j}^+(\Psi_{i,j})W_{i,j}^\top}}}=O(Ld');
\]
 and (ii) uses Jensen's inequality.\footnote{For any sequence $a_1,\ldots, a_K$: $(\sum_{i=1}^K a_i^2)^2 = K^2 \bp{\frac{1}{K}\sum_{i=1}^K a_i^2}^2 \leq K\sum_{i=1}^K\bp{ a_i^2}^2$.}
 
 Now we consider $\bbE\bb{\| H_i\xi_i\|_2^4}$.
  \begin{align}
  &\bbE[\| H_i\xi_i\|_2^4]=\bbE\bb{\bbE_{i}\bb{    \xi_{i}^\top  H_i^\top H_i \xi_{i}\xi_{i}^\top  H_i^\top  H_i  \xi_{i}
        } } \nonumber\\
       & \leq  (L+1)\bbE\bigg[\sum_{j=0}^L\frac{ \bbE_{i}\bb{R_{i,j}^4}\bbE_{i}\bb{ \{(\Psi_{i,j}-\bar{\Psi}_{i,j})^\top W_{i,j}^\top W_{i,j}(\Psi_{i,j}-\bar{\Psi}_{i,j})\}^2 }}{\hat{F}_{i,j}^2}\bigg],\label{eq:fourth_moment_feasible}
    \end{align}
where the inequality is by Cauchy-Schwartz inequality.
Note that $\hat{F}_{i,j}\geq \sigma^2$, and both $R_{i,j}^4$ and $\|\Psi_{i,j}-\bar{\Psi}_{i,j}\|_2^2$ are bounded by some finite constant, thus we have $(\Psi_{i,j}-\bar{\Psi}_{i,j})(\Psi_{i,j}-\bar{\Psi}_{i,j})^\top \precsim I_d$. Hence,
\begin{align}
    \eqref{eq:fourth_moment_feasible}
    \lesssim~&  L\cdot \bbE\bb{ \sum_{j=0}^L\bbE_{i}\bb{ \Tr\bp{\bp{\Psi_{i,j}-\bar{\Psi}_{i,j}}^\top\,W_{i,j}^\top W_{i,j}W_{i,j}^\top W_{i,j}\, \bp{\Psi_{i,j}-\bar{\Psi}_{i,j}}}}}\nonumber\\
    =~&  L\cdot  \bbE\bb{ \sum_{j=0}^L \Tr\bp{W_{i,j}^\top W_{i,j}W_{i,j}^\top W_{i,j}\Var_{i,j}^+(\Psi_{i,j})}}=L\cdot  \bbE\bb{  \sum_{j=0}^L\Tr\bp{W_{i,j}W_{i,j}^\top W_{i,j}\Var^+_{i,j}(\Psi_{i,j})W_{i,j}^\top}}\nonumber\\
    \stackrel{(i)}{\leq}~& L\cdot \bbE\bb{\sum_{j=0}^L \Tr\bp{W_{i,j}W_{i,j}^\top}\Tr\{W_{i,j}\Var^+_{i,j}(\Psi_{i,j})W_{i,j}^\top\}} \stackrel{(ii)}{\lesssim}  Ld'\cdot\bbE\bb{ \sum_{j=0}^L\Tr\bp{W_{i,j}W_{i,j}^\top}} . \label{eq:fourth_moment_continue}
\end{align}
where (i) uses that if $A,B$ are positive semi-definite, we have $\Tr(AB)\leq\Tr(A)\Tr(B)$, and (ii) uses Property \ref{property:weight_stabilizing}(c). 

Combining \eqref{eq:beta_2} and \eqref{eq:fourth_moment_continue}, we have that 
\begin{align}
        \beta_2 &\lesssim ~ \sum_{i=1}^{n} d'L\sqrt{ \bbE\bb{\sum_{j=0}^L  \Tr\bp{W_{i,j}^\top W_{i,j}}}} + (Ld')^{\frac{3}{2}} \sqrt{\kappa_4}  \nonumber\\
        &\leq d'Ln^{\frac{1}{2}}\sqrt{\sum_{i=1}^{n}\bbE\bb{\sum_{j=0}^L  \Tr\bp{W_{i,j}^\top W_{i,j}}}} + (Ld')^{\frac{3}{2}}n\sqrt{\kappa_4}\nonumber\\
        & = d'L n^{\frac{1}{2}}\sqrt{\sum_{j=0}^L\sum_{i=1}^{n}\bbE\bb{  \Tr\bp{W_{i,j}^\top W_{i,j}}}} + (Ld')^{\frac{3}{2}}n\sqrt{\kappa_4}\nonumber\\
        &= O\bp{(Ld')^{\frac{3}{2}}n^{\frac{\alpha_1}{2}+1}},
        \label{eq:beta_2_bound}
\end{align}
where the last equality is by Property \ref{property:weight_stabilizing}(b).

\paragraph{Applying Proposition \ref{prop:uniform_clt}.}  
Combining \eqref{eq:bounded_sigman_inverse_frob}, \eqref{eq:omega_n_bound}, \eqref{eq:beta_2_bound},
We have
\begin{align}
    &\sup_{C\in\calC}\left|\bbP\bp{n^{-1/2}\sum_{i=1}^{n}  H_i\xi_i\in C}-\bbP\bp{\calN(0, \Sigma_{n})\in C}\right|\nonumber\\
    \lesssim~&\inf_{\eta>0} \bc{
    \bp{\frac{\beta_2dL}{\eta^3}}^{\frac{1}{3}} +  \bp{\frac{\sqrt{dL\bbE[\|\Omega_{n}\|_2]}}{\eta}}^{\frac{2}{3}} +\eta\sqrt{\|n^{-1}\Sigma_{n}^{-1}\|_F}
    } \nonumber\\
  \lesssim~& \inf_{\eta>0} \bc{
            \frac{L^{\frac{2}{3}}n^{\frac{1}{3}+\frac{\alpha_1}{6}}}{\eta} 
            + \frac{L^{\frac{1}{3}}n^{\frac{1-\min(\gamma_F, \alpha_3)}{3}} }{\eta^{\frac{2}{3}}}
            +\eta n^{-\frac{1}{2}}L^{\frac{1}{4}}
  }. \label{eq:gaussian_error}
\end{align}
Choose $\eta= n^{\frac{5-2\gamma'}{10}}$ where $\gamma'\in(0,\min(\gamma_F, \alpha_3)]$. Continuing the above, we have
\begin{align*}
  \eqref{eq:gaussian_error} 
   \leq ~& L^{\frac{2}{3}}\,\min_{\gamma'\in(0, \min(\gamma_F, \alpha_3)]}\max\bp{ n^{\frac{\alpha_1-1}{6} + \frac{\gamma'}{5}}, n^{-\frac{\gamma'}{5}} }\\
  =~ &
  L^{\frac{2}{3}}\,\min_{\gamma'\in(0, \min(\gamma_F, \alpha_3)]} n^{-\min\bp{\frac{1-\alpha_1}{6} - \frac{\gamma'}{5},\frac{\gamma'}{5}}}\\
   =~&  L^{\frac{2}{3}}\, n^{-\min\bp{\frac{1-\alpha_1}{12}, \frac{\alpha_3}{5}, \frac{\gamma_F}{5}}}.\tag{choose $\gamma'=\min\bp{\frac{5(1-\alpha_1)}{12}, \gamma_F, \alpha_3}$}\nonumber
\end{align*}
 Collectively, we have
\begin{equation}
    \sup_{C\in\calC}\left|\bbP\bp{n^{-1/2}\sum_{i=1}^n  H_i\xi_i\in C}-\bbP\bp{\calN(0, \Sigma_{n})\in C}\right| = O\bp{
   L^{\frac{2}{3}} n^{-\min(\frac{1-\alpha_1}{12}, \frac{\alpha_3}{5}, \frac{\gamma_F}{5})}
    }
    \label{eq:step_1}
\end{equation}

\medskip

For the remaining  steps $2-3$, we need to invoke the following lemma.
\begin{lemma}[see \cite{barsov1987estimates,devroye2018total}]\label{lem:tvdist} Let $\Sigma_1, \Sigma_2$ be two positive definite covariance matrices and let $\lambda_1, \ldots, \lambda_d$ be the eigenvalues of $\Sigma_1^{-1}\Sigma_2-I$. Then $\text{TV}(\calN(0,\Sigma_1), \calN(0, \Sigma_2))\leq 2\min(\sqrt{\sum_i \lambda_i^2},1)$.
\end{lemma}

\subsubsection{Step 2: Connect $\calN(0, \Sigma_n)$ with $\calN(0, \Xi_n)$.}  Note that by \eqref{eq:diff_sigman_xin}, we have that $\bn{\Sigma_n - \Xi_n}_{Frob}^2=O(Ln^{-2\gamma_F})$. Moreover, under Property \ref{property:weight_stabilizing}(c),  $\Xi_n \succeq c_2 I $ and thus is invertible; also $\Sigma_n\succeq \frac{c_2}{2}I$ for large $n$ by \eqref{eq:invertability_of_sigman}.
Consider
\begin{align*}
    \Xi_n^{-1}\Sigma_n-I = \Xi_n^{-1} \cdot \bp{\Sigma_n - \Xi_n}.
\end{align*}
Let $\{\lambda_i\}$ be the eigenvalues of $\Xi_n^{-1}\Sigma_n-I $, we have
\begin{align*}
    \sum_{i}\lambda_i^2 &\lesssim\|\Xi_n^{-1}\Sigma_n-I\|_{Frob}^2 \tag{by Lemma \ref{lemma:eigenval_frob}}\\
   &= \Tr\bp{\Xi_n^{-1} \cdot \bp{\Sigma_n - \Xi_n} \cdot  \bp{\Sigma_n - \Xi_n} \Xi_n^{-1}}\\
   & = \Tr\bp{\Xi_n^{-2} \cdot \bp{\Sigma_n - \Xi_n}^2}\\
   & \leq \Tr\bp{\Xi_n^{-2}}\Tr\bp{\bp{\Sigma_n - \Xi_n}^2 }\tag{by Lemma \ref{lemma:trace_inequality}}  \\
   &\leq \Tr\bp{\Xi_n^{-1}}^2 \bn{ \bp{\Sigma_n - \Xi_n}}_{Frob}^2 \tag{since $\Xi_n^{-1}\succ 0$}\\
   &= O(L^3n^{-2\gamma_F}).
\end{align*}
Applying Lemma \ref{lem:tvdist}, we thus have 
\begin{equation}
    \text{TV}(\calN(0,\Sigma_n), \calN(0, \Xi_n))=O(L^{\frac{3}{2}}n^{-\gamma_F}).
    \label{eq:step_2}
\end{equation}

\medskip

\subsubsection{Step 3: Connect $\calN(0,\Xi_n)$ with $\calN(0, \widehat{\Xi}_n)$.}
We first show that event $\widehat{\calE}_n=\bc{ \widehat{\Xi}_n\succeq \frac{c_6}{2} I}$ happens with high probability.
Define $\overline{\Xi}_n$ as a block-diagonal matrix with its $j$-th diagonal block as,
\[
\frac{1}{n}\sum_{i=1}^{n} \bbE_{i}[W_{i,j}\Var^+_{i,j}(\Psi_{i,j})W_{i,j}^\top].
\]
 Note that $\widehat{\Xi}_n-\overline{\Xi}_n$ is the sum of martingale difference sequence. We have
\begin{align*}
    &\bbE\bb{
    \bn{\widehat{\Xi}_n-\overline{\Xi}_n}_{Frob}^2
    }= \bbE\bb{
    \Tr\bp{\bp{\widehat{\Xi}_n-\overline{\Xi}_n}\cdot \bp{\widehat{\Xi}_n-\overline{\Xi}_n}}
    } = \bbE\bb{
    \Tr\bp{\bp{\widehat{\Xi}_n-\overline{\Xi}_n}^2}
    }\\
   =&\frac{1}{n^2}\sum_{i=1}^{n}\bbE \Big[ \Tr\big(\bc{W_{i,j}\Var^+_{i,j}(\Psi_{i,j})W_{i,j}^\top - \bbE_{i}[W_{i,j}\Var^+_{i,j}(\Psi_{i,j})W_{i,j}^\top]}^2\big) \Big]\\
   \leq & \frac{1}{n^2}\sum_{i=1}^{n}  \bbE\bb{\Tr\bp{
   \bc{W_{i,j}\Var^+_{i,j}(\Psi_{i,j})W_{i,j}^\top }^2
   }}\tag{by Lemma \ref{lemma:trace_inequality_4}} \\
   =& \frac{1}{n^2}\sum_{i=1}^{n}  \bbE\bb{\Tr\bp{
   \bc{W_{i,j}\Var^+_{i,j}(\Psi_{i,j})W_{i,j}^\top }
   }^2} \\
   =&  O(Ln^{-1}).\tag{Property \ref{property:weight_stabilizing}(c)}
\end{align*}
Thus
\[
\bbE\bb{
    \bn{\widehat{\Xi}_n-\overline{\Xi}_n}_{2}
    }\leq \bbE\bb{
    \bn{\widehat{\Xi}_n-\overline{\Xi}_n}_{Frob}
    } \leq \bbE\bb{
    \bn{\widehat{\Xi}_n-\overline{\Xi}_n}_{Frob}^2
    }^{1/2}=O\bp{L^{\frac{1}{2}}n^{-\frac{1}{2}}}.
\]
Therefore,
\[
\bbE\bb{
    \bn{\widehat{\Xi}_n - {\Xi}_n}_{2}
    }\leq \bbE\bb{
    \bn{\widehat{\Xi}_n-\overline{\Xi}_n}_{2}
    } + \bbE\bb{
    \bn{\overline{\Xi}_n - {\Xi}_n}_{2}
    } = O(L^{\frac{1}{2}}n^{-\min(\frac{1}{2}, \alpha_3)}),
\] 
where we use the result that $ \bbE[\|\overline{\Xi}_n-\Xi_n\|_2] = O(n^{-\alpha_3})$, which holds as proven  below: 
\begin{align*}
   &\bbE[\|\overline{\Xi}_n-\Xi_n\|_2]\\
   \leq ~&\max_{j}\bbE\bb{\bn{\frac{1}{n}\sum_{i=1}^{n} \bbE_{i}[W_{i,j}\Var^+_{i,j}(\Psi_{i,j})W_{i,j}^\top] -\frac{1}{n}\sum_{i=1}^{n} \bbE[W_{i,j}\Var^+_{i,j}(\Psi_{i,j})W_{i,j}^\top] }_2}\\
 \leq~&  \max_{j}\frac{1}{n}\sum_{i=1}^{n}\bbE\bb{\bn{ \bbE_{i}[W_{i,j}\Var^+_{i,j}(\Psi_{i,j})W_{i,j}^\top] - \bbE[W_{i,j}\Var^+_{i,j}(\Psi_{i,j})W_{i,j}^\top] }_2}\tag{triangular inequality}\\
 \leq~&  \max_{j}\frac{1}{n}\sum_{i=1}^{n}\bbE\bb{\bn{ \bbE_{i,j}[W_{i,j}\Var^+_{i,j}(\Psi_{i,j})W_{i,j}^\top] - \bbE[W_{i,j}\Var^+_{i,j}(\Psi_{i,j})W_{i,j}^\top] }_2} \tag{Jensen's inequality}\\
 =~& O(n^{-\alpha_3}).\tag{by Property \ref{property:weight_stabilizing}(c)}
\end{align*}

On the other hand,  Property \ref{property:weight_stabilizing}(c) says that $\Xi_n\succeq c_2\cdot I$. Applying Lemma \ref{lem:high_probability_invertible}, we have that:
\[
\bbP(\widehat{\calE}_n^c) = O(L^{\frac{1}{2}}n^{-\min(\alpha_3,\frac{1}{2})}).
\]

Under Property \ref{property:weight_stabilizing}(c), we have that $\Xi_n \succeq c_2 I $ and thus is invertible. Now conditioning on event $\widehat{\calE}_n$ (and thus $\widehat{\Xi}_n$ is invertible),  we  have
\begin{align*}
    \Xi_n^{-1}\widehat{\Xi}_n-I = \Xi_n^{-1} \cdot \bp{\widehat{\Xi}_n - \Xi_n}.
\end{align*}
We overload notations and let $\{\lambda_i\}$ be the eigenvalues of $\Xi_n^{-1}\widehat{\Xi}_n-I $, we have
\begin{align*}
   & \sum_{i}\lambda_i^2 \lesssim \|\Xi_n^{-1}\widehat{\Xi}_n-I\|_{Frob}^2 \tag{by Lemma \ref{lemma:eigenval_frob}}\\
   &= \Tr\bp{\Xi_n^{-1} \cdot \bp{\widehat{\Xi}_n - \Xi_n} \cdot \bp{\widehat{\Xi}_n - \Xi_n} \Xi_n^{-1}}\\
   & \leq \Tr\bp{\Xi_n^{-2}}\Tr\bp{\bp{\widehat{\Xi}_n - \Xi_n}^2 }  =\Tr\bp{\Xi_n^{-1}}^2 \bn{ \widehat{\Xi}_n - \Xi_n}_{Frob}^2= O(L^3) \bn{ \widehat{\Xi}_n - \Xi_n}_{2}^2.
\end{align*}
Applying Lemma \ref{lem:tvdist}, 
we  have that, conditioning on event $\widehat{\calE}_n $ happens,
\begin{equation}
\label{eq:xin_hat_xin_tv}
    \text{TV}(\calN(0,\Xi_n), \calN(0, \widehat{\Xi}_n))=O(L^{\frac{3}{2}})\bn{ \widehat{\Xi}_n - \Xi_n}_{2},
\end{equation}
where both sides are  functions of the random matrix $\widehat{\Xi}_n$, and the  total variation distance on the left-hand-side is defined conditional on $\widehat{\Xi}_n$, which serves as the covariance matrix of the Gaussian distribution $\calN(0, \widehat{\Xi}_n)$.
Then for any $C\in\calC$, we have
\begin{align*}
&\ba{ \bbP\bp{\calN(0,\Xi_n)\in C } -  \bbP\bp{\calN(0,{\widehat{\Xi}}_n)\in A}}\\
\leq  & \ba{ \bbP\bp{\{\calN(0,\Xi_n)\in C\} \cap \widehat{\calE}_n } -  \bbP\bp{\{\calN(0,{\widehat{\Xi}}_n)\in C\} \cap \widehat{\calE}_n}} + 2 \bbP(\widehat{\calE}_n^c)\\
= & \ba{ \bbP\bp{\calN(0,\Xi_n)\in C \mid \widehat{\calE}_n} -  \bbP\bp{\calN(0,{\widehat{\Xi}}_n)\in C \mid\widehat{\calE}_n}}\bbP(\widehat{\calE}_n) + 2 \bbP(\widehat{\calE}_n^c)\\
   = & \ba{
  \bbE\bb{\one
  \bc{\calN(0,\Xi_n)\in C} -
  \one
  \bc{\calN(0,\widehat{\Xi}_n)\in C}
  \mid \widehat{\calE}_n
  } 
   }
   \bbP(\widehat{\calE}_n)
   + 2 \bbP(\widehat{\calE}_n^c)\\
     = & \ba{
  \bbE\bb{\bbE\bb{\one
  \bc{\calN(0,\Xi_n)\in C} -
  \one
  \bc{\calN(0,\widehat{\Xi}_n)\in C}
  \mid \widehat{\Xi}_n, \widehat{\calE}_n}\mid \widehat{\calE}_n
  } 
   }
   \bbP(\widehat{\calE}_n)
   + 2 \bbP(\widehat{\calE}_n^c)\\
        = & \ba{
  \bbE\bb{\bbP\bp{\calN(0,\Xi_n)\in C\mid \widehat{\Xi}_n, \widehat{\calE}_n} -\bbP\bp{\calN(0,\widehat{\Xi}_n)\in C\mid \widehat{\Xi}_n, \widehat{\calE}_n}
  \mid
 \widehat{\calE}_n  
   }}
   \bbP(\widehat{\calE}_n)
   + 2 \bbP(\widehat{\calE}_n^c)\\
   \leq & \bbE\bb{ \text{TV}(\calN(0,\Xi_n), \calN(0, \widehat{\Xi}_n))\mid \widehat{\calE}_n} \bbP(\widehat{\calE}_n) + 2 \bbP(\widehat{\calE}_n^c)\\
   = &\bbE\bb{O(L^{3/2})\bn{ \widehat{\Xi}_n - \Xi_n}_{2}\one[\widehat{\calE}_n]} + 2 \bbP(\widehat{\calE}_n^c)
   \tag{applying \eqref{eq:xin_hat_xin_tv}}\\
   \leq & \bbE\bb{O(L^{3/2})\bn{\widehat{\Xi}_n - \Xi_n}_{2}} + 2 \bbP(\widehat{\calE}_n^c)
   =  O(L^2 n^{-\min(\alpha_3,1/2)}).
\end{align*}
Thus 
\begin{align}
    &\sup_{C\in \calC}\ba{ \bbP\bp{\calN(0,\Xi_n)\in C} -  \bbP\bp{\calN(0,\widehat{\Xi}_n)\in C}} =  O(L^2 n^{-\min(\alpha_3,1/2)}) \label{eq:step_3}
\end{align}

\medskip

\subsubsection{Put everything together.} Combining \eqref{eq:step_1}, \eqref{eq:step_2}, \eqref{eq:step_3}, we have
\begin{align*}
    &\sup_{C\in \calC}\left|\bbP\bp{n^{-1/2}\sum_{i=1}^n  H_i\xi_i\in C}-\bbP\bp{\calN(0, \widehat{\Xi}_n)\in C}\right| 
    = O  \bp{ L^2\cdot n^{-\min(\frac{1-\alpha_1}{12}, \frac{\alpha_3}{5}, \frac{\gamma}{5})}}.
\end{align*}


\subsection{Asymptotic regularity of $B_n$}
\label{appendix:regularity_B_normal}
Note that $B_n$ is block upper triangular, where the block sizes are corresponding to the decomposition of $\theta=( \theta_1, \dots, \theta_L)$. 
Recall that by construction, $\hat{f}\in[\sigma^2, M^2]$.

For any $j\leq j'$ we have:
 \begin{align*}
  B_{n,j,j'}=n^{-1}\sum_{i=1}^n H_{i,j}(\Psi_{i,j}-\bar{\Psi}_{i,j})\Phi_{i,j'}^\top .
\end{align*}
Define 
\begin{align*}
        B^0_{n,j,j'}=n^{-1}\sum_{i=1}^n\bbE_{i,j}[ H_{i,j}(\Psi_{i,j}-\bar{\Psi}_{i,j})\Phi_{i,j'}^\top ],\mbox{ and }\delta_{n,j,j'}=B^0_{n,j,j'}-B_{n,j,j'}.
\end{align*}
For $j\in\{1,\dots,L\}$, we have:
\begin{align*}
    \lambda_{\min} ((B^0_{n,j,j})^\top B^0_{n,j,j})\geq~&\frac{1}{M}  \lambda_{\min}\bp{
   \bc{ \frac{1}{n}\sum_{i=1}^n \bbE_{i,j}[W_{i,j}\Cov^+_{i,j}(\Psi_{i,j},\Phi_{i,j}) ]}^\top
    \bc{ \frac{1}{n}\sum_{i=1}^n \bbE_{i,j}[W_{i,j}\Cov^+_{i,j}(\Psi_{i,j},\Phi_{i,j}) ]}
    } \\
    \geq ~&\frac{c_2^2}{M}n^{-\alpha_1}.\tag{by Property \ref{property:weight_stabilizing}(a)}
\end{align*}

Next we shall show that for $j\leq j'$, we have
  \begin{align*}
  \bbE[\| \delta_{n,j,j'}\|_{Frob}^2]=O(n^{-1}),
  \end{align*}
  and that the event $\calE^M_n$ happens with high probability:
  \begin{align*}
      \bbP(\bp{\calE^M_n}^c)=O(Ln^{\alpha_1-1}).
  \end{align*}

\subsubsection{Asymptotic negligibility of $\delta_{n,j,j'}$.}  With $j\leq j'$,  we have
\begin{align*}
    \bbE[\|\delta_{n,j,j'}\|_{Frob}^2 ]
    =~& n^{-2}\bbE\Big[\Tr\Big(\bc{\sum_{i=1}^n H_{i,j}\bp{(\Psi_{i,j}-\bar{\Psi}_{i,j})\Phi_{i,j'}^\top -\bbE_{i,j}[(\Psi_{i,j}-\bar{\Psi}_{i,j})\Phi_{i,j'}^\top ]}}\\
  &\quad\quad  \times\bc{\sum_{i=1}^n H_{i,j}\bp{(\Psi_{i,j}-\bar{\Psi}_{i,j})\Phi_{i,j'}^\top -\bbE_{i,j}[(\Psi_{i,j}-\bar{\Psi}_{i,j})\Phi_{i,j'}^\top ]}}^\top  \Big)\Big]\\
  \stackrel{(i)}{=}~& n^{-2}\sum_{i=1}^n\Tr\Big( \bbE\Big[H_{i,j}\bp{(\Psi_{i,j}-\bar{\Psi}_{i,j})\Phi_{i,j'}^\top -\bbE_{i,j}[(\Psi_{i,j}-\bar{\Psi}_{i,j})\Phi_{i,j'}^\top ]}\\
  &\quad \quad \times \bp{(\Psi_{i,j}-\bar{\Psi}_{i,j})\Phi_{i,j'}^\top -\bbE_{i,j}[(\Psi_{i,j}-\bar{\Psi}_{i,j})\Phi_{i,j'}^\top ]}^\top H_{i,j}^\top 
  \Big]\Big)\\
  \stackrel{(ii)}{\leq}~&  n^{-2}\sum_{i=1}^n\Tr \bp{
\bbE\bb{ H_{i,j} \bbE^+_{i,j}\bb{(\Psi_{i,j}-\bar{\Psi}_{i,j})\Phi_{i,j'}^\top \Phi_{i,j'}(\Psi_{i,j}-\bar{\Psi}_{i,j})^\top } H_{i,j}^\top }
  }\\
   \stackrel{(iii)}{\lesssim}~& n^{-2}\sum_{i=1}^n\Tr \bp{
\bbE\bb{ W_{i,j} \bbE^+_{i,j}\bb{(\Psi_{i,j}-\bar{\Psi}_{i,j})(\Psi_{i,j}-\bar{\Psi}_{i,j})^\top } W_{i,j}^\top }
  } \stackrel{(iv)}{=} O(n^{-1}).
\end{align*}
where (i) uses the fact  that $\delta_{n,j,j'}$ is a sum of martingale difference sequence, (ii) uses Lemmas \ref{lemma:trace_inequality_3} \& \ref{lemma:trace_inequality_4}, (iii) uses that $\Phi_{i,j'}$ is bounded and Lemma  \ref{lemma:trace_inequality_3}, and (iv) uses Property \ref{property:weight_stabilizing}(c).

\subsubsection{High probability event $\calE^M_n$.} 
We finally show that $\bbP(\bp{\calE^M_n}^c)=O(Ln^{\alpha_1-1})$. When $\calE^M_n$ does not happen, by Lemma~\ref{lemma:eigenvalue_tri}, there exists a $j_0\in\{0,1,,\dots, l\}$ and $x$ with $\|x\|_2=1$ such that 
\begin{equation}
    \label{eq:ct_ex_normal}
    x^\top B_{n,j_0,j_0}^\top B_{n,j_0,j_0} x < \frac{c_2^2n^{-\alpha_1}}{4M}.
\end{equation}
Since $\delta_{n,j_0,j_0}=B_{n,j_0,j_0}-B_{n,j_0,j_0}^0$, by Lemma~\ref{lem:lower-eigenvalue}, we have the following:
\begin{align*}
    \frac{c_1^2n^{-\alpha_1}}{4M^2} > x^\top B_{n,j_0,j_0}^\top B_{n,j_0,j_0} x \geq \frac{1}{2} x^\top (B_{n,j_0,j_0}^0)^\top B_{n,j_0,j_0}^0 x - 2\|\delta_{n,j_0,j_0}\|_{Frob}^2 \geq  \frac{n^{-\alpha_1}c_2^2}{2M} - 2\|\delta_{n,j_0,j_0}\|_{Frob}^2.
\end{align*}

 Rearranging yields the following results:
\begin{align}
    \|\delta_{n,j_0,j_0}\|_{Frob}^2 > \frac{n^{-\alpha_1}c_2^2}{8M}
\end{align}
Thus by a union bound, we have
\begin{align*}
    \bbP(\bp{\calE^M_n}^c) 
    \leq~& \bbP\bp{\exists j: \|\delta_{n,j,j}\|_{Frob}^2 > \frac{n^{-\alpha_1}c_2^2}{8M}}\\
    \leq~& \sum_{j=0}^L\bbP\bp{ \|\delta_{n,j,j}\|_{Frob}^2 > \frac{n^{-\alpha_1}c_2^2}{8M} }  \\
    \leq~& \sum_{j=0}^L\frac{8M \bbE[\|\delta_{n,j,j}\|_{Frob}^2]}{c_2^2n^{-\alpha_1}} =O(Ln^{\alpha_1-1}) .
\end{align*}

\section{Proof of Corollaries}
\subsection{Proof of Corollary \ref{cor:consistency}}
\subsubsection{Identity weights.}
Consider $H_{i,j}$ being the identity matrix, we verify Property \ref{property:weight_regularizing}.

\emph{Property \ref{property:weight_regularizing}(a)}. We have 
\begin{align*}
   \frac{1}{n}\sum_{i=1}^n \bbE_{i,j}\bb{H_{i,j} \Var^+_{i,j}(\Phi_{i,j}) }= \frac{1}{n}\sum_{i=1}^n \bbE_{i,j}\bb{(I_{d_\tau}\otimes X_{i,j})\Var^+_{i,j}(T_{i,j}) (I_{d_\tau}\otimes X_{i,j})^\top}\\
   \succeq  c  \frac{1}{n}\sum_{i=1}^n t_{i,j}^{-\alpha}\bbE_{i,j}\bb{ (I_{d_\tau}\otimes X_{i,j}) (I_{d_\tau}\otimes X_{i,j})^\top}  = \frac{1}{n}\sum_{i=1}^n t_{i,j}^{-\alpha}I_{d_\tau}\otimes\bbE_{i,j}\bb{ X_{i,j}X_{i,j}^\top}\gtrsim   n^{-\alpha} I.
\end{align*}
Thus:
\[
\bc{\frac{1}{n}\sum_{i=1}^n \bbE_{i,j}\bb{H_{i,j} \Var^+_{i,j}(\Phi_{i,j}) }}^\top \bc{\frac{1}{n}\sum_{i=1}^n \bbE_{i,j}\bb{H_{i,j} \Var^+_{i,j}(\Phi_{i,j}) }}\gtrsim   n^{-2\alpha} I.
\]

\emph{Property \ref{property:weight_regularizing}(b)}. We have
\begin{align*}
    \frac{1}{n}\sum_{i=1}^n \Tr\bp{H_{i,j}H_{i,j}^\top} = \frac{1}{n}\sum_{i=1}^n \Tr(I)= O(1).
\end{align*}

\emph{Property \ref{property:weight_regularizing}(c)}. We have
\begin{align*}
    \frac{1}{n} \sum_{i=1}^n \Tr\bp{\bbE_{i,j}[H_{i,j} \Var_{i,j}^+(\Phi_{i,j}) H_{i,j}^\top]}= \frac{1}{n} \sum_{i=1}^n \Tr\bp{\bbE_{i,j}[\Var_{i,j}^+(\Phi_{i,j})]}= O(1).
\end{align*}

\subsubsection{Consistency weights.} Consider $H_{i,j} = (I_{d_\tau}\otimes X_{i,j})\Var^+_{i,j}(T_{i,j})^{-1/2} 
    (I_{d_\tau}\otimes X_{i,j})^\top $.

\emph{Property \ref{property:weight_regularizing}(a)}. We have 
\begin{align*}
   &\frac{1}{n}\sum_{i=1}^n \bbE_{i,j}\bb{H_{i,j} \Var^+_{i,j}(\Phi_{i,j}) }\\
   &= \frac{1}{n}\sum_{i=1}^n \bbE_{i,j}\bb{ 
    (I_{d_\tau}\otimes X_{i,j})\Var^+_{i,j}(T_{i,j})^{-1/2} 
    (I_{d_\tau}\otimes X_{i,j})^\top
    (I_{d_\tau}\otimes X_{i,j})\Var^+_{i,j}(T_{i,j}) (I_{d_\tau}\otimes X_{i,j})^\top
   }\\
     &= \frac{1}{n}\sum_{i=1}^n \bbE_{i,j}\bb{ 
    \|X_{i,j}\|_2^2(I_{d_\tau}\otimes X_{i,j})\Var^+_{i,j}(T_{i,j})^{-1/2} 
\Var^+_{i,j}(T_{i,j}) 
    (I_{d_\tau}\otimes X_{i,j})^\top)
   }\\
    &\succeq \frac{c^2}{n}\sum_{i=1}^n t_{i,j}^{-\alpha/2}\bbE_{i,j}\bb{ 
    (I_{d_\tau}\otimes X_{i,j}) 
    (I_{d_\tau}\otimes X_{i,j})^\top
   } =\frac{c^2}{n}\sum_{i=1}^n t_{i,j}^{-\alpha/2} \bbE_{i,j}\bb{ 
    I_{d_\tau}\otimes( X_{i,j} X_{i,j}^\top) 
   } \gtrsim   n^{-\alpha/2} I.
\end{align*}
Thus:
\[
\bc{\frac{1}{n}\sum_{i=1}^n \bbE_{i,j}\bb{H_{i,j} \Var^+_{i,j}(\Phi_{i,j}) }}^\top \bc{\frac{1}{n}\sum_{i=1}^n \bbE_{i,j}\bb{H_{i,j} \Var^+_{i,j}(\Phi_{i,j}) }}\gtrsim   n^{-\alpha} I.
\]

\emph{Property \ref{property:weight_regularizing}(b)}.  We have
\begin{align*}
    &\frac{1}{n}\sum_{i=1}^n \Tr\bp{H_{i,j}H_{i,j}^\top} \\
    &= \frac{1}{n}\sum_{i=1}^n \Tr(   (I_{d_\tau}\otimes X_{i,j})\Var^+_{i,j}(T_{i,j})^{-1/2} 
    (I_{d_\tau}\otimes X_{i,j})^\top
    (I_{d_\tau}\otimes X_{i,j})
    \Var^+_{i,j}(T_{i,j})^{-1/2} 
    (I_{d_\tau}\otimes X_{i,j})^\top
    )\\
    &= \frac{1}{n}\sum_{i=1}^n\| X_{i,j}\|_2^{2} \Tr(   (I_{d_\tau}\otimes X_{i,j})\Var^+_{i,j}(T_{i,j})^{-1/2} 
    \Var^+_{i,j}(T_{i,j})^{-1/2} 
    (I_{d_\tau}\otimes X_{i,j})^\top
    )\\
    & \leq  \frac{c^{-1}}{n}\sum_{i=1}^nt_{i,j}^\alpha\| X_{i,j}\|_2^{2} \Tr(   (I_{d_\tau}\otimes X_{i,j})
    (I_{d_\tau}\otimes X_{i,j})^\top
    )\\
     & = \frac{c^{-1}}{n}\sum_{i=1}^nt_{i,j}^\alpha\| X_{i,j}\|_2^{2} \| X_{i,j}\|_2^{2} \Tr(   I
    ) =  O(n^{\alpha}).
\end{align*}

\emph{Property \ref{property:weight_regularizing}(c)}. We have
\begin{align*}
    &\frac{1}{n} \sum_{i=1}^n \Tr\bp{\bbE_{i,j}[H_{i,j} \Var_{i,j}^+(\Phi_{i,j}) H_{i,j}^\top]}\\
  =&  \frac{1}{n}\sum_{i=1}^n \Tr\Big(\bbE_{i,j}\Big[
    (I_{d_\tau}\otimes X_{i,j})\Var^+_{i,j}(T_{i,j})^{-1/2} 
     (I_{d_\tau}\otimes X_{i,j})^\top
    (I_{d_\tau}\otimes X_{i,j})\Var^+_{i,j}(T_{i,j}) \\
    &\quad \quad (I_{d_\tau}\otimes X_{i,j})^\top
     (I_{d_\tau}\otimes X_{i,j})
     \Var^+_{i,j}(T_{i,j})^{-1/2} 
     (I_{d_\tau}\otimes X_{i,j})^\top
   \Big]\Big)\\
     =& \frac{1}{n}\sum_{i=1}^n \Tr\bp{\bbE_{i,j}\bb{ \|X_{i,j}\|_2^4 
    (I_{d_\tau}\otimes X_{i,j})\Var^+_{i,j}(T_{i,j})^{-1/2} 
\Var^+_{i,j}(T_{i,j}) \Var^+_{i,j}(T_{i,j})^{-1/2} 
    (I_{d_\tau}\otimes X_{i,j})^\top
   }}\\
    =& \frac{1}{n}\sum_{i=1}^n\Tr\bp{ \bbE_{i,j}\bb{  \| X_{i,j}\|_2^4
    (I_{d_\tau}\otimes X_{i,j}) 
    (I_{d_\tau}\otimes X_{i,j})^\top
   }} =\frac{1}{n}\sum_{i=1}^n \Tr\bp{\bbE_{i,j}\bb{  \| X_{i,j}\|_2^4
    I_{d_\tau}\otimes( X_{i,j} X_{i,j}^\top) 
   }} \\
   =&  \frac{1}{n}\sum_{i=1}^n    K\cdot\Tr\bp{\bbE_{i,j}\bb{  \| X_{i,j}\|_2^4
 X_{i,j} X_{i,j}^\top
   }}= \frac{1}{n}\sum_{i=1}^n    K\cdot\bp{\bbE_{i,j}\bb{  \| X_{i,j}\|_2^6
   }} = O(K c^3)= O(1).
\end{align*}

\subsection{Proof of Corollary \ref{cor:uniform_ci}}
\label{appendix:cor_ci}

Let $a$ be a given confidence level. Define set
\[
\calA(\ell) =\{x\in\bbR^{1+dL}: |\ell^\top  x|\leq q_{\ell,\frac{a}{2}}\},
\].
By definition of $q_{\ell,\frac{\alpha}{2}}$, we have 
\begin{align*}
   \bbP \bp{
(B_n^\top A B_n)^\dagger B_n^\top A  Z_\xi\in \calA(\ell)}=&
   \bbP\bp{
|\ell^\top (B_{n}^\top A B_{n})^{-1} B_n^\top A Z_\xi|\leq q_{\ell,\frac{\alpha}{2}}
}\nonumber\\
=&
   \bbP\bp{
|\calN(0, n^{-1}
\ell^\top (B_{n}^\top A B_{n})^{-1} B_n^\top A
\widehat{\Xi}_n 
A B_n (B_{n}^\top A B_{n})^{-1} \ell
)
|\leq q_{\ell,\frac{a}{2}}
}\\
=& 1-a.
\end{align*}
The Strong Gaussian Approximation results says:
\begin{equation*}
 \left| \bbP\bp{ \sqrt{n} (\hat{\theta}_n -\theta^*)\in \calA(\ell)}-\bbP\bp{(B_n^\top A B_n)^\dagger B_n^\top A Z_\xi\in \calA(\ell)}\right| 
   = O  \bp{ L^2\cdot n^{-\min(\frac{1-\alpha_1}{12},\frac{1-(L+1)\alpha_1}{2}, \frac{\alpha_3}{5}, \frac{\gamma}{5})}}.
\end{equation*}
We thus have:
\begin{align*}
    \bbP\bp{ \ell^\top \theta^* \in \bb{
    \ell^\top \hat{\theta}\pm n^{-\frac{1}{2}}q_{\ell,\frac{a}{2}}}} =& \bbP\bp{ 
\sqrt{n}(\hat{\theta}_n-\theta^*)\in \calA(\ell)
}\\
=&    \bbP \bp{
(B_n^\top A B_n)^\dagger B_n^\top A  Z_\xi\in \calA(\ell)} +  O  \bp{ L^2\cdot n^{-\min(\frac{1-\alpha_1}{12},\frac{1-(L+1)\alpha_1}{2}, \frac{\alpha_3}{5}, \frac{\gamma}{5})}}\\
=& (1-a) +  O  \bp{ L^2\cdot n^{-\min(\frac{1-\alpha_1}{12},\frac{1-(L+1)\alpha_1}{2}, \frac{\alpha_3}{5}, \frac{\gamma}{5})}}.
\end{align*}

\subsection{Proof of Corollary \ref{cor:uniform_cb}}
\label{appendix:cor_cb}

Let $a$ be a given confidence level. Define set
\[
\calA(a) = \bc{
x\in\bbR^{1+dL}: \left\|(\hat{D}_n^{\dagger})^{1/2} x\right\|_\infty\leq   q_{\infty,\, 1-a}
}.
\]
Then we have
\begin{align*}
 \bbP  \bp{\theta^*\in \bb{\hat{\theta}_n \pm n^{-\frac{1}{2}} q_{\infty,1-a}\cdot \sqrt{\hat{d}_n} }} 
   & =  \bbP \bp{n^{\frac{1}{2}}(\theta^*-\hat{\theta}_n)\in \bb{ \pm  q_{\infty,1-a}\cdot \sqrt{\hat{d}_n}  }}\\
    &= \bbP \bp{(\hat{D}_n^\dagger)^{1/2}n^{\frac{1}{2}}(\theta^*-\hat{\theta}_n)\in \bb{ \pm  q_{\infty,1-a}\cdot\one\bc{\hat d_n\neq \zero}  }} \\
    & = \bbP \bp{\bn{(\hat{D}_n^\dagger)^{1/2}n^{\frac{1}{2}}(\theta^*-\hat{\theta}_n)}_\infty\leq  q_{\infty,1-a} }\\
    &=\bbP \bp{ n^{\frac{1}{2}}(\theta^*-\hat{\theta}_n)\in \calA(a) }
\end{align*}
On the other hand, we have:
\begin{align*}
    \bbP\bp{\bc{B_{n}^\top A B_{n}}^\dagger B_n^\top A Z_\xi\in A(a)} & = \bbP\bp{
    \left\|(\hat{D}_n^{\dagger})^{1/2} 
    \bc{B_{n}^\top A B_{n}}^\dagger B_n^\top A Z_\xi
    \right\|_\infty\leq   q_{\infty,\, 1-a}
    }\\
    & = \bbP\bp{
\|\calN(0,(\hat{D}_n^{\dagger})^{\frac{1}{2}}
 \bc{B_{n}^\top A B_{n}}^\dagger B_n^\top A\widehat{\Xi}_n AB_n \bc{B_{n}^\top A B_{n}}^\dagger
(\hat{D}_n^{\dagger})^{\frac{1}{2}})\|_\infty
\leq q_{\infty,\, 1-a}}\\
&=1-a \tag{by definition of $q_{\infty,1-a}$}.
\end{align*}
Moreover, by Strong Gaussian Approximation,  we have
\begin{equation*}
 \ba{ \bbP\bp{n^{\frac{1}{2}}(\theta^*-\hat{\theta}_n)\in A(a)}-\bbP\bp{\bc{B_{n}^\top A B_{n}}^\dagger B_n^\top A Z_\xi\in A(a)}} =O\bp{ L^2\cdot n^{-\min\bp{\frac{1-\alpha_1}{12}, \frac{1-(L+1)\alpha_1}{2},\frac{\alpha_3}{5},\frac{\gamma}{5} }}}.
\end{equation*}
Collectively, we have
\begin{align*}
    \bbP\bc{\theta^*\in \bb{\hat{\theta}_n \pm n^{-\frac{1}{2}} q_{\infty,1-a}\cdot \sqrt{\hat{d}_n} }}
 &=   \bbP\bc{\bc{n^{\frac{1}{2}}(\theta^*-\hat{\theta}_n)\in \calA(a)}} 
    \\
       &=  \bbP\bc{\bc{B_{n}^\top A B_{n}}^\dagger B_n^\top A Z_\xi\in A(a)} +O\bp{ L^2\cdot n^{-\min\bp{\frac{1-\alpha_1}{12}, \frac{1-(L+1)\alpha_1}{2},\frac{\alpha_3}{5},\frac{\gamma}{5} }}}
\\
    &=1-a + O\bp{ L^2\cdot n^{-\min\bp{\frac{1-\alpha_1}{12}, \frac{1-(L+1)\alpha_1}{2}, \frac{\alpha_3}{5},\frac{\gamma}{5} }}}.
\end{align*}

\subsection{Proof of Corollary \ref{cor:normality}}

Recall that $V_{i,j}:=\bbE_{i,j}[X_{i,j}X_{i,j}^\top]$. Consider weighting scheme $H_{i,j} := \hat{f}_{i,j}^{-1/2}W_{i,j}$, where $\hat{f}_{i,j}$ satisfies Condition \eqref{eq:f_convergence} and 
\[
 W_{i,j}=(I_{d_{\tau}}\otimes\hat{V}_{i,j}^{-1/2}) (I_{d_\tau}\otimes X_{i,j})\Var^+_{i,j}(T_{i,j})^{-1/2} 
    (I_{d_\tau}\otimes X_{i,j})^\top\cdot\|X_{i,j}\|_2^{-2},
 \]
 for  $ \hat{V}_{i,j}$ adapted to $\calF_{i,j}$ and satisfying
 $c^{-1}\cdot I \succeq\hat{V}_{i,j}\succeq c\cdot I$ for some $c>0$ and 
 \[
  \frac{1}{n}
    \sum_{i=1}^n \bbE\bb{\left\|\hat{V}_{i,j}-V_{i,j}
    \right\|_{2}} = O(n^{-\gamma}).
 \]
We verify $W_{i,j}$ satisfies Property \ref{property:weight_stabilizing} with $\alpha_1=\alpha$ and $\alpha_3=\gamma$.

\paragraph{Checking Property \ref{property:weight_stabilizing}(a).}  We have
\begin{align*}
      \frac{1}{n}\sum_{i=1}^n \bbE_{i,j}\bb{W_{i,j} \Var^+_{i,j}(\Phi_{i,j}) }&= \frac{1}{n}\sum_{i=1}^n \bbE_{i,j}\Big[ 
   (I_{d_{\tau}}\otimes\hat{V}_{i,j}^{-1/2})
    (I_{d_\tau}\otimes X_{i,j})\Var^+_{i,j}(T_{i,j})^{-1/2} 
     (I_{d_\tau}\otimes X_{i,j})^\top\\
     &\quad\quad \quad \quad\quad \quad\quad \times \|X_{i,j}\|_2^{-2}
    (I_{d_\tau}\otimes X_{i,j})\Var^+_{i,j}(T_{i,j}) (I_{d_\tau}\otimes X_{i,j})^\top
   \Big]\\
     &= \frac{1}{n}\sum_{i=1}^n (I_{d_{\tau}}\otimes\hat{V}_{i,j}^{-1/2})\bbE_{i,j}\bb{ 
    (I_{d_\tau}\otimes X_{i,j})\Var^+_{i,j}(T_{i,j})^{-1/2} 
\Var^+_{i,j}(T_{i,j}) 
    (I_{d_\tau}\otimes X_{i,j})^\top
   }\\
    &= \frac{1}{n}\sum_{i=1}^n (I_{d_{\tau}}\otimes\hat{V}_{i,j}^{-1/2})\bbE_{i,j}\bb{ 
    (I_{d_\tau}\otimes X_{i,j}) 
\Var^+_{i,j}(T_{i,j})^{1/2} 
    (I_{d_\tau}\otimes X_{i,j})^\top
   }
\end{align*}
Since $\hat{V}_{i,j} \preceq c^{-1}\cdot I$, we have
both $(I_{d_{\tau}}\otimes\hat{V}_{i,j}^{-1/2})$ and $\bbE_{i,j}\bb{ 
    (I_{d_\tau}\otimes X_{i,j}) 
\Var^+_{i,j}(T_{i,j})^{1/2} 
    (I_{d_\tau}\otimes X_{i,j})^\top
   }$ being symmetric positive definite.  Therefore $\frac{1}{n}\sum_{i=1}^n \bbE_{i,j}\bb{W_{i,j} \Var^+_{i,j}(\Phi_{i,j}) }$ has eigenvalues being real number and positive. We have:
\begin{align*}
&\lambda_{\min}\bp{\bc{\frac{1}{n}\sum_{i=1}^n \bbE_{i,j}\bb{W_{i,j} \Var^+_{i,j}(\Phi_{i,j}) }}^\top \bc{\frac{1}{n}\sum_{i=1}^n \bbE_{i,j}\bb{W_{i,j} \Var^+_{i,j}(\Phi_{i,j}) }}}\\
  \geq   &\lambda_{\min}\bp{\frac{1}{n}\sum_{i=1}^n \bbE_{i,j}\bb{W_{i,j} \Var^+_{i,j}(\Phi_{i,j}) }}^2\\
    \geq &
    c^{1/2}\lambda_{\min}\bp{
    \frac{1}{n}\sum_{i=1}^n \bbE_{i,j}\bb{ 
    (I_{d_\tau}\otimes X_{i,j}) 
\Var^+_{i,j}(T_{i,j})^{1/2} 
    (I_{d_\tau}\otimes X_{i,j})^\top
    }}^2\\
    \geq  &    c^{1/2}\lambda_{\min}\bp{
    \frac{1}{n}\sum_{i=1}^n t_{i,j}^{-\alpha/2}\bbE_{i,j}\bb{ 
    (I_{d_\tau}\otimes X_{i,j}) 
    (I_{d_\tau}\otimes X_{i,j})^\top
    }}^2 \\
    =&  c^{1/2}\lambda_{\min}\bp{
    \frac{1}{n}\sum_{i=1}^n t_{i,j}^{-\alpha/2}I_{d_\tau}\otimes\bbE_{i,j}\bb{ 
     X_{i,j}X_{i,j}^\top 
    }}^2 = \Omega(1) n^{-\alpha}.
\end{align*}

\paragraph{Checking Property \ref{property:weight_stabilizing}(b).} 

\begin{align*}
        &\frac{1}{n}\sum_{i=1}^n \Tr\bp{W_{i,j}W_{i,j}^\top} \\
    &= \frac{1}{n}\sum_{i=1}^n \Tr\big( (I_{d_{\tau}}\otimes\hat{V}_{i,j}^{-1/2})  (I_{d_\tau}\otimes X_{i,j})\Var^+_{i,j}(T_{i,j})^{-1/2} 
     (I_{d_\tau}\otimes X_{i,j})^\top\\
     &\quad\quad \times \| X_{i,j}\|_2^{-2}
    (I_{d_\tau}\otimes X_{i,j}\cdot\| X_{i,j}\|_2^{-2})\Var^+_{i,j}(T_{i,j})^{-1/2} 
    (I_{d_\tau}\otimes X_{i,j})^\top(I_{d_{\tau}}\otimes\hat{V}_{i,j}^{-1/2})
    \big)\\
    &= \frac{1}{n}\sum_{i=1}^n\| X_{i,j}\|_2^{-2} \Tr(  (I_{d_{\tau}}\otimes\hat{V}_{i,j}^{-1}) (I_{d_\tau}\otimes X_{i,j})\Var^+_{i,j}(T_{i,j})^{-1/2} 
    \Var^+_{i,j}(T_{i,j})^{-1/2} 
    (I_{d_\tau}\otimes X_{i,j})^\top
    )\\
    &\lesssim \frac{c^{-1}}{n}\sum_{i=1}^n\| X_{i,j}\|_2^{-2} \Tr(  (I_{d_\tau}\otimes X_{i,j})\Var^+_{i,j}(T_{i,j})^{-1/2} 
    \Var^+_{i,j}(T_{i,j})^{-1/2} 
    (I_{d_\tau}\otimes X_{i,j})^\top
    )\\
    & \leq  \frac{c^{-2}}{n}\sum_{i=1}^nt_{i,j}^\alpha\| X_{i,j}\|_2^{-2} \Tr(   (I_{d_\tau}\otimes X_{i,j})
    (I_{d_\tau}\otimes X_{i,j})^\top
    )\\
     & =  \frac{c^{-2}}{n}\sum_{i=1}^nt_{i,j}^\alpha\| X_{i,j}\|_2^{-2} \Tr(  (I_{d_\tau}\otimes X_{i,j})^\top (I_{d_\tau}\otimes X_{i,j})
    )\\
     & =  \frac{c^{-2}}{n}\sum_{i=1}^nt_{i,j}^\alpha\| X_{i,j}\|_2^{-2} \Tr( \|X_{i,j}\|_2^{2} I_{d_\tau}
    ) = \frac{c^{-2}}{n}\sum_{i=1}^nt_{i,j}^\alpha\| X_{i,j}\|_2^{-2} \|X_{i,j}\|_2^{2} \Tr(   I_{d_\tau}
    ) =  O(n^{\alpha}).
\end{align*}

\paragraph{Checking Property \ref{property:weight_stabilizing}(c).}  
We have
\begin{align*}
    & W_{i,j} \Var_{i,j}^+(\Phi_{i,j}) W_{i,j}^\top\\
  =&   (I_{d_{\tau}}\otimes\hat{V}_{i,j}^{-1/2})
    (I_{d_\tau}\otimes X_{i,j})\Var^+_{i,j}(T_{i,j})^{-1/2} 
   \bp{(I_{d_\tau}\otimes X_{i,j})^\top\cdot\| X_{i,j}\|_2^{-2}}
    (I_{d_\tau}\otimes X_{i,j})\Var^+_{i,j}(T_{i,j}) \\
    &\quad \times (I_{d_\tau}\otimes X_{i,j})^\top
    (I_{d_\tau}\otimes X_{i,j}\cdot\| X_{i,j}\|_2^{-2})\Var^+_{i,j}(T_{i,j})^{-1/2} 
    (I_{d_\tau}\otimes X_{i,j})^\top 
(I_{d_{\tau}}\otimes\hat{V}_{i,j}^{-1/2}) \\
   =&   (I_{d_{\tau}}\otimes\hat{V}_{i,j}^{-1/2})
    (I_{d_\tau}\otimes X_{i,j})\Var^+_{i,j}(T_{i,j})^{-1/2} 
    (\|X_{i,j}\|_2^{2}\cdot I_{d_\tau}\cdot\|X_{i,j}\|_2^{-2})
    \Var^+_{i,j}(T_{i,j})    \\
    &\quad \times (\| X_{i,j}\|_2^{2}\cdot I_{d_\tau}\cdot\| X_{i,j}\|_2^{-2})\Var^+_{i,j}(T_{i,j})^{-1/2} 
    (I_{d_\tau}\otimes X_{i,j})^\top 
  (I_{d_{\tau}}\otimes\hat{V}_{i,j}^{-1/2}) \\
      =&   (I_{d_{\tau}}\otimes\hat{V}_{i,j}^{-1/2}) 
    (I_{d_\tau}\otimes X_{i,j})\Var^+_{i,j}(T_{i,j})^{-1/2} 
    \Var^+_{i,j}(T_{i,j}) \Var^+_{i,j}(T_{i,j})^{-1/2} \\
    &\quad 
    (I_{d_\tau}\otimes X_{i,j})^\top 
(I_{d_{\tau}}\otimes\hat{V}_{i,j}^{-1/2})  \\
    =&  (I_{d_{\tau}}\otimes\hat{V}_{i,j}^{-1/2}) 
    (I_{d_\tau}\otimes X_{i,j})
    (I_{d_\tau}\otimes X_{i,j})^\top
(I_{d_{\tau}}\otimes\hat{V}_{i,j}^{-1/2}) \\
    =&   (I_{d_{\tau}}\otimes\hat{V}_{i,j}^{-1/2})\bp{I_{d_\tau}\otimes 
   (X_{i,j} X_{i,j}^\top) 
   }(I_{d_{\tau}}\otimes\hat{V}_{i,j}^{-1/2}) \\
   =&   I_{d_{\tau}}\otimes(\hat{V}_{i,j}^{-1/2} 
  X_{i,j} X_{i,j}^\top
   \hat{V}_{i,j}^{-1/2}) \\
\end{align*}
With $c\cdot I \preceq \hat{V}_{i,j} \preceq c^{-1}I$,   $c\cdot I \preceq V_{i,j} \preceq c^{-1}I$ and $\| X_{i,j}\|_2^2\leq c^{-1}$, we have
\begin{align*}
    \Tr(W_{i,j} \Var_{i,j}^+(\Phi_{i,j}) W_{i,j}^\top) =O(c^{-2}),
\end{align*}
and
\begin{align*}
    \frac{1}{n}\sum_{i=1}^n \bbE_{i,j}[W_{i,j} \Var_{i,j}^+(\Phi_{i,j}) W_{i,j}^\top] \succeq c^2 \cdot I.
\end{align*}
Finally, we have
\begin{align*}
    &\frac{1}{n}\sum_{i=1}^n \bbE\bb{\|\bbE_{i,j}[W_{i,j} \Var_{i,j}^+(\Phi_{i,j}) W_{i,j}^\top] - I\|_2} 
    = \frac{1}{n}\sum_{i=1}^n \bbE\bb{\|\bbE_{i,j}[ I_{d_{\tau}}\otimes(\hat{V}_{i,j}^{-1/2} 
  X_{i,j} X_{i,j}^\top
   \hat{V}_{i,j}^{-1/2})] - I\|_2}\\
    = & \frac{1}{n}\sum_{i=1}^n \bbE\bb{\|I_{d_{\tau}}\otimes(\hat{V}_{i,j}^{-1/2}V_{i,j}\hat{V}_{i,j}^{-1/2}-I)  \|_2} 
    = \frac{1}{n}\sum_{i=1}^n \bbE\bb{\|\hat{V}_{i,j}^{-1/2}V_{i,j}\hat{V}_{i,j}^{-1/2}-I  \|_2} \\
   \leq & \frac{1}{n}\sum_{i=1}^n \bbE\bb{\|\hat{V}_{i,j}^{-\frac{1}{2}}\|^2_2\|V_{i,j}- \hat{V}_{i,j} \|_2}
   \leq  \frac{c^{-1}}{n}\sum_{i=1}^n \bbE\bb{\|V_{i,j}- \hat{V}_{i,j} \|_2} = O(n^{-\gamma}),
\end{align*}
which also leads to $\frac{1}{n}\sum_{i=1}^n \|\bbE[W_{i,j} \Var_{i,j}^+(\Phi_{i,j}) W_{i,j}^\top] - I\|_2= O(n^{-\gamma})$ by Jensen's inequality.
Therefore, 
\begin{align*}
   & \frac{1}{n}   \sum_{i=1}^n \bbE\left\|
 \bbE_{i,j}[W_{i,j}\Var^+_{i,j}(\Phi_{i,j})W_{i,j}^\top] - \bbE[W_{i,j}\Var_{i,j}^+(\Phi_{i,j})W_{i,j}^\top]
    \right\|_{2}\\
    \leq ~& \frac{1}{n}   \sum_{i=1}^n \bbE\bb{\|
 \bbE_{i,j}[W_{i,j}\Var^+_{i,j}(\Phi_{i,j})W_{i,j}^\top] -I\|_2}+  \|\bbE[W_{i,j}\Var_{i,j}^+(\Phi_{i,j})W_{i,j}^\top]
    -I\|_{2} = O(n^{-\gamma}).
\end{align*}

\section{Supplementary Results for High-dimensional Markovian Models}
\label{appendix:hdmm}
\subsection{Proof of Lemma \ref{lemma:partial_linear_model}}
\label{appendix:partial_linear_model}

\subsubsection{Show that Assumption \ref{assump:exogeneity} is satisfied.}
Given $S_{1:j},T_{1:j-1}$, consider any counterfactual treatment assignments $\tau_{j:L}\in\calT^{L-j+1}$, and denote the corresponding counterfactual context under the sequence of treatments $T_{1:j-1}, \tau_{j:L}$ as $s_{j+1:L}$. By  unrolling the outcome, we have  
\begin{align*}
    \y_i(T_{1:j-1}, \tau_{j:L}) =~& \theta_{j}^\top (\tau_{j}\otimes\chi_{j}(  S_{i,j,\Omega}))) +\sum_{j'>j}\theta_{j'}^\top ( \tau_{j'}\otimes\chi_{j'}(  s_{j',\Omega}))) + \beta_{j}^\top S_{i,j} +  S_{i,1}^\top \kappa_j +\sum_{j'=j}^{L-1}\eta_{i,j'}^\top \beta_{j'+1} + \epsilon_i.
\end{align*}
When conditioning on $S_{i, 1:j}, T_{i,1:j-1}$, the only randomness that remains in $T_j$ is $\zeta_{i,j}$, which is independent of $ \eta_{i,j:L-1}, \epsilon_i$. Moreover, the only remnant randomness in $Y_i$ and in the counterfactual context $s_{j+1:L}$ are the noise terms $ \eta_{i,j:L-1}, \epsilon_i$. Thus, $Y_i(T_{1:j-1}, \tau_{j:L})$ is independent of $T_j$ conditional on $S_{i, 1:j}, T_{i,1:j-1}$, verifying Assumption \ref{assump:exogeneity}.

\subsubsection{Show that Assumption \ref{assump:linear_blip_function} is satisfied.}
We now verify the linear blip function assumption \ref{assump:linear_blip_function}.  First by bilinearity, we have  $\phi_j(s_{1:j}, (\tau_{1:j-1}, 0))=0\otimes \chi(s_{j,\Omega})=\zero$. Then with $0_{1:L}$ as the evaluation policy,  for any $s_{1:j}\in\calS^j$ and $\tau_{1:j-1}\in\calT^{j-1}$, we have that:
\begin{align*}
    \y_i(T_{i,1:j}, 0_{j+1:L}) = (\theta^*_{j})^\top (T_{i,j}\otimes\chi_{j}(  S_{i,j,\Omega})))  + \beta_{j}^\top S_{i,j} +  S_{i,1}^\top \kappa_j +\sum_{j'=j}^{L-1}\eta_{i,j'}^\top \beta_{j'+1} + \epsilon_i.
\end{align*}
Thus the blip function under the baseline policy satisfies that 
\begin{align*}
    \gamma_j(S_{i,1:j}, T_{i,1:j}) &= \bbE\left[ \y_i(T_{i,1:j}, 0_{j+1:L}) - \y_i(T_{i,1:j-1}, 0_{j:L})\mid S_{i, 1:j}, T_{i,1:j}\right]=(\theta_{j}^*)^\top (T_{i,j}\otimes \chi_{j}(  S_{i,j,\Omega}))),
\end{align*}
which has a linear form.
   
\subsubsection{Show that Assumption \ref{assump:homoscedasticity} is satisfied.}
\label{appendix:show_homoscedasticity}
Finally, verifying homoscedasticity Assumption \ref{assump:homoscedasticity} is straightforward. 
Note that by repeatedly unrolling the outcome equation we get:
\[
   \y_i = \sum_{j'=j}^L(\theta^*_{j'})^\top (T_{i,j'}\otimes\chi_{j'}( S_{i,j',\Omega})) + S_{i,j}^\top\beta_j  +  S_{i,1}^\top \kappa_j + 
   \sum_{j'=j}^{L-1}\eta_{i,j'}^\top \beta_{j'+1} + \epsilon_i.
\]
Thus by subtracting the continuation effect of future treatment, we get:
\[
    R_{i,j}:=\y_i - \sum_{j'=j}^L(\theta^*_{j'})^\top (T_{i,j'}\otimes \chi_{j'}( S_{i,j',\Omega})) = S_{i,j}^\top\beta_j  +  S_{i,1}^\top \kappa_j + 
   \sum_{j'=j}^{L-1}\eta_{i,j'}^\top \beta_{j'+1} + \epsilon_i,
\]
We see that conditional on $S_{i,1:j}$, the term on the right hand side is independent of $T_j$. Moreover, we have that:
\begin{align*}
    \bbE\left[R_{i,j}\mid S_{i, 1:j}, T_{1:j-1}\right] = S_{i,j}^\top\beta_j  +  S_{i,1}^\top \kappa_j 
\end{align*}
Thus we see that:
\begin{align*}
    \Var_{i,j}(R_{i,j}) =~& \bbE\bb{ (R_{i,j} - \bbE[R_{i,j}\mid S_{i,j}, T_{t_{i,j}-1}])^2\mid S_{i,j}, T_{t_{i,j}-1}}\\
    =~&  \bbE\bb{\bp{ \sum_{j'=j}^{L-1}\eta_{i,j'}^\top \beta_{j'+1} + \epsilon_i}^2 \mid S_{i,j}, T_{1:j-1}}\\
    =~& \bbE\bb{ \bp{ \sum_{j'=j}^{L-1}\eta_{i,j'}^\top \beta_{j'+1} + \epsilon_i}^2 },
\end{align*}
which has a universal lower bound $\bbE[\epsilon_i^2]$.

Lastly, the Model \ref{algo:plmdgp} satisfies Assumption \ref{ass:bilinear} by construction.
\subsection{Proof of Lemma \ref{lemma:homoscedasticity_plm}}
The result in Lemma \ref{lemma:homoscedasticity_plm} has already been shown in Appendix \ref{appendix:show_homoscedasticity}.

\subsection{Estimating $f_{j}$}
We prove the following lemma.

\begin{lemma}
\label{lemma:estimate_f}
Consider Model \ref{algo:plmdgp}.  Fix a stage $j$. Let $\hat{\theta}^{(C)}_{i,j}$ and $\hat{g}_{i,j}$ be consistent estimates of $\theta^*$ and $g_j$ respectively, using the previous $i-1$ episodes,  which satisfy that,
$
\label{eq:nuisance_consistency}
\bbE[\|\hat{\theta}^{(C)}_{i,j}-\theta^*\|_2]= O(i^{-\gamma_\theta})\quad  \mbox{and}  \quad   \|\hat{g}_{i,j} - g_j\|_{1,\infty}=O(i^{-\gamma_G}).
$
Let $\hat{f}_{i,j}(\cdot)=\Clip_{[\sigma^2,M^2]}\bp{\hat{g}_{i,j}(\cdot)^2 + \hat{\sigma}^2_{i,j}}$, with  $\hat{\sigma}_{i,j}$ as the variance estimator given in Algorithm \ref{algo:f} and the clipping operator defined in \eqref{eq:clip_scalar}. 
We have, $\|\hat{f}_{i,j} - f_j\|_{1,\infty}= O(i^{-\min(\gamma_\theta,\gamma_G, 1/2)})$.
\end{lemma}

\begin{proof}
\paragraph{Convergence of $\hat{\sigma}$.} Fix a stage $j$.
Define 
\begin{align*}
    \Tilde{\sigma}^2_{i,j} = \frac{1}{i-1}\sum_{i'=1}^{i-1} \bp{\y_{i'}-\sum_{j'=j}^L  \Phi_{i',j'}^\top \theta_{j'}^* - G_{i',j}}^2= \frac{1}{i-1}\sum_{i'=1}^{i-1}(R_{i',j} - \bbE_{i',j}^+[R_{i',j}])^2,
\end{align*}
where recall that $G_{i,j}=\bbE_{i,j}^+[R_{i,j}]$.

  Lemma \ref{lemma:homoscedasticity_plm} shows  the homoscedasticity of $R_{i,j}$,  
which implies $(R_{i',j} - \bbE_{i',j}^+[R_{i',j}])^2$ is i.i.d.~across $i'$ and has the same  mean value  $\sigma_j^2$ for all $i'$.
Therefore, 
\begin{align*}
    \bbE[(\Tilde{\sigma}^2_{i,j}-\sigma_j^2)^2]&=\bbE\bb{\bc{\frac{1}{i-1}\sum_{i'=1}^{i-1}\bp{\y_{i'}-\sum_{j'=j}^L \Phi_{i',j'}^\top \theta_{j'}^*- G_{i',j}}^2 - \sigma_j^2}^2}\\
    &=\frac{\sum_{i'=1}^{i-1} \bbE[\{(\y_{i'}-\sum_{j'=j}^L \Phi_{i',j'}^\top \theta_{j'}^*- G_{i',j})^2 - \sigma_j^2\}^2]}{(i-1)^2} = O((i-1)^{-1}).
\end{align*}
So it holds that
\begin{equation}
\label{eq:sigma_consistency_1}
    \bbE[|\Tilde{\sigma}^2_{i,j}-\sigma_j^2|] \leq  \bbE[(\Tilde{\sigma}^2_{i,j}-\sigma_j^2)^2]^{1/2} = O(i^{-1/2}).
\end{equation}
It remains to characterize the difference between $\Tilde{\sigma}^2_{i,j} $ and $ \hat{\sigma}_{i,j}^2$. We have
\begin{align}
   &\left| \Tilde{\sigma}^2_{i,j} - \hat{\sigma}^2_{i,j}\right| \nonumber \\
   =~&\Bigg| \frac{1}{i-1}\sum_{i'=1}^{i-1} \Bigg\{\bp{\hat{g}_{i,j}(S_{i',1:j}, T_{i', 1:j-1})- G_{i',j} + \sum_{j'=j}^L\Phi_{i',j'}^\top (\hat{\theta}^{(C)}_{t_{i,j},j'}-\theta_{j'}^*)}\nonumber\\
   &\quad\quad \times \bp{
   2\y_{i'} -\sum_{j'=j}^L (\theta_{j'}^*+\hat{\theta}^{(C)}_{t_{i,j},j'})^\top \Phi_{i',j'} - G_{i',j}-\hat{g}_{i,j}(S_{i',1:j}, T_{i', 1:j-1})
   }\Bigg\}\Bigg|\nonumber \\
\stackrel{(i)}{\lesssim}~&\frac{1}{i-1}\sum_{i'=1}^{i-1}\bc{\ba{
\hat{g}_{i,j}(S_{i',1:j}, T_{i', 1:j-1})- G_{i',j}
}
 + \|\hat{\theta}_{i,j}^{(C)}-\theta^*\|_2}\nonumber \\
 =~& \frac{1}{i-1}\sum_{i'=1}^{i-1}\ba{
\hat{g}_{i,j}(S_{i',1:j}, T_{i', 1:j-1})- G_{i',j}
} +\|\hat{\theta}^{(C)}_{i,j}-\theta^*\|_2,
 \label{eq:convergence_sigma_1}
\end{align}
where 
(i) uses the boundedness of $\Phi_{i,j}$ and  $ 2\y_{i'} -\sum_{j'=j}^L (\theta_{j'}^*+\hat{\theta}_{n,j'})^\top \Phi_{i',j'} - G_{i',j}-\hat{g}_{i,j}(S_{i',1:j}, T_{i', 1:j-1})$.

Continuing \eqref{eq:convergence_sigma_1}, we have that 
\begin{equation}
\label{eq:sigma_consistency_2}
\begin{aligned}
      \bbE\bb{|\Tilde{\sigma}^2_{i,j} - \hat{\sigma}^2_{i,j}|} & \lesssim \frac{1}{i-1}\sum_{i'=1}^{i-1} \bbE\bb{
  |\hat{g}_{i,j}(S_{i',1:j}, T_{i', 1:j-1})- g_j(S_{i',1:j}, T_{i', 1:j-1})|
    } + \bbE\bb{\|\hat{\theta}_{i,j}^{(C)}-\theta^*\|_2}\\
  &\leq   \|\hat{g}_{i,j}-g_j\|_{1,\infty} + \bbE\bb{\|\hat{\theta}_{i,j}^{(C)}-\theta^*\|_2}  .
\end{aligned}
\end{equation}
Combining \eqref{eq:sigma_consistency_1} and \eqref{eq:sigma_consistency_2}, we have
\begin{equation}
     \bbE\bb{|\sigma^2_{j} - \hat{\sigma}^2_{i,j}|} \lesssim  \|\hat{g}_{i,j}-g_j\|_{1,\infty} + \bbE\bb{\|\hat{\theta}_{i,j}^{(C)}-\theta^*\|_2} + O(i^{-1/2}).
\end{equation}

\paragraph{Convergence of $\hat{f}_{i,j}$.}  We have, for any $(s_{1:j},\tau_{1:j-1})\in\calS^{j}\times \calT^{j-1}$:
\begin{align*}
  &|\hat{f}_{i,j}(s_{1:j},\tau_{1:j-1})-f_{j}(s_{1:j},\tau_{1:j-1})|
 \\
=&\ba{ \Clip_{[\sigma^2,M^2]}\bp{\hat{g}_{i,j}(\cdot)^2 + \hat{\sigma}^2_{i,j}}- g_{j}(s_{1:j},\tau_{1:j-1})^2 - \sigma_j^2  }\\
=&\ba{ \min(\max(\hat{g}_{i,j}(\cdot)^2 + \hat{\sigma}^2_{i,j},\sigma^2),M^2)- g_{j}(s_{1:j},\tau_{1:j-1})^2 - \sigma_j^2  }\tag{by clipping definition in \eqref{eq:clip_scalar}}\\
= & \ba{ \hat{g}_{i,j}(\cdot)^2 + \hat{\sigma}^2_{i,j}- g_{j}(s_{1:j},\tau_{1:j-1})^2 - \sigma_j^2  } + \min\bp{\hat{g}_{i,j}(\cdot)^2 + \hat{\sigma}^2_{i,j}-\sigma^2, 0} + \min\bp{M^2-\hat{g}_{i,j}(\cdot)^2 - \hat{\sigma}^2_{i,j}, 0}
\tag{since $g_{j}(s_{1:j},\tau_{1:j-1})^2 + \sigma_j^2\in[\sigma^2, M^2]$}
\\
\leq & 
\ba{ \hat{g}_{i,j}(s_{1:j},\tau_{1:j-1})^2 + \hat{\sigma}_{i,j}^2 - g_{j}(s_{1:j},\tau_{1:j-1})^2 - \sigma_j^2  }\\
  \leq &
  \ba{ \hat{g}_{i,j}(s_{1:j},\tau_{1:j-1})-  g_{j}(s_{1:j},\tau_{1:j-1})}\cdot \ba{ \hat{g}_{i,j}(s_{1:j},\tau_{1:j-1})+  g_{j}(s_{1:j},\tau_{1:j-1})}
  + |\sigma_j^2 -\hat{\sigma}_{i,j}^2|\\
 \lesssim & \ba{ \hat{g}_{i,j}(s_{1:j},\tau_{1:j-1})-  g_{j}(s_{1:j},\tau_{1:j-1})}+ |\sigma_j^2 -\hat{\sigma}_{i,j}^2|.
\end{align*}

We therefore have:
\begin{equation}
    \|\hat{f}_{i,j}-f_j\|_{1,\infty} \lesssim
    \|\hat{g}_{i,j}-g_j\|_{1,\infty} + \bbE\bb{\|\hat{\theta}_{i,j}^{(C)}-\theta^*\|_2} = O(i^{-\min(\gamma_\theta,\gamma_G,1/2)}).
\end{equation}

\end{proof}

\subsection{Estimating $\nu_{j}$}
We prove the following lemma.
\begin{lemma}
    \label{lemma:estimate_v}
Fix a stage $j>1$. Let  $\hat{h}_{i,j}$ be a consistent estimate of  $h_{j}$ using the previous $i-1$ episodes, satisfying that  $\|\hat{h}_{i,j} - h_{j}\|_{1,\infty}=O(i^{-\gamma_h})$. Define
\begin{equation}
    \check\nu_{i,j}(\cdot)=\frac{1}{i-1}\sum_{i'=1}^{i-1}  \chi\big(\hat{h}_{i,j}(\cdot) + S_{i',j,\Omega}- \hat{h}_{i,j}(S_{i',1:j-1},  T_{i,1:j-1})\big)\chi\big(\hat{h}_{i,j}(\cdot) + S_{i',j,\Omega}- \hat{h}_{i,j}(S_{i',1:j-1},  T_{i,1:j-1})\big)^\top.
\end{equation}
Define $\hat{\nu}_{i,j}(\cdot)=\Clip_{[c,c^{-1}]}\bp{\check\nu_{i,j}(\cdot)}$, where the clipping operator is given in \eqref{eq:clip_matrix}.
Then $\hat{\nu}_{i,j}$ satisfies
     $ \|\hat{\nu}_{i,j}-\nu_{j}\|_{1,\infty} = O(i^{-\min(\gamma_h, 1/2)})$.
\end{lemma}

\begin{proof}
Fix a stage index $j$. We first show that
 $ \|\check{\nu}_{i,j}-\nu_{j}\|_{1,\infty} = O(i^{-\min(\gamma_h, 1/2)})$.
We use $\check{\nu}_{i,j,k_1,k_2}$ and $\nu_{k_1,k_2}$ to denote the $(k_1,k_2)$-th entry of $\check{\nu}_{i,j}$ and $\nu$ respectively. Similarly, we use $\chi_k$ to denote the $k$-th entry of the vector $\chi$.
The condition assumes that 
 $ \|\hat{h}_{i,j} - h_{j}\|_{1,\infty}=O(i^{-\gamma_h})$. We have
\begin{equation*}
 \check\nu_{i,j,k_1,k_2}(\cdot)=\frac{1}{i-1}\sum_{i'=1}^{i-1}  \chi_{k_1}\big(\hat{h}_{i,j}(\cdot) + S_{i',j,\Omega}- \hat{h}_{i,j}(S_{i',1:j-1},  T_{i',1:j-1})\big)
 \chi_{k_2}\big(\hat{h}_{i,j}(\cdot) + S_{i',j,\Omega}- \hat{h}_{i,j}(S_{i',1:j-1},  T_{i',1:j-1})\big).
  \end{equation*}
 Define 
  \begin{align*}
   \tilde\nu_{i,j,k_1,k_2}(\cdot)=\frac{1}{i-1}\sum_{i'=1}^{i-1}  \chi_{k_1}\big(h_{i,j}(\cdot) + S_{i',j,\Omega}- h_{i,j}(S_{i',1:j-1},  T_{i',1:j-1})\big)\chi_{k_2}\big(h_{i,j}(\cdot) + S_{i',j,\Omega}- h_{i,j}(S_{i',1:j-1},  T_{i',1:j-1})\big).
 \end{align*}
 We have, for any $(s_{1:j-1},\tau_{1:j-1})\times \calS^{j-1}\times \calT^{j-1}$,
 \begin{align*}
    &| \tilde{\nu}_{i,j,k_1,k_2}(s_{1:j-1},\tau_{1:j-1}) - \check{\nu}_{i,j,k_1,k_2}(s_{1:j-1},\tau_{1:j-1})|\\
    &\stackrel{(i)}{\lesssim} \frac{1}{i-1} \sum_{i'=1}^{i-1}\Big(|\chi_{k_1}\big(\hat{h}_{i,j}(\cdot) + S_{i',j,\Omega}- \hat{h}_{i,j}(S_{i',1:j-1},  T_{i',1:j-1})\big)
    -\chi_{k_1}\big(h_{i,j}(\cdot) + S_{i',j,\Omega}- h_{i,j}(S_{i',1:j-1},  T_{i',1:j-1})\big)
    | \\
    &\quad 
    +
    |\chi_{k_2}\big(\hat{h}_{i,j}(\cdot) + S_{i',j,\Omega}- \hat{h}_{i,j}(S_{i',1:j-1},  T_{i',1:j-1})\big)
    -
    \chi_{k_2}\big(h_{i,j}(\cdot) + S_{i',j,\Omega}- h_{i,j}(S_{i',1:j-1},  T_{i',1:j-1})\big)
    |
    \Big)\\
    &\stackrel{(ii)}{\lesssim} \frac{1}{i-1} \sum_{i'=1}^{i-1}\bp{|\hat{h}_{i,j}(s_{1:j-1},\tau_{1:j-1}) -  h_j(s_{1:j-1},\tau_{1:j-1})| + | \hat{h}_{i,j}(S_{i',1:j-1}, T_{i',1:j-1}) -h_j(S_{i',1:j-1}, T_{i',1:j-1}) |}\\
    &=\Big|\hat{h}_{i,j}(s_{1:j-1},\tau_{1:j-1}) -  h_j(s_{1:j-1},\tau_{1:j-1})\Big| +  \frac{1}{i-1}\sum_{i'=1}^{i-1}\Big| \hat{h}_{i,j}(S_{i',1:j-1}, T_{i',1:j-1}) -h_j(S_{i',1:j-1}, T_{i',1:j-1}) \Big|\\
    &\leq \max_{s'_{2:j-1},\tau'_{1:j-1}}|\hat{h}_{i,j}(s_{1}, s'_{2:j-1},\tau'_{1:j-1}) -  h_j(s_{1}, s'_{2:j-1},\tau'_{1:j-1})|\\
   &\quad \quad +\frac{1}{i-1}\sum_{i'=1}^{i-1}\max_{s'_{2:j-1},\tau'_{1:j-1}}|\hat{h}_{i,j}(S_{i',1}, s'_{2:j-1},\tau'_{1:j-1}) -h_j(S_{i',1}, s'_{2:j-1},\tau'_{1:j-1})|
 \end{align*}
where (i) uses that $\chi$ is bounded
and  (ii) uses that $\chi$ is Lipschitz. 
Taking the expectation on both sides with respect to $s_1$ and $S_{i',1}$, we have
\begin{equation}
    \|\tilde{\nu}_{i,j,k_1,k_2} - \check\nu_{j,k_1,k_2}\|_{1,\infty}
\lesssim 2\|\hat{h}_{i,j} - h_j\|_{1,\infty}.
\end{equation}
We next show that $\tilde{\nu}_{i,j,k_1,k_2}$ approximates $\nu_{j,k_1,k_2}$ at rate $i^{-1/2}$.
Fix indices $i,j$. For any $(s_{1:j-1},\tau_{1:j-1})\in \calS^{j-1}\times \calT^{j-1}$ in the bounded state domain, we have 
\begin{align*}
   &\nu_{j,k_1,k_2}(s_{1:j-1},\tau_{1:j-1}) = \bbE_{\eta_{i,j}}[\chi_{k_1}(h(s_{1:j-1},\tau_{1:j-1})+\eta_{i,j})
   \chi_{k_2}(h(s_{1:j-1},\tau_{1:j-1})+\eta_{i,j})
   ] \\
\mbox{and}\quad &\tilde{\nu}_{i,j,k_1,k_2}(s_{1:j-1},\tau_{1:j-1}) = \frac{1}{i-1} \sum_{i'=1}^{i-1}\chi_{k_1}(h(s_{1:j-1},\tau_{1:j-1})+\eta_{i',j})\chi_{k_2}(h(s_{1:j-1},\tau_{1:j-1})+\eta_{i',j}).   
\end{align*}
Note that $\{\eta_{i',j}\}$ are i.i.d.~as $\eta_{i,j}$, and so $\tilde{\nu}_{i,j,k_1,k_2}(s_{1:j-1},\tau_{1:j-1})$ is the i.i.d.~empirical average for $\nu_{j,k_1,k_2}(s_{1:j-1},\tau_{1:j-1})$. With $\chi(\cdot)$  bounded, we have
\[
\bbE_{\eta_{i,j}}[|\nu_{j,k_1,k_2}(s_{1:j-1},\tau_{1:j-1})- \tilde{\nu}_{i,j,k_1,k_2}(s_{1:j-1},\tau_{1:j-1})|] = O(i^{-1/2}), \quad \forall (s_{1:j-1},\tau_{1:j-1})\in\calS^{j-1}\times \calT^{j-1}.
\]
yielding:
\[
\|\nu_{j,k_1,k_2}-\tilde\nu_{i,j,k_1,k_2}\|_{1,\infty} = O(i^{-1/2}).
\]
Collectively, we have
\begin{align*}
\|\nu_{j,k_1,k_2}-\check\nu_{i,j,k_1,k_2}\|_{1,\infty} \leq \|\nu_{j,k_1,k_2}-\tilde\nu_{i,j,k_1,k_2}\|_{1,\infty} + \|\tilde\nu_{i,j,k_1,k_2}-\check\nu_{i,j,k_1,k_2}\|_{1,\infty}= O(i^{-\min(\gamma_h, 1/2)}).
\end{align*}
Finally, we have 
\begin{align*}
    \|\nu_{j}-\check\nu_{i,j}\|_{1,\infty} &= \bbE_{s_1\sim P_S}\big[\sup_{(s_{2:j_1},\tau_{1:j_2})\in\calS^{j_1-1}\times \calT^{j_2}} \|\{\nu_{j}-\check\nu_{i,j}\}(s_{1:j_1}, \tau_{1:j_2})\|_{2}\big]\\
    &\leq  \bbE_{s_1\sim P_S}\big[\sup_{(s_{2:j_1},\tau_{1:j_2})\in\calS^{j_1-1}\times \calT^{j_2}} \|\{\nu_{j}-\check\nu_{i,j}\}(s_{1:j_1}, \tau_{1:j_2})\|_{Frob}\big]\\
       &\leq  \bbE_{s_1\sim P_S}\big[\sup_{(s_{2:j_1},\tau_{1:j_2})\in\calS^{j_1-1}\times \calT^{j_2}} \sum_{k_1,k_2}|\{\nu_{j,k_1,k_2}-\check\nu_{i,j,k_1,k_2}\}(s_{1:j_1}, \tau_{1:j_2})|\big]\\
       &\leq \bbE_{s_1\sim P_S}\big[\sum_{k_1,k_2}\sup_{(s_{2:j_1},\tau_{1:j_2})\in\calS^{j_1-1}\times \calT^{j_2}} |\{\nu_{j,k_1,k_2}-\check\nu_{i,j,k_1,k_2}\}(s_{1:j_1}, \tau_{1:j_2})|\big]\\
       &= \sum_{k_1,k_2}\|\nu_{j,k_1,k_2}-\check\nu_{i,j,k_1,k_2}\|_{1,\infty} = O(i^{-\min(\gamma_h, 1/2)}).
\end{align*}

Now we show that  $ \|\nu_{j}-\hat\nu_{i,j}\|_{1,\infty}\lesssim  \log(d_x)^2\|\nu_{j}-\check\nu_{i,j}\|_{1,\infty}$. We have:
\begin{align*}
    \|\nu_{j}-\hat\nu_{i,j}\|_{1,\infty} &= \bbE_{s_1\sim P_S}\big[\sup_{(s_{2:j_1},\tau_{1:j_2})\in\calS^{j_1-1}\times \calT^{j_2}} \|\{\nu_{j}-\hat\nu_{i,j}\}(s_{1:j_1}, \tau_{1:j_2})\|_{2}\big]\\
     &= \bbE_{s_1\sim P_S}\big[\sup_{(s_{2:j_1},\tau_{1:j_2})\in\calS^{j_1-1}\times \calT^{j_2}} \|\{\nu_{j}-\Clip_{[c,c^{-1}]}\bp{\check\nu_{i,j}}\}(s_{1:j_1}, \tau_{1:j_2})\|_{2}\big]\tag{by definition of $\hat\nu$ and the clipper operator in \eqref{eq:clip_matrix}}\\
       &= \bbE_{s_1\sim P_S}\big[\sup_{(s_{2:j_1},\tau_{1:j_2})\in\calS^{j_1-1}\times \calT^{j_2}} \|\{
       \Clip_{[c,c^{-1}]}\bp{\nu_{j}}-\Clip_{[c,c^{-1}]}\bp{\check\nu_{i,j}}\}(s_{1:j_1}, \tau_{1:j_2})\|_{2}\big]\tag{true $\nu_{j}$ has eigenvalues within $[c,c^{-1}]$}\\
     &\lesssim
     \bbE_{s_1\sim P_S}\big[\sup_{(s_{2:j_1},\tau_{1:j_2})\in\calS^{j_1-1}\times \calT^{j_2}} \log(d_x)^2\|\{\nu_{j}-\check\nu_{i,j}\}(s_{1:j_1}, \tau_{1:j_2})\|_{2}\big]
     \tag{by Lemma \ref{lemma:clip}}
     \\
     & = \log(d_x)^2  \|\nu_{j}-\check\nu_{i,j}\|_{1,\infty} = \log(d_x)^2  O(i^{-\min(\gamma_h, 1/2)}).
\end{align*}
\end{proof}

\subsection{Estimation Guarantees of $\hat{g}_{i,j}$ and $\hat{h}_{i,j}$}
\label{appendix:plmm_nuisance_estimation}
We prove the following lemma.
\begin{lemma}
\label{lemma:estimation_rate_plmm}
Consider Model \ref{algo:plmdgp}.  Suppose $\bbE_{i,j}[S_{i,j}S_{i,j}^\top]\preceq C_s I$
almost surely for a universal constant $C_s$. Suppose the initial state distribution $P_s$ and the state transition noise $\eta_{i,j}$ have positive-definite covariance matrices with smallest eigenvalues greater than a universal constant $c_s$.
Fix a stage $j$. 

\begin{itemize}
    \item  For the $g_j(\cdot)$ in \eqref{eq:g_hdmm}, we estimate it using   $\hat{g}_{i,j}$ based on the  Lasso loss:
    \begin{equation}
         (\hat{\beta}_{i,j}, \hat{\kappa}_{i,j})= \argmin_{\tilde{\beta}, \tilde{\kappa}} \frac{1}{i-1}\sum_{i'=1}^{i-1} \bp{\hat{R}_{i',j}-\tilde{\beta}^\top S_{i',j} -  \tilde{\kappa}^\top S_{i',1}}^2 + \lambda_g (\|\tilde{\beta}\|_1+ \|\tilde{\kappa}\|_1)\tag{\ref{eq:lasso_g}},
    \end{equation}
    where $\hat{R}_{i,j}:=\y_i- \sum_{j'=j}^L \Phi_{i,j'}^\top \hat{\theta}_j$ is constructed via  an estimate  $\hat{\theta}$ of $\theta^*$.
    Let $s$ be $\|(\beta_{i,j};\kappa_{i,j})\|_0$, where $\|x\|_0$ is the number of nonzero elements in $x$.
      Then  $ \|\hat{g}_{i,j}-g_j\|_{1,\infty} =O\big(\frac{sC_s}{c_s^2} \big(\sqrt{\frac{\log(d_s/2)}{i}}+\bbE\|\hat{\theta}-\theta^*\|_2 \big)\big)$.

    \item 
    For the $h_{j}(\cdot)$ in \eqref{eq:h_hdmm}, we estimate it using   $\hat{h}_{i,j}$ based on the  Lasso loss: 
        \begin{align}
       (\hat{A}_{i,j,k}, \hat{B}_{i,j,k}, \hat{M}_{i,j,k}) &= \argmin_{(\tilde{A}_k, \tilde{B}_k, \tilde{M}_k)} \frac{1}{i-1}\sum_{i'=1}^{i-1} \bp{S_{i',j,k}- \tilde{A}_k^\top \Phi_{i',j-1}- \tilde{B}_k^\top S_{i',j-1} - \tilde{M}_k^\top S_{i',1})}^2\nonumber \\
      &\quad\quad\quad\quad\quad\quad\quad\quad + \lambda_k ( \|\tilde{A}_k\|_1 +\|\tilde{B}_k\|_1 +\|\tilde{M}_k\|_1  ), \quad\quad \mbox{for}\quad k\in\Omega;\tag{\ref{eq:lasso_h}}
      \end{align}
    Let $\hat{h}_{i,j}=(\hat{h}_{i,j,k}; k\in \Omega)$. Let $s$ be $\max_{k\in\Omega}\|(A_{i,j,k};B_{i,j,k};M_{i,j,k})\|_0$.
       Suppose that $\bbE_{i,j-1}[\Phi_{i,j-1}\Phi_{i,j-1}^\top]\preceq C_s I$. Under Assumption \ref{assump:overlap}, 
      $
                    \|\hat{h}_{i,j} - h_{j}\|_{1,\infty} = O\big(
  \frac{\sqrt{d_\Omega} sC_s^2}{c_s^2}\sqrt{\frac{\log(d_s/2)}{i^{1-2\alpha}}}\big),
      $
      where $\alpha$ is the exploration rate of the behavior policy as specified in \eqref{eq:plmm_behavior_policy_rate}.
\end{itemize}
\end{lemma}

For notation convenience, in the next we establish estimation guarantees when using $n$ episodes $\{(S_{i,1},T_{i,1}, \dots, S_{i,L}, T_{i,L}, Y_i)\}_{i\in[n]}$, and we use $\hat{\theta}$ to replace $\hat{\theta}_{i,j}^{(C)}$ that denotes some given consistent estimate of structural parameter $\theta^*$. Formal results in Lemma \ref{lemma:estimation_rate_plmm} can be directly obtained by replacing $n$ with $i$ in the notations.
We will use the following theorem to show Lemma \ref{lemma:estimation_rate_plmm} (proof  deferred to Appendix \ref{appendix:proof_general_lasso}).

\begin{theorem}
\label{thm:lasso_general}
Consider  $n$ data points $\{(x_i, z_i)\}_{i=1}^n$, where  all random variables are bounded. Consider scenarios where we can only observe data with (polluted) labels $(x_i, \hat{z}_i)$.
Let $f_0(x)=x^\top \theta_0$ with $\theta_0$ being sparse with support $S$ and suppose that:
\[
z_i=f_0(x) + \eta_i
\]
with $\eta_i$  i.i.d.~and uniformly bounded.
Let $\hat{\theta}_n$ be the solution to the Lasso loss defined on the (polluted) data:
\begin{equation}
\label{eq:lasso_loss}
\hat{\theta}_n= \argmin_{\theta} \frac{1}{n}\sum_{i=1}^n (\hat{z}_i- x_i^\top \theta)^2 + \lambda \|\theta\|_1.
\end{equation}
Suppose that the empirical covariance matrix $\Sigma_n = \frac{1}{n}\sum_{i=1}^n x_i\, x_i^\top$ satisfies the following restricted strong convexity property: for any $\nu$ in the restricted cone defined by the inequality $\|\nu_{S^c}\|_1\leq 3\|\nu_S\|_1$, we have that for $n$ larger than some constant that can depend on $\delta,p_x$ and $s:=|S|$:
\begin{align}
\label{eq:strong_convexity_sigma_n}
\nu^\top \Sigma_n \nu \geq \gamma \|\nu\|_2^2 \quad \mbox{w.p.}~1-\delta,
\end{align}
Choosing 
\begin{equation}
\label{eq:lasso_regularizer}
    \lambda\geq 2\bn{\frac{1}{n}\sum_{i=1}^n (\hat{z}_i-\theta_0^\top x_i)x_i}_{\infty},
\end{equation} for large enough $n$, we have, w.p. $1-\delta$:
\begin{align}
        \|\hat{\theta}_n-\theta_0\|_2 \leq~& 12 \frac{\lambda \sqrt{s}}{\gamma}, &
        \|\hat{\theta}_n-\theta_0\|_1 \leq~& 48 \frac{\lambda s}{\gamma}.
\end{align}
\end{theorem}

\subsubsection{Estimation guarantee for $g_j$.}
 Fix a stage index $j$. Under Model \ref{algo:plmdgp},
 \[
  g_j(S_{i, 1}, S_{i,j}) = \bbE_{i,j}\bb{\y_i- \sum_{j'=j}^L \Phi_{i,j'}^\top\theta_{j'}^* } 
    = S_{i,j}^\top\beta_j  +  S_{i,1}^\top \kappa_j,
\]
where $\beta_j,\kappa_j$ are sparse linear vectors.
Using the notation in Theorem  \ref{thm:lasso_general}, we have the explainable variable $x_i:=(S_{i,1}, S_{i,j})$ and the dependent variable, 
\[z_i:=R_{i,j}=\y_i- \sum_{j'=j}^L \Phi_{i,j'}^\top\theta_{j'}^*=g_j(S_{i, 1}, S_{i,j})+\epsilon_{i,j},\]
where $\epsilon_{i,j}= \sum_{j'=j}^{L-1}\eta_{i,j'}^\top \beta_{j'+1} + \epsilon_i$ as defined in \eqref{eq:final_unroll_plm}.
 However, we can only observe the polluted labels $\hat{z}_i=\hat{R}_{i,j}:=\y_i- \sum_{j'=j}^L \Phi_{i,j'}^\top \hat{\theta}_j$, constructed via  a consistent estimate  $\hat{\theta}$ of $\theta^*$. With bounded state, we have:   
 \[\bbE[|z_i-\hat{z_i}|]=O(\bbE[\|\hat{\theta}-\theta^*\|_1]).\] 

We first verify condition \eqref{eq:strong_convexity_sigma_n}.
Let $\Sigma_1=\frac{1}{n}\sum_{i=1}^n \bbE[S_{i,1} S_{i,1}^\top \mid \calF_{i,1}]$ and $\Sigma_j = \frac{1}{n}\sum_{i=1}^n \bbE[S_{i,j} S_{i,j}^\top \mid \calF_{i,j}]$ and $\Sigma_{1,n}=\frac{1}{n}\sum_{i=1}^n S_{i,1} S_{i,1}^\top$ and $\Sigma_{j,n}=\frac{1}{n}\sum_{i=1}^n S_{i,j} S_{i,j}^\top$. Similarly, let $\Sigma_n=\frac{1}{n}\sum_{i=1}^n x_ix_i^\top$.  By a  Hoeffding-Azuma inequality, with bounded covariates, we can also derive that, w.p.~$1-\delta$,
\begin{align*}
    \|\Sigma_{1,n} - \Sigma_1\|_{\infty} \lesssim \sqrt{\frac{\log(d_s/\delta)}{n}},  \quad \|\Sigma_{j,n} - \Sigma_j\|_{\infty} \lesssim \sqrt{\frac{\log(d_s/\delta)}{n}}.
\end{align*}
For any $\nu=(\nu_1; \nu_j)$ with the restricted strong convexity property, which yields  $\|\nu\|_1\leq 4\sqrt{s}\|\nu\|_2$ by \eqref{eq:restricted_strong_convexity_implication}, where $\nu_1$ and $\nu_j$ have the same dimensions as $S_{i,1}$ and $S_{i,j}$ respectively, 
first note that:
\begin{align*}
    \nu^\top\Sigma_n\nu = \frac{1}{n}\sum_{i} (\nu^\top x_i)^2 \geq~& \frac{1}{2} \frac{1}{n}\sum_{i} (\nu_1^\top S_{i,1})^2 - \frac{1}{n}\sum_{i} (\nu_j^\top S_{i,j})^2 =~ \frac{1}{2}\nu_1^\top\Sigma_{1,n}\nu_1 - \nu_j'\Sigma_{j,n}\nu_j
    \tag{We use $(a+b)^2\geq \frac{1}{2}a^2-b^2$}
    \\
    \gtrsim~& \frac{1}{2} \nu_1^\top\Sigma_1\nu_1 - \nu_j^\top \Sigma_j\nu_j - \sqrt{\frac{\log(d_s)/\delta)}{n}}(\|\nu_1\|_1^2 + \|\nu_j\|_1^2)\\
    \geq~& \frac{1}{2} \nu_1^\top\Sigma_1\nu_1- \nu_j^\top \Sigma_j\nu_j - 2\sqrt{\frac{\log(d_s)/\delta)}{n}}(\|\nu_1\|_1 + \|\nu_j\|_1)^2\\
    =~& \frac{1}{2} \nu_1^\top\Sigma_1\nu_1 - \nu_j^\top \Sigma_j\nu_j - 2\sqrt{\frac{\log(d_s)/\delta)}{n}}\|\nu\|_1^2\\
    \geq~& \frac{1}{2} \nu_1^\top\Sigma_1\nu_1 - \nu_j^\top \Sigma_j\nu_j - 32s\sqrt{\frac{\log(d_s)/\delta)}{n}}\|\nu\|_2^2\\
    =~& \frac{1}{2} \nu_1^\top\Sigma_1\nu_1 - \nu_j^\top \Sigma_j\nu_j - 32s\sqrt{\frac{\log(d_s)/\delta)}{n}}(\|\nu_1\|_2^2 + \|\nu_j\|_2^2)\\
    \geq ~& \frac{1}{2} c_s \|\nu_1\|_2^2 - C_s \|\nu_j\|_2^2 - 32s\sqrt{\frac{\log(d_s)/\delta)}{n}}(\|\nu_1\|_2^2 + \|\nu_j\|_2^2),
\end{align*}
where we use that $\Sigma_1 \succeq c_sI$ and $\Sigma_j\preceq C_s I$.
Thus for $n$ larger than some constant that depends on $c_s, C_s, s$, we have that:
\begin{align}
\label{eq:sigma_n_1}
    \nu^\top\Sigma_n\nu \gtrsim~& \frac{1}{4} c_s \|\nu_1\|_2^2 - 2C_s \|\nu_j\|_2^2.
\end{align}
Now we lower bound $\nu^\top\Sigma_n\nu$ by some function of $\|\nu_j\|_2^2$ with the goal of achieving  \eqref{eq:strong_convexity_sigma_n}. We have 
\begin{align*}
    \nu^\top\Sigma_n\nu &= \frac{1}{n}\sum_{i} (\nu_1^\top S_{i,1} + \nu_j^\top S_{i,j})^2 = \frac{1}{n}\sum_{i} (\nu_1^\top S_{i,1} + \nu_j^\top (S_{i,j}-\eta_{i,j}+\eta_{i,j}))^2\\
    &\geq \frac{1}{n} \sum_{i=1}^n (\nu_j^\top \eta_{i,j})^2 + 2  \frac{1}{n}\sum_{i}(\nu_1^\top S_{i,1} + \nu_j^\top (S_{i,j}-\eta_{i,j}))\nu_j^\top \eta_{i,j}\\
    &\geq \frac{1}{n} \sum_{i=1}^n (\nu_j^\top \eta_{i,j})^2 -\|\nu_1\|_1\|\nu_j\|_1\bn{\frac{1}{n}\sum_i S_{i,1}\eta_{i,j}'}_{\infty} - \|\nu_j\|_1^2\bn{\frac{1}{n}\sum_i (S_{i,j}-\eta_{i,j})\eta_{i,j}^\top}_{\infty} .
 \end{align*}
Recall that $\eta_{i,j}$ is denotes the state transition noise in Model~\ref{algo:plmdgp}. By construction, we have $\bbE[S_{i,1} \eta_{i,j}^\top\mid \calF_{i,1}]=0$ and $\bbE[(S_{i,j} - \eta_{i,j})\eta_{i,j}^\top\mid \calF_{i,j}]=0$.
The latter holds because $S_{i,j} - \eta_{i,j}$ is a function of $(S_{i,1:j-1},T_{1:j-1})$ (denote this function by  $q$), and hence 
\[
\bbE[(S_{i,j} - \eta_{i,j})\eta_{i,j}^\top\mid \calF_{i,j}]=
\bbE[q(S_{i,1:j-1},T_{1:j-1})\eta_{i,j}^\top\mid \calF_{i,j}]=
q(S_{i,1:j-1},T_{1:j-1})\bbE[\eta_{i,j}^\top\mid \calF_{i,j}]=
0.\]
Then we have by a martingale Hoeffding-Azuma, w.p.~$1-\delta$,
\begin{align*}
  \bn{\frac{1}{n}\sum_i S_{i,1}\eta_{i,j}^\top}_{\infty}\lesssim\sqrt{\frac{\log(d_s/\delta)}{n}}, \quad   \bn{\frac{1}{n}\sum_i (S_{i,j}-\eta_{i,j})\eta_{i,j}^\top}_{\infty} \lesssim \sqrt{\frac{\log(d_s/\delta)}{n}}.
\end{align*}
Further, by a Hoeffding-Azuma for i.i.d, w.p.~$1-\delta$,
\begin{align*}
    \left\|\frac{1}{n} \sum_i (\eta_{i,j} \eta_{i,j}^\top - \bbE[\eta_{i,j} \eta_{i,j}^\top])\right\|_{\infty} \lesssim \sqrt{\frac{\log(d_s/\delta)}{n}}.
\end{align*}
From this we derive:
\begin{align}
    \nu^\top\Sigma_n\nu \geq~& \frac{1}{n} \sum_{i=1}^n (\nu_j^\top \eta_{i,j})^2 -\|\nu_1\|_1\|\nu_j\|_1\bn{\frac{1}{n}\sum_i S_{i,1}\eta_{i,j}^\top}_{\infty} - \|\nu_j\|_1^2\bn{\frac{1}{n}\sum_i (S_{i,j}-\eta_{i,j})\eta_{i,j}^\top}_{\infty} \nonumber\\
    \gtrsim~& \nu_j'\frac{1}{n} \sum_{i=1}^n \bbE[\eta_{i,j} \eta_{i,j}^\top] \nu_j - (\|\nu_1\|_1 \|\nu_j\|_1 + 2\|\nu_j\|_1^2)\sqrt{\frac{\log(d_s/\delta)}{n}} \nonumber\\
    \gtrsim~& c_s \|\nu_j\|_2^2 - (\|\nu_1\|_1^2 + \|\nu_j\|_1^2)\sqrt{\frac{\log(d_s/\delta)}{n}}\label{eq:sigma_n_helper}\\
    \geq~& c_s\|\nu_j\|_2^2 - \|\nu\|_1^2 \sqrt{\frac{\log(d_s/\delta)}{n}} \gtrsim~ c_s \|\nu_j\|_2^2 - s\|\nu\|_2^2\sqrt{\frac{\log(d_s/\delta)}{n}}.\label{eq:sigma_n_2}
\end{align}
Combining \eqref{eq:sigma_n_1} \& \eqref{eq:sigma_n_2}, we have that:
\begin{align*}
    \|\nu\|_2^2 =  \|\nu_1\|_2^2 +  \|\nu_j\|_2^2 \leq~& \frac{4}{c_s} \nu^\top\Sigma_n\nu + \left(1+\frac{8C_s}{c_s}\right) \|\nu_j\|_2^2\\
    \lesssim~& \left(\frac{1}{c_s} + \frac{C_s}{c_s^2}\right) \nu^\top\Sigma_n\nu  +\frac{s}{c_s}\bp{1+\frac{C_s}{c_s}} \sqrt{\frac{\log(d_s/\delta)}{n}} \|\nu\|_2^2.
\end{align*}
For $n$ larger than some constant that depends on $s, c_s, C_s$, we have:
\begin{align*}
    \|\nu\|_2^2 \lesssim~& \left(\frac{1}{c_s} + \frac{C_s}{c_s^2}\right) \nu^\top\Sigma_n\nu  + \frac{1}{2}\|\nu\|_2^2 \implies \|\nu\|_2^2 \lesssim \left(\frac{1}{c_s} + \frac{C_s}{c_s^2}\right) \nu^\top\Sigma_n\nu.
\end{align*}
Thus we recover the strong convexity statement we were after in \eqref{eq:strong_convexity_sigma_n}, with:
\begin{align}
\label{eq:strong_convexity_gamma_g}
    \frac{1}{\gamma} \sim \left(\frac{1}{c_s} + \frac{C_s}{c_s^2}\right) \lesssim \frac{C_s}{c_s^2}.
\end{align}

Next we characterize the magnitude of $\lambda$ that satisfies \eqref{eq:lasso_regularizer}.  We have $ \|\hat{R}_{i,j}-R_{i,j}\| =O(\|\hat{\theta}-\theta^*\|_2)$. Assuming states are universally bounded, we have:
\begin{align*}
    &\left\| \frac{1}{n} \sum_{i=1}^n (\hat{R}_i -S_{i,j}^\top\beta_j - S_{i,1}^\top \kappa_j)\, (S_{i,1};S_{i,j}) \right\|_\infty \\
    \leq~& \left\| \frac{1}{n} \sum_{i=1}^n (R_{i,j} -S_{i,j}^\top\beta_j - S_{i,1}^\top \kappa_j)\, (S_{i,1};S_{i,j})  \right\|_\infty  + O(\|\hat{\theta}-\theta^*\|_2) = \left\| \frac{1}{n} \sum_{i=1}^n \epsilon_{i,j}\,(S_{i,1};S_{i,j})  \right\|_\infty  + O(\|\hat{\theta}-\theta^*\|_2)\\
    \leq~& \left\| \frac{1}{n} \sum_{i=1}^n \epsilon_{i,j}\, S_{i,1} \right\|_\infty + \left\| \frac{1}{n} \sum_{i=1}^n \epsilon_{i,j}\, S_{i,j} \right\|_\infty  + O(\|\hat{\theta}-\theta^*\|_2),
\end{align*}
where $\epsilon_{i,j}= \sum_{j'=j}^{L-1}\eta_{i,j'}^\top \beta_{j'+1} + \epsilon_i$ as defined in \eqref{eq:final_unroll_plm}.
Since the noises $\epsilon_{i,j}$ are bounded and satisfy that $\bbE[\epsilon_{i,j}S_{i,1}\mid \calF_{i,1} ] = \bbE[\epsilon_{i,j} S_{i,j}\mid \calF_{i,j}]=0$, by a Hoeffding-Azuma for martingales, w.p.~$1-\delta$:
\begin{align*}
    \left\| \frac{1}{n} \sum_{i=1}^n \epsilon_{i,j}\, S_{i,1} \right\|_\infty =\left\| \frac{1}{n} \sum_{i=1}^n \left(\epsilon_{i,j}\, S_{i,1} - \bbE\left[\epsilon_{i,j}\, S_{i,1}\mid \calF_{i,1}\right]\right)\right\|_{\infty} \lesssim \sqrt{\frac{\log(2d_s/\delta)}{n}},\\
    \left\| \frac{1}{n} \sum_{i=1}^n \epsilon_{i,j}\, S_{i,j} \right\|_\infty=\left\| \frac{1}{n} \sum_{i=1}^n \left(\epsilon_{i,j}\, S_{i,j} - \bbE\left[\epsilon_{i,j}\, S_{i,j}\mid \calF_{i,j}\right]\right)\right\|_{\infty} \lesssim \sqrt{\frac{\log(2d_s/\delta)}{n}}.
\end{align*}
Thus we get that:
\begin{align*}
   \bn{\frac{1}{n}\sum_{i=1}^n (\hat{R}_{i,j} -S_{i,j}^\top\beta_j - S_{i,1}^\top \kappa_j)}_{\infty}\lesssim \sqrt{\frac{\log(d_s/\delta)}{n}} +O(\|\hat{\theta}-\theta^*\|_2), \quad \mbox{w.p.}\quad 1-\delta,
\end{align*}
and it suffices to take:
\begin{align}
\label{eq:lambda_set_g}
    \lambda = \lambda(\delta) \sim \left(\sqrt{\frac{\log(d_s/\delta)}{n}} + O(\|\hat{\theta}-\theta^*\|_2)\right)
\end{align}
 to satisfy \eqref{eq:lasso_regularizer} w.p.  $1-\delta$. Combining \eqref{eq:strong_convexity_gamma_g} \& \eqref{eq:lambda_set_g}, we are ready to employ Theorem \ref{thm:lasso_general}: w.p.~$1-\delta$,
 \begin{equation}
 \label{eq:high_prob_beta_bound}
      \|(\hat{\beta}_j;\hat{\kappa}_j)-(\beta_j;\kappa_j)\|_1 \leq \frac{sC_s}{c_s^2}\left(\sqrt{\frac{\log(d_s/\delta)}{n}} + O(\|\hat{\theta}-\theta^*\|_2)\right).
 \end{equation}
 Now we use the Lemma A.4 from \cite{shalev2014understanding}, which is provided below for convenience:
 \begin{lemma}[Lemma A.4, \cite{shalev2014understanding}]
 \label{lemma:from_probability_convergence_to_l1}
     Let $X $ be a random variable and $x'\in \bbR$ be a scalar and assume that there exists $a>0$ and $b\geq e$ such that for all $t\geq 0$ we have $\bbP\bp{|X-x'|>t}\leq 2b e^{-t^2/a^2}$. Then, $\bbE[|X-x'|]\leq a(2+\sqrt{\log(b)})$.
 \end{lemma}
Define 
\begin{equation}
\label{eq:z_definition_1}
    Z:=\bp{\frac{c_s^2}{sC_s}\|(\hat{\beta}_j;\hat{\kappa}_j)-(\beta_j;\kappa_j)\|_1-O(\|\hat{\theta}-\theta^*\|_2)}_+=\max\bp{\frac{c_s^2}{sC_s}\|(\hat{\beta}_j;\hat{\kappa}_j)-(\beta_j;\kappa_j)\|_1-O(\|\hat{\theta}-\theta^*\|_2), 0}.
\end{equation}
By \eqref{eq:high_prob_beta_bound}, we have:
\[
\bbP\bp{Z>\sqrt{\frac{\log(d_s/\delta)}{n}}}\leq \delta\quad\stackrel{t=\sqrt{\frac{\log(d_s/\delta)}{n}}}{\Longrightarrow}\quad \bbP\bp{Z>t} \leq d_se^{-nt^2}.
\]
Collectively, we have:
\begin{align*}
    \bbE\bb{
    \frac{c_s^2}{sC_s}\|(\hat{\beta}_j;\hat{\kappa}_j)-(\beta_j;\kappa_j)\|_1-O(\|\hat{\theta}-\theta^*\|_2)
    }&\leq \bbE[Z]\tag{by definition of $Z$ in \eqref{eq:z_definition_1}}\\
    &\leq n^{-1/2}\bp{2+\sqrt{\log\frac{d_s}{2}}}\tag{Lemma  \ref{lemma:from_probability_convergence_to_l1}}.
\end{align*}
We therefore have:
 \[
 \bbE[\|(\hat{\beta}_j;\hat{\kappa}_j)-(\beta_j;\kappa_j)\|_1] = O\bp{\frac{sC_s}{c_s^2}\bp{\sqrt{\frac{\log(d_s/2)}{n}} + \bbE\|\hat{\theta}-\theta^*\|_2}}.
 \]
Therefore with bounded state space,
\[
\|\hat{g}-g_j\|_{1,\infty} \leq \bbE[\|(\hat{\beta}_j;\hat{\kappa}_j)-(\beta_j;\kappa_j)\|_1 \|(s_1;s_j)\|_{\infty}] = O\bp{ \frac{sC_s}{c_s^2} \bp{\sqrt{\frac{\log(d_sn)}{n}}+O(\bbE\|\hat{\theta}-\theta^*\|_2) }}.
\]
Above we remind that $d_s$ is the state dimension, $C_s$ is such that $\bbE[S_{i,j}S_{i,j}^\top|\calF_{i,j}]\preceq C_s\cdot I$; $c_s$ is such that $\bbE[\eta_{i,j-1}\eta_{i,j-1}^\top]\succeq c_s I$ and $\bbE[S_{i,1}S_{i,1}^\top]\succeq c_s I$; $s$ is the  cardinality of nonzero entries in $(\beta_j;\kappa_j)$.

\subsubsection{Estimation guarantee for $h_j$.} Fix a stage index $j$. Under Model \ref{algo:plmdgp}, for each $k\in\Omega$, where recall that $\Omega$ represents the state coordinates involved in the operation $\chi_j$, we have  
\[ 
h_{j,k}(S_{i,1},S_{i,j-1}, T_{i,j-1}) = A_{j,k}^\top \Phi_{i,j-1}+B_{j,k}^\top S_{i,j-1} + M_{j,k}^\top S_{i,1},\quad \mbox{with}\quad \Phi_{i,j-1}:=\chi_{j-1}(S_{i,j-1})\otimes T_{i,j-1},
\] 
for sparse vectors $(A_{j,k}, B_{j,k}, M_{j,k})$. 
Denote the Lasso regularizer coefficient as $\lambda_k$.
Under the notation in Theorem \ref{thm:lasso_general},  the explainable variables are $x_i=(S_{i,1}, S_{i,j-1}, \Phi_{i,j-1})$ and the dependent variable $z_i=S_{i,j,k}$.

We first verify condition \eqref{eq:strong_convexity_sigma_n}.
Let $\Sigma_1=\frac{1}{n}\sum_{i=1}^n \bbE_{i,1}[S_{i,1} S_{i,1}^\top]$ 
and $\Sigma_{j,S} = \frac{1}{n}\sum_{i=1}^n \bbE_{i,j-1}[S_{i,j-1} S_{i,j-1}^\top]$ and $\Sigma_{j,T} = \frac{1}{n}\sum_{i=1}^n \bbE_{i,j-1}[ \Phi_{i,j-1}\Phi_{i,j-1}^\top ]$. Let $\Sigma_{1,n}=\frac{1}{n}\sum_{i=1}^n S_{i,1} S_{i,1}^\top $ and $\Sigma_{j,S,n} = \frac{1}{n}\sum_{i=1}^n S_{i,j-1} S_{i,j-1}^\top $ and $\Sigma_{j,T,n} = \frac{1}{n}\sum_{i=1}^n \Phi_{i,j-1}\Phi_{i,j-1}^\top $. 
 Similarly, let $\Sigma_n=\frac{1}{n}\sum_{i=1}^n x_ix_i^\top$.  By a  Hoeffding-Azuma inequality, with bounded covariates, we can also derive that, w.p.~$1-\delta$,
\begin{align*}
    &\|\Sigma_{1,n} - \Sigma_1\|_{\infty} \lesssim \sqrt{\frac{\log(d_s/\delta)}{n}},  \quad \|\Sigma_{j,S,n} - \Sigma_{j,S}\|_{\infty} \lesssim \sqrt{\frac{\log(d_s/\delta)}{n}}, \\
    &\|\Sigma_{j,T,n} - \Sigma_{j,T}\|_{\infty} \lesssim \sqrt{\frac{\log(d_\chi d_T/\delta)}{n}} \lesssim  \sqrt{\frac{\log( d_s/\delta)}{n}},
\end{align*}
where $d_\chi$ and $d_T$ are dimensions for $\chi(\cdot)$ and $T$ respectively. 


For any $\nu=(\nu_1;\nu_{j,S}; \nu_{j,T})$ with the restricted strong convexity property, which yields  $\|\nu\|_1\leq 4\sqrt{s}\|\nu\|_2$ by \eqref{eq:restricted_strong_convexity_implication}, where $\nu_1,\nu_{j,S}, \nu_{j,T}$ have the same dimensions as $S_{i,1}, S_{i,j-1}, \Phi_{i,j-1}$ respectively, 
first note that:
\begin{align}
    \nu^\top\Sigma_n\nu =~& \frac{1}{n}\sum_{i} (\nu^\top x_i)^2 \geq~ \frac{1}{2} \frac{1}{n}\sum_{i} (\nu_1^\top S_{i,1})^2 - \frac{1}{n}\sum_{i} (\nu_{j,S}^\top S_{i,j-1}+ \nu_{j,T}^\top\Phi_{i,j-1})^2 \nonumber \\
    \geq~& \frac{1}{2} \frac{1}{n}\sum_{i} (\nu_1^\top S_{i,1})^2 - \frac{2}{n}\sum_{i} (\nu_{j,S}^\top S_{i,j-1})^2 -\frac{2}{n}\sum_{i} (\nu_{j,T}^\top\Phi_{i,j-1})^2 \nonumber\\
    =~& \frac{1}{2}\nu_1^\top\Sigma_{1,n}\nu_1 -2\nu_{j,S}^\top \Sigma_{j,S,n}\nu_{j,S} -2\nu_{j,T}^\top \Sigma_{j,T,n}\nu_{j,T} \nonumber\\
    \gtrsim~&\frac{1}{2}\nu_1^\top\Sigma_{1}\nu_1 -2\nu_{j,S}^\top \Sigma_{j,S}\nu_{j,S} -2\nu_{j,T}^\top \Sigma_{j,T}\nu_{j,T} -\sqrt{\frac{\log(d_s)/\delta)}{n}}(\|\nu_1\|_1^2+\|\nu_{j,T}\|_1^2+\|\nu_{j,S}\|_1^2)\nonumber\\
   \geq ~& \frac{1}{2} c_s \|\nu_1\|_2^2 - 2C_s \|\nu_{j,S}\|_2^2 - 2C_s \|\nu_{j,T}\|_2^2 -\sqrt{\frac{\log(d_s)/\delta)}{n}}(\|\nu_1\|_1^2+\|\nu_{j,T}\|_1^2+\|\nu_{j,S}\|_1^2). \label{eq:sigma_n_h_1}
\end{align}
where we use that $\Sigma_1 \succeq c_sI$ and $\Sigma_{j,S}\preceq C_s I$ and  $\Sigma_{j,T}\preceq C_s I$.
On the other hand, we also have
\begin{align}
     \nu^\top\Sigma_n\nu =& \frac{1}{n}\sum_{i} (\nu^\top x_i)^2 \geq~ \frac{1}{2} \frac{1}{n}\sum_{i} (\nu_1^\top S_{i,1}+\nu_{j,S}^\top S_{i,j-1})^2 - \frac{1}{n}\sum_{i} (\nu_{j,T}^\top\Phi_{i,j-1})^2 \nonumber\\
     \gtrsim~&\frac{1}{2}c_s\|\nu_{j,S}\|^2_2 - (\|\nu_1\|_1^2 + \|\nu_{j,S}\|_1^2)\sqrt{\frac{\log(d_s/\delta)}{n}}- \frac{1}{n}\sum_{i} (\nu_{j,T}^\top\Phi_{i,j-1})^2 \tag{following \eqref{eq:sigma_n_helper}}\nonumber\\
     =~&\frac{1}{2}c_s\|\nu_{j,S}\|^2_2 - (\|\nu_1\|_1^2 + \|\nu_{j,S}\|_1^2)\sqrt{\frac{\log(d_s/\delta)}{n}}- \nu_{j,T}^\top \Sigma_{j,T,n}\nu_{j,T}^\top \nonumber\\
     \gtrsim~& \frac{1}{2}c_s\|\nu_{j,S}\|^2_2 - (\|\nu_1\|_1^2 + \|\nu_{j,S}\|_1^2)\sqrt{\frac{\log(d_s/\delta)}{n}}- \nu_{j,T}^\top \Sigma_{j,T}\nu_{j,T}^\top -\sqrt{\frac{\log(d_s)/\delta)}{n}}\|\nu_{j,T}\|_1^2\nonumber\\
     \geq~& \frac{1}{2}c_s\|\nu_{j,S}\|^2_2 - C_s\|\nu_{j,T}\|_2^2-\sqrt{\frac{\log(d_s)/\delta)}{n}}(\|\nu_1\|_1^2 + \|\nu_{j,S}\|_1^2+ \|\nu_{j,T}\|_1^2)\label{eq:sigma_n_h_2}
\end{align}
Now we lower bound $\nu^\top\Sigma_n\nu$ by a function of $\|\nu_{j,T}\|_2$ with the goal of achieving \eqref{eq:strong_convexity_sigma_n}.
Define $\bar{\Phi}_{i,j-1}:=\bbE_{i,j-1}^+[\Phi_{i,j-1}]$. We have
\begin{align*}
    \nu^\top\Sigma_n\nu &= \frac{1}{n}\sum_{i} (\nu_1^\top S_{i,1}+\nu_{j,S}^\top S_{i,j-1}+ \nu_{j,T}^\top\Phi_{i,j-1})^2 =  \frac{1}{n}\sum_{i} (\nu_1^\top S_{i,1}+\nu_{j,S}^\top S_{i,j-1}+ \nu_{j,T}^\top \bar{\Phi}_{i,j-1} +\nu_{j,T}^\top (\Phi_{i,j-1}-\bar{\Phi}_{i,j-1})  )^2\\
    &\geq \frac{1}{n} \sum_{i=1}^n (\nu_{j,T}^\top (\Phi_{i,j-1}-\bar{\Phi}_{i,j-1}) )^2 +   \frac{2}{n}\sum_{i}(\nu_1^\top S_{i,1}+\nu_{j,S}^\top S_{i,j-1}+ \nu_{j,T}^\top \bar{\Phi}_{i,j-1})\nu_{j,T}^\top (\Phi_{i,j-1}-\bar{\Phi}_{i,j-1})\\
    &\geq \frac{1}{n} \sum_{i=1}^n (\nu_{j,T}^\top (\Phi_{i,j-1}-\bar{\Phi}_{i,j-1}) )^2 -\|\nu_1\|_1\|\nu_{j,T}\|_1\bn{\frac{1}{n}\sum_i S_{i,1}(\Phi_{i,j-1}-\bar{\Phi}_{i,j-1})^\top}_{\infty}
   \\
   &\quad\quad - \|\nu_{j,S}\|_1\|\nu_{j,T}\|_1\bn{\frac{1}{n}\sum_i S_{i,j-1}(\Phi_{i,j-1} -\bar{\Phi}_{i,j-1})^\top}_{\infty}
    - \|\nu_{j,T}\|_1^2\bn{\frac{1}{n}\sum_i \bar{\Phi}_{i,j-1}(\Phi_{i,j-1}-\bar{\Phi}_{i,j-1})^\top}_{\infty} .
 \end{align*}
 Note that we have
 \[
 \bbE_{i,j-1}^+[S_{i,1}(\Phi_{i,j-1}-\bar{\Phi}_{i,j-1})^\top]=0, \quad \bbE_{i,j-1}^+[S_{i,j-1}(\Phi_{i,j-1} -\bar{\Phi}_{i,j-1})^\top]=0,\quad \bbE_{i,j-1}^+[\bar{\Phi}_{i,j-1}(\Phi_{i,j-1}-\bar{\Phi}_{i,j-1})^\top]=0.
 \]
 Then we have by a martingale Hoeffding-Azuma, w.p.~$1-\delta$,
 \begin{align*}
& \bn{\frac{1}{n}\sum_i S_{i,1}(\Phi_{i,j-1}-\bar{\Phi}_{i,j-1})^\top}_{\infty}\lesssim \sqrt{\frac{\log(d_s/\delta)}{n}},\quad
\bn{\frac{1}{n}\sum_i S_{i,j-1}(\Phi_{i,j-1} -\bar{\Phi}_{i,j-1})^\top}_{\infty}\lesssim \sqrt{\frac{\log(d_s/\delta)}{n}},\\
&
     \bn{\frac{1}{n}\sum_i \bar{\Phi}_{i,j-1}(\Phi_{i,j-1}-\bar{\Phi}_{i,j-1})^\top}_{\infty}\lesssim \sqrt{\frac{\log(d_s/\delta)}{n}},\\
    & \bn{ \frac{1}{n} \sum_{i=1}^n 
     (\Phi_{i,j-1}-\bar{\Phi}_{i,j-1})(\Phi_{i,j-1}-\bar{\Phi}_{i,j-1})^\top -\bbE_{i,j-1}[(\Phi_{i,j-1}-\bar{\Phi}_{i,j-1})(\Phi_{i,j-1}-\bar{\Phi}_{i,j-1})^\top]
     }_\infty\lesssim \sqrt{\frac{\log(d_s/\delta)}{n}}
 \end{align*}
 From this we derive:
 \begin{align}
      \nu^\top\Sigma_n\nu &\gtrsim \nu_{j,T} \bbE_{i,j-1}[(\Phi_{i,j-1}-\bar{\Phi}_{i,j-1})(\Phi_{i,j-1}-\bar{\Phi}_{i,j-1})^\top]\nu_{j,T}\nonumber\\
      &\quad\quad-(\|\nu_{j,T}\|_1^2+\|\nu_1\|_1\|\nu_{j,T}\|_1+ \|\nu_{j,S}\|_1\|\nu_{j,T}\|_1+\|\nu_{j,T}\|_1^2)\sqrt{\frac{\log(d_s/\delta)}{n}}\nonumber\\
      &\gtrsim i^{-\alpha}\|\nu_{j,T}\|_2^2 -(\|\nu_1\|_1^2+\|\nu_{j,S}\|_1^2+\|\nu_{j,T}\|_1^2)\sqrt{\frac{\log(d_s/\delta)}{n}}, \label{eq:sigma_n_h_3}
 \end{align}
 where we use Assumption \ref{assump:overlap} and AM-GM inequality.

Combining \eqref{eq:sigma_n_h_1}, \eqref{eq:sigma_n_h_2}, \eqref{eq:sigma_n_h_3}, we have
\begin{align*}
    &\|\nu\|_2^2=\|\nu_1\|_2^2+\|\nu_{j,S}\|_2^2+\|\nu_{j,T}\|_2^2\\
    \lesssim& \bp{i^\alpha+\frac{1}{c_s}} \nu^\top\Sigma_n\nu+ \frac{C_s}{c_s}(\|\nu_{j,S}\|_2^2+\|\nu_{j,T}\|_2^2)  + \bp{i^\alpha+\frac{1}{c_s}}(\|\nu_1\|_1^2+\|\nu_{j,S}\|_1^2+\|\nu_{j,T}\|_1^2)\sqrt{\frac{\log(d_s/\delta)}{n}}\\
    \lesssim& \bp{i^\alpha+\frac{1}{c_s}} \nu^\top\Sigma_n\nu+ \frac{C_s}{c_s}\|\nu_{j,T}\|_2^2+\frac{C_s}{c_s}\bp{\frac{1}{c_s}\nu^\top\Sigma_n\nu+ \frac{C_s}{c_s}\|\nu_{j,T}\|_2^2 + \frac{1}{c_s}(\|\nu_1\|_1^2+\|\nu_{j,S}\|_1^2+\|\nu_{j,T}\|_1^2)\sqrt{\frac{\log(d_s/\delta)}{n}}} \\
      &\quad\quad + \bp{i^\alpha+\frac{1}{c_s}}(\|\nu_1\|_1^2+\|\nu_{j,S}\|_1^2+\|\nu_{j,T}\|_1^2)\sqrt{\frac{\log(d_s/\delta)}{n}}\\
    \lesssim & \bp{i^\alpha+\frac{1}{c_s}+\frac{C_s}{c_s^2}} \nu^\top\Sigma_n\nu + \bp{\frac{C_s}{c_s}+\frac{C_s^2}{c_s^2}}\|\nu_{j,T}\|_2^2 +\bp{i^\alpha+\frac{1}{c_s}+\frac{C_s}{c_s^2}}(\|\nu_1\|_1^2+\|\nu_{j,S}\|_1^2+\|\nu_{j,T}\|_1^2)\sqrt{\frac{\log(d_s/\delta)}{n}}\\
    \lesssim & \bp{i^\alpha+\frac{1}{c_s}+\frac{C_s}{c_s^2}+\frac{C_s^2}{c_s^2}i^\alpha} \nu^\top\Sigma_n\nu
    +\bp{i^\alpha+\frac{1}{c_s}+\frac{C_s}{c_s^2}+\frac{C_s^2}{c_s^2}i^\alpha}(\|\nu_1\|_1^2+\|\nu_{j,S}\|_1^2+\|\nu_{j,T}\|_1^2)\sqrt{\frac{\log(d_s/\delta)}{n}}\\
    \lesssim& \frac{C_s^2}{c_s^2}i^\alpha\nu^\top\Sigma_n\nu + \frac{C_s^2}{c_s^2}i^\alpha(\|\nu_1\|_1^2+\|\nu_{j,S}\|_1^2+\|\nu_{j,T}\|_1^2)\sqrt{\frac{\log(d_s/\delta)}{n}}\\
    \leq& \frac{C_s^2}{c_s^2}i^\alpha\nu^\top\Sigma_n\nu + \frac{C_s^2}{c_s^2}i^\alpha(\|\nu_1\|_1+\|\nu_{j,S}\|_1+\|\nu_{j,T}\|_1)^2\sqrt{\frac{\log(d_s/\delta)}{n}}\\
     =& \frac{C_s^2}{c_s^2}i^\alpha\nu^\top\Sigma_n\nu + \frac{C_s^2}{c_s^2}i^\alpha\|\nu\|_1^2\sqrt{\frac{\log(d_s/\delta)}{n}}\\
          \leq & \frac{C_s^2}{c_s^2}n^\alpha\nu^\top\Sigma_n\nu + 16 s\frac{C_s^2}{c_s^2}\|\nu\|_2^2\sqrt{\frac{\log(d_s/\delta)}{n^{1-2\alpha}}}.\\
\end{align*}
For $n$ larger than some constant that depends on $s, c_s, C_s,\alpha$, we have:
\[
\|\nu\|_2^2 \lesssim \frac{C_s^2}{c_s^2}n^\alpha\nu^\top\Sigma_n\nu +\frac{1}{2}\|\nu\|_2^2,
\]
which leads to recovery of strong convexity statement we were after in \eqref{eq:strong_convexity_sigma_n}, with:
\begin{align}
\label{eq:strong_convexity_gamma_h}
    \frac{1}{\gamma} \sim \frac{C_s^2}{c_s^2}n^\alpha.
\end{align}

Next for each state coordinate $k$, we characterize the magnitude of $\lambda_k$ that satisfies \eqref{eq:lasso_regularizer}. 
Assuming states are universally bounded, we have:
\begin{align*}
    &\bn{\frac{1}{n}\sum_{i=1}^n (S_{i,j,k} -\gamma_{j,k}^\top S_{i,1} - B_{j-1,k}^\top S_{i,j-1} - A_{j-1,k}^\top \Phi_{i,j-1}))(S_{i,1}, S_{i,j-1}, \Phi_{i,j-1})}_{\infty} \\
= & \bn{\frac{1}{n}\sum_{i=1}^n \eta_{i,j,k}(S_{i,1}, S_{i,j-1}, \Phi_{i,j-1})}_{\infty}\\
\leq  &  \bn{\frac{1}{n}\sum_{i=1}^n \eta_{i,j,k}S_{i,1}}_{\infty} +  \bn{\frac{1}{n}\sum_{i=1}^n\eta_{i,j,k}S_{i,j-1}}_{\infty} +  \bn{\frac{1}{n}\sum_{i=1}^n\eta_{i,j,k} \Phi_{i,j-1}}_{\infty}.
\end{align*}
where $\eta_{i,j}$ as defined in \eqref{eq:final_unroll_plm} are the state transition noises.
Since the noises $\eta_{i,j}$ are bounded and satisfy that $\bbE_{i,1}[\eta_{i,j,k}S_{i,1} ] = \bbE_{i,j-1}[\eta_{i,j,k} S_{i,j-1}]= \bbE_{i,j-1}[\eta_{i,j,k} \Phi_{i,j-1}]=0$, by a Hoeffding-Azuma for martingales, w.p.~$1-\delta$:
\begin{align*}
    \bn{\frac{1}{n}\sum_{i=1}^n \eta_{i,j,k}S_{i,1}}_{\infty}  \lesssim \sqrt{\frac{\log(d_s/\delta)}{n}},  \quad \bn{\frac{1}{n}\sum_{i=1}^n\eta_{i,j,k}S_{i,j-1}}_{\infty}\lesssim\sqrt{\frac{\log(d_s/\delta)}{n}}, \quad 
 \bn{\frac{1}{n}\sum_{i=1}^n\eta_{i,j,k} \Phi_{i,j-1}}_{\infty} \lesssim \sqrt{\frac{\log(d_s/\delta)}{n}}.
\end{align*}
Thus we get that:
\begin{align*}
   \bn{\frac{1}{n}\sum_{i=1}^n (S_{i,j,k} -\gamma_{j,k}^\top S_{i,1} - B_{j-1,k}^\top S_{i,j-1} - A_{j-1,k}^\top \Phi_{i,j-1}))(S_{i,1}, S_{i,j-1}, \Phi_{i,j-1})}_{\infty} \lesssim \sqrt{\frac{\log(d_s/\delta)}{n}} \quad \mbox{w.p.}\quad 1-\delta,
\end{align*}
and it suffices to take:
\begin{align}
\label{eq:lambda_set_h}
    \lambda = \lambda(\delta) \sim \left(\sqrt{\frac{\log(d_s/\delta)}{n}} \right)
\end{align}
 to satisfy \eqref{eq:lasso_regularizer} w.p.~$1-\delta$.
Combining \eqref{eq:strong_convexity_gamma_h} \& \eqref{eq:lambda_set_h}, by Theorem \ref{thm:lasso_general},  for each coordinate $k\in\Omega$, we have
\begin{equation}
\label{eq:high_prob_h_bound}
     \|(\hat{A}_{j,k};\hat{B}_{j,k};\hat{M}_{j,k})-(A_{j,k};B_{j,k}; M_{j,k})\|_1 \leq  \frac{sC_s^2}{c_s^2}\sqrt{\frac{\log(d_s/\delta)}{n^{1-2\alpha}}}.
\end{equation}

 Now we use the Lemma \ref{lemma:from_probability_convergence_to_l1} to show the $L^1$ convergence.
Define 
\begin{equation}
\label{eq:z_definition_2}
    Z:=\frac{c_s^2}{sC_s^2}\|(\hat{A}_{j,k};\hat{B}_{j,k};\hat{M}_{j,k})-(A_{j,k};B_{j,k}; M_{j,k})\|_1 .
\end{equation}
By \eqref{eq:high_prob_h_bound}, we have:
\[
\bbP\bp{Z>\sqrt{\frac{\log(d_s/\delta)}{n^{1-2\alpha}}}}\leq \delta\quad\stackrel{t=\sqrt{\frac{\log(d_s/\delta)}{n^{1-2\alpha}}}}{\Longrightarrow}\quad \bbP\bp{Z>t} \leq d_se^{-nt^2}.
\]
Applying Lemma \ref{lemma:from_probability_convergence_to_l1}, we have 
\begin{align*}
    \bbE[Z]\leq n^{-\frac{1-2\alpha}{2}}\bp{2+\sqrt{\log\frac{d_s}{2}}}.
\end{align*}
We therefore have:
 \[
\bbE[\|(\hat{A}_{j,k};\hat{B}_{j,k};\hat{M}_{j,k})-(A_{j,k};B_{j,k}; M_{j,k})\|_1 ] = O\bp{\frac{sC^2_s}{c_s^2}\bp{\sqrt{\frac{\log(d_s/2)}{n^{1-2\alpha}}} }}.
 \]

Therefore with bounded state space,
\[
  \|\hat{h}_{j,k} - h_{j,k}\|_{1,\infty} =O\bp{  \frac{sC_s^2}{c_s^2}\sqrt{\frac{\log(d_s/2)}{n^{1-2\alpha}}} },
\]
and thus collectively, we have
\begin{align*}
    \|\hat{h}_{j,\Omega} - h_{j,\Omega}\|_{1,\infty} = O\bp{  \frac{\sqrt{d_\Omega} sC_s^2}{c_s^2}\sqrt{\frac{\log(d_s/2)}{n^{1-2\alpha}}} }.
\end{align*}
Above we remind that $d_s$ is the state dimension, $d_\Omega$ is the dimension of sub-state that $\chi_j()$ operates on, $C_s$ is such that $\bbE[S_{i,j}S_{i,j}|\calF_{i,j}]\preceq C_s\cdot I$; $c_s$ is such that $\bbE[\eta_{i,j-1}\eta_{i,j-1}^\top]\succeq c_s I$ and $\bbE[S_{i,1}S_{i,1}^\top]\succeq c_s I$; $s$ is the  cardinality of nonzero columns in $(A_{j};B_j;M_j)$.

\subsubsection{Proof of Theorem \ref{thm:lasso_general}.}
\label{appendix:proof_general_lasso}
Let $\nu=\hat{\theta}-\theta_0$.
We first show that when $\lambda\geq \frac{2}{n}\sum_{i=1}^n \|(\hat{z}_i-\theta_0^\top x_i)x_i\|_{\infty}$, we have $\nu$ satisfies the restricted strong convexity property $\|\nu_{S^c}\|_1\leq 3\|\nu_S\|_1$.    
Since $\hat{\theta}$ optimizes \eqref{eq:lasso_loss}, we have
\begin{align*}
    \lambda(\|\theta_0\|_1-\|\hat{\theta}\|_1)&\geq \hat{L}_n(\hat{\theta})-\hat{L}(\theta_0) && (\mbox{optimality of }\hat{\theta})\\
    &\geq \br{\partial_{\theta}\hat{L}_n(\theta_0), \hat{\theta}-\theta_0} && (\mbox{convexity of }\hat{L}_n)\\
     &= \br{\partial_{\theta}\hat{L}_n(\theta_0)-\partial_{\theta} L(\theta_0), \hat{\theta}-\theta_0} && (\mbox{optimality of }\theta_0)\\
     &\geq -\bn{\partial_{\theta}\hat{L}_n(\theta_0)-\partial_{\theta} L(\theta_0)}_{\infty}\bn{\hat{\theta}-\theta_0}_1 &&(\mbox{Cauchy Schwartz inequality})\\
     &\geq -\frac{\lambda}{2}\bn{\hat{\theta}-\theta_0}_1 &&(\mbox{assumption on }\lambda).
\end{align*}
where we also used that $\partial_{\theta} L(\theta_0)=\bbE[(z-\theta_0^\top x)x]=0$ and that $\|  \partial_{\theta} \hat{L}_n(\theta_0) \|_\infty=\| \frac{1}{n} \sum_{i=1}^n (\hat{z}_i - \theta_0^\top x_i)\, x_i\|_\infty$. We thus have,
\begin{align*}
   \frac{1}{2} \|\nu\|_1\geq \|\theta_0+\nu\|_1-\|\theta_0\|_1.
\end{align*}
Let $S$ be the support of $\theta_0$ (and so $s=|S|$), thus  we have
\begin{align*}
    \frac{1}{2}\|\nu\|_1\geq \|\theta_{0,S}+\nu_S\|_1 + \|\nu_{S^c}\|_1-\|\theta_{0,S}\|_1\geq \|\nu_{S^c}\|_1-\|\nu_S\|_1,
\end{align*}
which leads to $\|\nu_{S^c}\|_1\leq 3\|\nu_S\|_1$. 
We now show that this property implies:
\begin{align}
\label{eq:error_bound}
    \nu^\top \Sigma_n \nu \leq~& 12\lambda \sqrt{s} \|\nu\|_2, & \|\nu\|_1 \leq~& 4\sqrt{s}\|\nu\|_2.
\end{align} 
We have 
\begin{align}
\label{eq:restricted_strong_convexity_implication}
    \|\nu\|_1=\|\nu_{S^c}\|_1+\|\nu_{S}\|_1\leq 4\|\nu_{S}\|_1\leq 4\sqrt{s}\|\nu_S\|_2\leq 4\sqrt{s}\|\nu\|_2,
\end{align}
which shows the RHS of \eqref{eq:error_bound}.
On the other hand, by strong convexity of $\hat{\ell}$, we also have
\begin{align*}
     \lambda(\|\theta_0\|_1-\|\hat{\theta}\|_1)&\geq \hat{L}_n(\hat{\theta})-\hat{L}_n(\theta_0) && (\mbox{optimality of }\hat{\theta})\\
    &\geq \br{\partial_{\theta}\hat{L}_n(\theta_0), \hat{\theta}-\theta_0} + \frac{1}{2}\nu^\top\Sigma_n\nu && (\mbox{strong convexity of }\hat{L}_n)\\
    &\geq -\frac{\lambda}{2}\|\nu\|_1 +  \frac{1}{2}\nu^\top\Sigma_n\nu  &&(\mbox{following the exact argument above}),
\end{align*}
where $\Sigma_n=\frac{1}{n}\sum_{i=1}^n x_ix_i^\top$.
Again note that
\begin{align*}
    \|\theta_0\|_1-\|\hat{\theta}\|_1 = \|\theta_0\|_1-\|\theta_0+\nu\|_1\leq \|\nu\|_1.
\end{align*}
So we have $\nu^\top\Sigma_n\nu \leq 12\lambda\sqrt{s}\|\nu\|_2$,
which shows the LHS of \eqref{eq:error_bound}. Then by the assumption on the $\Sigma_n$ restricted convexity,  
 we have  w.p.~$1-\delta$, $\nu^\top\Sigma_n\nu\geq \gamma \|\nu\|_2^2$
for some $\gamma$. Putting it all together, we have,
\begin{align}
\label{eq:error_bound_2}
    \gamma \|\nu\|_2^2\leq  \nu^\top\Sigma_n\nu \leq 12 \lambda \sqrt{s} \|\nu\|_2 \implies \|\nu\|_2 \leq \frac{12\lambda\sqrt{s}}{\gamma} \implies \|\nu\|_1 \leq 4\sqrt{s} \|\nu\|_2 \leq \frac{48\lambda s}{\gamma},
\end{align}
concluding the proof.

\subsection{Proof of Corollary \ref{cor:plmm_normality_rate}}

\label{appendix:proof_hdmm_gaussian}
We now show Corollary \ref{cor:plmm_normality_rate}. Lemma \ref{lemma:partial_linear_model} shows that Model \ref{algo:plmdgp} satisfies Assumptions \ref{assump:exogeneity}, \ref{assump:linear_blip_function} \& \ref{assump:homoscedasticity}.
Under Assumption \ref{assump:overlap}, Corollary \ref{cor:consistency} shows that the $\tilde{\theta}_{i,j}^{(C)}$ in Algorithm \ref{algo:plmdgp_weights} satisfies $\|\tilde{\theta}_{i,j}^{(C)} - \theta^*\|_2 = O\bp{(i-1)^{-\lambda_\theta}}$ for the estimation rate $\lambda_\theta=\frac{1-L\alpha}{2}$
Lemma \ref{lemma:estimation_rate_plmm} shows that the exponents of the estimation rates for $g_j$ and $h_j$ are $\gamma_G = \min(\frac{1}{2}, \lambda_\theta)$ and $\lambda_h=\frac{1-2\alpha}{2}$.
Invoking Lemma \ref{lemma:estimate_f}, we have that:
\[
\|\hat f_{i,j}-f_j\|_{1,\infty} = O(i^{-\min(\gamma_\theta,\gamma_G,1/2)}) = O(i^{-\frac{1-L\alpha}{2}}).
\]
Invoking Lemma \ref{lemma:estimate_v}, we have that:
\[
\|\hat \nu_{i,j}-\nu_j\|_{1,\infty} = O(i^{-\min(\gamma_h,1/2)}) = O(i^{-\frac{1-2\alpha}{2}}).
\]
Finally, invoking Corollary \ref{cor:full_normality}, we have the uniform Gaussian approximation rate to be $\tilde O\big(L^2 n^{-\min(\frac{1-\alpha}{12}, \frac{\gamma(1-L\alpha)}{10}, \frac{1-(L+1)\alpha}{2})}\big)  $.

\section{Auxiliary Lemmas}

For a matrix $A$, let $\Tr(A)$ be the sum of diagonal entries of $A$.

\begin{lemma}
     \label{lemma:var_inequality}
    Let $\alpha,\beta\in\bbR^d$ be some vector-valued random variables, with $\bbE[\alpha]=\zero$ and $\|\beta\|_2\leq c$ for some universal constant $c$. We have
    \[
\bbE[(\alpha\beta^\top -\bbE[\alpha\beta^\top ]) (\alpha\beta^\top -\bbE[\alpha\beta^\top ])^\top]\preceq c^2 \Var(\alpha).
    \]
\end{lemma}
\begin{proof}
  For any $x\in\bbR^d$,  we have
    \begin{align*}
       &  x^\top \bbE[(\alpha\beta^\top -\bbE[\alpha\beta^\top ]) (\alpha\beta^\top -\bbE[\alpha\beta^\top ])^\top] x = x^\top \bbE[\alpha\beta^\top\beta\alpha^\top] x - x^\top \bbE[\alpha\beta^\top]\bbE[\beta\alpha^\top] x \\
    \leq & x^\top \bbE[\alpha\beta^\top\beta\alpha^\top] x = \bbE[(x^\top \alpha)^2\|\beta\|_2^2]\leq c^2\bbE[(x^\top \alpha)^2] = c^2x^\top \bbE[\alpha\alpha^\top]x = c^2x^\top\Var(\alpha)x,
    \end{align*}
completing the proof.
\end{proof}

\begin{lemma}
\label{lemma:frob_inequality}
    Let $A_i,B_i\in\bbR^{d\times d}$ for $i=1,\dots,n$. We have
    \[
   \bn{\sum_{i=1}^n A_iB_i^\top }_{Frob}^2 \leq n\bp{\sum_{i=1}^n \|A_i\|_{Frob}^2\|B_i\|_{Frob}^2} .
    \]
\end{lemma}
\begin{proof}
    We have
   \begin{align*}
        \bn{\sum_{i=1}^n A_iB_i^\top }_{Frob}^2 & \stackrel{(i)}{\leq} \bp{
        \sum_{i=1}^n \|A_iB_i^\top \|_{Frob}
        }^2\stackrel{(ii)}{\leq}n \bp{
        \sum_{i=1}^n \|A_iB_i^\top \|_{Frob}^2
        } \\
        &\stackrel{(iii)}{\leq}n \bp{
        \sum_{i=1}^n \Tr\bp{A_iB_i^\top B_iA_i^\top }
        }  = n \bp{
        \sum_{i=1}^n \Tr\bp{A_i^\top A_iB_i^\top B_i}
        } \\
        &\stackrel{(iv)}{\leq}n \bp{
        \sum_{i=1}^n \Tr\bp{A_i^\top A_i}\Tr\bp{B_i^\top B_i}
        }= n\bp{\sum_{i=1}^n \|A_i\|_{Frob}^2\|B_i\|_{Frob}^2}.
   \end{align*} 
where (i) is by triangular inequality, (ii) is by Cauchy-Schwartz inequality,  (iii) is by the definition of Frobenius norm, (iv) is by Lemma \ref{lemma:trace_inequality}.
\end{proof}

\begin{lemma}
    \label{lemma:eigenval_frob}
    Let $A, B$ be two positive definite matrices of dimension $d\times d$. Then all eigenvalues of $A^{-1}B-I$ are real number, and  denote these eigenvalues as $\{\lambda_1,\dots, \lambda_d\}$.
    Let $(\lambda_{\max}(A),\lambda_{\min}(A))$ be the largest and smallest eigenvalues of $A$ respectively.
    We further have $\sum_{i=1}^d \lambda_i^2\leq \frac{\lambda_{\max}(A)}{\lambda_{\min}(A)}\|A^{-1}B-I\|_{Frob}^2$.
\end{lemma}
\begin{proof}
We have:
\begin{align*}
    A^{-1}B-I  = 
    A^{-\frac{1}{2}}(A^{-\frac{1}{2}}BA^{-\frac{1}{2}} - I)A^{\frac{1}{2}}.
\end{align*}
Therefore $A^{-1}B-I$ is similar to matrix $C:=A^{-\frac{1}{2}}BA^{-\frac{1}{2}} - I$, which is symmetric and  has full real eigenvalues. As a result, $A^{-1}B-I$ also has the same real eigenvalues. Denote them as $\{\lambda_1,\dots, \lambda_d\}$. 
On the other hand, since $A$ is positive definite,  there exist an orthogonal matrix $U$ and a diagonal matrix $V$ with diagonal entries being the eigenvalues of $A$, such that $A = UVU^\top$.
We have:
\begin{align*}
    \|A^{-1}B-I\|_{Frob}^2 
    & = \Tr\bp{( A^{-\frac{1}{2}}C A^{\frac{1}{2}})( A^{-\frac{1}{2}}CA^{\frac{1}{2}})^\top} = \Tr\bp{( A^{-\frac{1}{2}}CA^{\frac{1}{2}})( A^{\frac{1}{2}}CA^{-\frac{1}{2}})}\\
    & = \Tr\bp{A^{-1}CAC} =  \Tr\bp{UV^{-1}U^\top CUVU^\top C}= \Tr\bp{V^{-1}U^\top CUVU^\top C U}\\
    & \geq \lambda_{\max}(A)^{-1} \Tr\bp{U^\top CUVU^\top C U} = \lambda_{\max}(A)^{-1} \Tr\bp{VU^\top C^2U}\\
    &\geq \lambda_{\max}(A)^{-1}\lambda_{\min}(A) \Tr\bp{U^\top C^2U} = \frac{\lambda_{\min}(A)}{\lambda_{\max}(A)}\Tr(C^2) = \frac{\lambda_{\min}(A)}{\lambda_{\max}(A)}\sum_{i=1}^d \lambda_i^2.
\end{align*}
\end{proof}

\begin{lemma}
    \label{lemma:trace_inequality}
    Let $A,B$ be positive semi-definite matrices. We have
    \[
    \Tr(AB)\leq \Tr(A)\Tr(B).
    \]
\end{lemma}
\begin{proof}
    We can diagonalize $A,B$ as
    \[
    A= U\Lambda_AU^\top ,\quad B=V\Lambda_BV^\top ,
    \]
    where $U,V$ are orthogonal matrices, $\Lambda_A,\Lambda_B$ are diagonal matrices with nonnegative diagonal entries. We have
    \begin{align*}
        \Tr(AB) = & \Tr(U\Lambda_AU^\top V\Lambda_BV^\top ) = \Tr(\Lambda_AU^\top V\Lambda_BV^\top U)\\
        \leq &\Tr(\Lambda_A)\Tr(U^\top V\Lambda_BV^\top U) = \Tr(A)\Tr(B).
    \end{align*}
\end{proof}

\begin{lemma}
    \label{lemma:matrix_inequality}
    Let $\alpha,\beta\in\bbR^d$ be some vector-valued random variables, and $\|\beta\|_{2}\leq c$ for some universal constant $c$ almost surely, we have
    \[
    \bbE[\alpha\beta^\top ]\bbE[\beta\alpha^\top ] \preceq c^2\bbE[\alpha\alpha^\top ].
    \]
\end{lemma}
\begin{proof}
    It suffices to showing that for any $x\in\bbR^d$, there is
    \[
    x^\top  \bbE[\alpha\beta^\top ]\bbE[\beta\alpha^\top ]x\leq c^2x^\top \bbE[\alpha\alpha^\top ]x.
    \]
    We have 
    \begin{align*}
         x^\top  \bbE[\alpha\beta^\top ]\bbE[\beta\alpha^\top ]x =  \sum_{j=1}^d \bbE[\{\beta\alpha^\top x\}_j]^2\leq \sum_{j=1}^d \bbE[\{\beta\alpha^\top x\}_j^2] = \bbE\bb{
         \|\beta\alpha^\top x\|_2^2
         }\\
         \leq \bbE\bb{
         \|\beta\|_2^2\|\alpha^\top x\|_2^2
         }\leq c^2 \bbE\bb{
         \|\alpha^\top x\|_2^2
         } = c^2 x^\top  \bbE[\alpha\alpha^\top ]x.
    \end{align*}
\end{proof}

\begin{lemma}
\label{lemma:trace_inequality_2}
    Let $A,B$ be positive semi-definite matrices. Suppose $A\preceq B$, we have $\Tr(A)\leq \Tr(B)$.
\end{lemma}
\begin{proof}
    Since $A\preceq B$, we have $B-A$ is positive semi-definite and thus
    \[
    \Tr(B-A)\geq 0 \Rightarrow \Tr(A)\leq \Tr(B).
    \]
\end{proof}

\begin{lemma}
    \label{lemma:trace_inequality_3}
    Let $A,B$ be square  matrices with $B-A$ be positive semi-definite. It holds that $\Tr(A)\leq \Tr(B)$.
\end{lemma}
\begin{proof}
$B-A$ is positive semi-definite, and so $\Tr(B-A)\geq 0$, or equivalently, $\Tr(A)\leq \Tr(B)$.
\end{proof}

\begin{lemma}
\label{lemma:trace_inequality_4}
    Let $A\in\bbR^{m\times n}$ be a random  matrix. It holds that $\bbE[(A-\bbE[A])(A-\bbE[A])^\top ]\preceq \bbE[AA^\top ]$.
\end{lemma}

\begin{proof}
    It suffices to show that for any $x\in \bbR^m$: $x^\top  \bbE[(A-\bbE[A])(A-\bbE[A])^\top ]x\leq x^\top  \bbE[AA^\top ] x$.
    \begin{align*}
        x^\top  \bbE[(A-\bbE[A])(A-\bbE[A])^\top ]x 
        =~& x^\top \bbE[AA^\top ]x - x^\top  \bbE[A]\bbE[A]^\top  x = x^\top  \bbE[AA^\top ]x - (x^\top  \bbE[A])^2\\
        \geq~& x^\top  \bbE[AA^\top ] x
    \end{align*}
\end{proof}

\begin{lemma}
\label{lemma:eigenvalue_tri}
    Let $A$ be a block  upper triangular matrix. Then the eigenvalues of $A$ is the combination of eigenvalues of its  diagonal blocks. 
\end{lemma}
\begin{proof}
    Let $\{A_{1}, \dots, A_{d}\}$ be the diagonal blocks of $A$. Then $\lambda$ is an eigenvalue of $A$ if and only if $\det(A-\lambda I)=0$. Note that
    \begin{align*}
        \det(A-\lambda I)=\prod_{i=1}^d\det(A_i-\lambda I).
    \end{align*}
    Then $\lambda$ must be an eigenvalue of one diagonal block of $A$.
\end{proof}


\begin{lemma}\label{lem:lower-eigenvalue}
    Consider two matrices $B$ and $B_0$ and let $\delta=B-B_0$. Then for any vector $x$, with $\|x\|_2=1$, we have:
    \begin{align}
        x^\top B^\top B x \geq \frac{1}{2} x^\top B_0^\top B_0 x - 2\,\|\delta\|_{Frob}^2
    \end{align}
\end{lemma}
\begin{proof}
By applying the Cauchy-Scwarz and the AM-GM inequalities, we have:
\begin{align*}
    x^\top B^\top B x =~& x^\top (B_0-\delta)^\top (B_0-\delta) x=~ x^\top B_0^\top B_0 x - 2 x^\top \delta^\top B_0 x + x^\top \delta^\top \delta x\\
    \geq~& x^\top B_0^\top B_0 x - 2 x^\top \delta^\top B_0 x\\
    \geq~& x^\top B_0^\top B_0 x - 2 \|\delta x\|_2 \, \|B_0 x\|_2 \tag{Cauchy-Schwarz}\\
    \geq~& x^\top B_0^\top B_0 x - 2 \|\delta x\|_2^2 -\frac{1}{2} \|B_0 x\|_2^2 \tag{AM-GM inequality: $\alpha\beta = \bp{\frac{\alpha}{\sqrt{\eta}}}\bp{\sqrt{\eta}\beta}\leq \frac{1}{2\eta}\alpha^2+\frac{\eta}{2}\beta^2$ for any $\eta>0$}\\
    =~& \frac{1}{2} x^\top B_0^\top B_0 x - 2 \|\delta x\|_2^2 \\\geq~& \frac{1}{2} x^\top B_0^\top B_0 x - 2 \|\delta\|_{Frob}^2.  \tag{For a matrix $A$, $\|A\|_2\leq \|A\|_{Frob}$}
\end{align*}
\end{proof}

\begin{lemma}\label{lem:high_probability_invertible}
    Let $A_n, B_n$ be positive semi-definite random matrix. If $A_n\succeq c\cdot I$ a.s., and that $\bbE[\|A_n-B_n\|_2]=O(n^{-\alpha})$, then $P(B_n\succeq\frac{c}{2}\cdot I) = 1- O(n^{-\alpha})$.
\end{lemma}
\begin{proof}
    Let $x$ be any unit vector with $\|x\|_2=1$. Then 
    \begin{align*}
            \|B_nx\|_2 \geq & \|A_n x\| - \|(A_n-B_n)x\| && \text{(Triangular inequality)}\\
            \geq &c - \|A_n-B_n\|_2 &&(A_n\succeq c\cdot I, a.s. \, \& \, \mbox{Definition of }\|\cdot\|_2).
    \end{align*}
    Thus $\lambda_{\min}(B_n)\geq c - \|A_n-B_n\|_2$.
    \begin{align*}
        1 - P\bp{B_n\succeq \frac{c}{2}\cdot I} = P\bp{\lambda_{\min}(B_n)\leq \frac{c}{2}} \\
        \leq P\bp{ c - \|A_n-B_n\|_2 \leq\frac{c}{2}}
        =P\bp{  \|A_n-B_n\|_2 \geq\frac{c}{2}} \leq \frac{2}{c}\bbE[  \|A_n-B_n\|_2] = O(n^{-\alpha}).
    \end{align*}

\end{proof}

\begin{lemma}\label{lemma:lower_bound_v}
Let $x\in\bbR^{p}$, $y\in\bbR^{d}$ be random vectors. Let $a=(a_1;\dots; a_{p})$ with $a_1,\dots, a_p\in \bbR^d$ be any given vector. Suppose that $\Var(y|x)\succeq c_y\cdot I$ and $\bbE[xx^\top]\succeq c_x \cdot I$. We have
\[
\bbE[\Var(a^\top (x\otimes y)|x)] \succeq c_xc_y\|a\|_2^2.
\]
\end{lemma}
\begin{proof}
    Let matrix $A:=(a_1,\dots,a_p)\in\bbR^{d\times p}$. We have 
    \begin{align*}
      a^\top (x\otimes y) = \bp{\sum_{i=1}^p x_ia_i}^\top y  = (Ax)^\top y.
    \end{align*}
    Thus 
    \begin{align*}
        \Var(a^\top (x\otimes y)|x) &= \Var((Ax)^\top y\mid x) = \bbE\bb{((Ax)^\top (y-\bbE[y|x]))^2\mid x}\\
        &=\bbE\bb{(Ax)^\top (y-\bbE[y|x])(y-\bbE[y|x])^\top Ax\mid x}\\
        &=(Ax)^\top\Var(y\mid x) (Ax)\geq c_y x^\top A^\top A x.
    \end{align*}
    Therefore, 
    \begin{align*}
       \bbE[\Var(a^\top (x\otimes y)|x)] &\geq c_y\bbE[ x^\top A^\top A x] = c_y \bbE[\Tr(x^\top A^\top A x)] = c_y\bbE[\Tr(A^\top A xx^\top )]\\
       &=c_y\Tr( \bbE[A^\top A xx^\top ])  = c_y\Tr( A^\top A\bbE[ xx^\top ]) = c_y\Tr(  A\bbE[ xx^\top ]A^\top)\\
      & \geq c_y\Tr(  A c_x\cdot I A^\top) =  c_xc_y \Tr(A^\top A)=c_xc_y \|a\|_2^2.
    \end{align*}
\end{proof}

\begin{lemma}\label{lemma:col_rank_matrix}
Let matrix $A\in\mathbf{R}^{m\times n}$ with $m\geq n$ have  full column rank. Let matrix $B\in\mathbf{R}^{m\times m}$ be a positive definite matrix. Then $A^\top B A$ is positive definite. 
\end{lemma}
\begin{proof}
We prove the result by contradiction. By construction, $A^\top B A$ is positive semi-definite with $B$ being positive definite. Suppose, for the sake of contradiction, that $A^\top B A$ is singular, with its smallest eigenvalue equal to zero. Then there exists a nonzero vector $x \in \mathbb{R}^n$ such that $A^\top B A x = \mathbf{0}$. This implies that:
\begin{align*}
& \|B^{1/2} A x\|_2^2 = x^\top A^\top B A x  = 0\\
 \quad \Rightarrow \quad & B^{1/2} A x = 0 \\
 \quad \Rightarrow \quad & Ax=0 \tag{$B^{1/2}$ is positive definite.}\\
 \quad \Rightarrow \quad & x=\zero \tag{$A$ has full column rank.}
\end{align*}
which leads to contradiction and thus  $A^\top B A$  is non-singular and positive definite.
\end{proof}

\begin{lemma}
    \label{lemma:sum_scalar_trace}
    Let $A=\frac{1}{n}\sum_{i=1}^n\bbE[\Delta_i \Sigma_i]$ with $\Delta_i$ being a scalar and $\Sigma_i$ being positive semi-definite of dimension $d\times d$. Then:
    \begin{equation}
        \|A\|_{Frob}\leq \sqrt{d}\frac{1}{n}\sum_i\bbE[|\Delta_i| \Tr(\Sigma_i)].
    \end{equation}
\end{lemma}
\begin{proof}
    By standard norm inequality, we have
    \[
    \|A\|_{Frob}\leq \sqrt{d}\|A\|_2.
    \]
    Moreover, by Jensen's inequality:
    \begin{align*}
        \|A\|_2 = \bn{\frac{1}{n}\sum_{i=1}^n \bbE[\Delta_i\Sigma_i]}_2\leq \frac{1}{n}\sum_{i=1}^n \bbE[\|\Delta_i\Sigma_i\|_2] = \frac{1}{n}\sum_{i=1}^n \bbE[|\Delta_i|\cdot\|\Sigma_i\|_2].
    \end{align*}
    For any $\Sigma_i$ being positive semi-definite, we have
    \begin{align*}
        \|\Sigma_i\|_2=\lambda_{d}(\Sigma_i)\leq \sum_{j=1}^d \lambda_j(\Sigma_i)=\Tr(\Sigma_i),
    \end{align*}
    where we use $\lambda_1(\Sigma_i)\leq \dots \leq \lambda_d(\Sigma_i)$ to denote the eigenvalues of $\Sigma_i$. Then:
  \begin{align*}
    &\|A\|_2\leq \frac{1}{n} \sum_{i=1}^n \bbE[|\Delta_i|\cdot\Tr(\Sigma_i)]\\
\Rightarrow & \|A\|_{Frob}\leq \sqrt{d} \frac{1}{n} \sum_{i=1}^n \bbE[|\Delta_i|\cdot\Tr(\Sigma_i)].
   \end{align*}
\end{proof}

\begin{lemma}[Theorem, \cite{bhatia2010modulus}]
    \label{lemma:contraction}
    There exists a positive number $c$ such that for any two $n\times n$ matrices $A$ and $B$ we have:
    \[
    \||A|-|B|\|_2 \leq c\log(n)\|A-B\|_2,
    \]
    where the matrix absolute value $|M|:=(M^\top M)^{1/2}$.
\end{lemma}

\begin{lemma}
    \label{lemma:clip}
    Let  $A, B$ be two symmetric  matrices of size $d\times d$. For any positive numbers $a<b$, we have:
    \[
    \bn{\Clip_{[a,b]}(A)- \Clip_{[a,b]}(B)}_2\lesssim\log(d)^2\|A-B\|_2,
    \]
    where the clipping operator is defined in \eqref{eq:clip_matrix}.
\end{lemma}
\begin{proof}
Given a symmetric matrix $M = U \Lambda U^\top$, where $U$ is an orthogonal matrix and $\Lambda = \text{diag}\{\lambda_1, \dots, \lambda_d\}$ is the diagonal matrix of eigenvalues, the matrix absolute value is defined as $|M| := (M^2)^{1/2}$.
Define its positive part as:
  \[
  M_+ := \frac{1}{2}\bp{|M|+M} = U\frac{1}{2}(\Lambda + |\Lambda|)U^\top = \frac{1}{2}\diag\{
  \max(\lambda_1,0), \dots,\max(\lambda_d,0)
  \}U^\top. 
  \]
  Using Lemma \ref{lemma:contraction}, for any symmetric matrices $M$, $N$,  we have:
  \begin{align}
      \|M_+-N_+\|_2 &= \bn{\frac{1}{2}\bp{|M|+M} - \frac{1}{2}\bp{|N|+N}}_2\nonumber \\
      &\leq \frac{1}{2}\bn{|M|-|N|}_2 + \frac{1}{2}\bn{M-N}_2 \tag{Triangular inequality}\nonumber\\
      & \leq\frac{1}{2}\cdot c \log(d)\bn{M-N}_2 + \frac{1}{2}\bn{M-N}_2\tag{by Lemma \ref{lemma:contraction}}\nonumber\\
      &=\frac{c\log(d)+1}{2}\bn{M-N}_2. \label{eq:matrix_positive}
  \end{align}
Now for the clipping floor at $a$, define:
\[
\text{floor}_a(M):=aI + (M-aI)_+.
\]
For any symmetric matrices $M$ and $N$ of size $d\times d$:
\begin{align}
    \bn{\text{floor}_a(M)-\text{floor}_a(N)}_2 &= \bn{(M-aI)_+ - (N-aI)_+}_2\nonumber \\
   & \leq  \frac{c\log(d)+1}{2} \bn{(M-aI)-(N-aI)}_2 \tag{by \eqref{eq:matrix_positive}}\nonumber\\
   &= \frac{c\log(d)+1}{2} \bn{M-N}_2.\label{eq:floor}
\end{align}
Similarly, for the clipping ceil at $b$, define
\[
\text{ceil}_b(M):=bI - (bI-M)_+.
\]
For any symmetric matrices $M$ and $N$ of size $d\times d$
\begin{align}
    \bn{\text{ceil}_b(M)-\text{ceil}_b(N)}_2 &= \bn{-(bI-M)_+ + (bI-N)_+}_2\nonumber \\
   & \leq  \frac{c\log(d)+1}{2} \bn{(bI-N)-(bI-M)}_2 \tag{by \eqref{eq:matrix_positive}}\nonumber\\
   &= \frac{c\log(d)+1}{2} \bn{M-N}_2.\label{eq:ceil}
\end{align}
Finally, note that $\Clip_{[a,b]}(M) =\text{ceil}_b(\text{floor}_a(M))$ for any symmetric matrix $M$. We thus have:
\begin{align*}
    \bn{
\Clip_{[a,b]}(A) - \Clip_{[a,b]}(B)
    }_2  &= \bn{
    \text{ceil}_b\bp{\text{floor}_a(A)} - \text{ceil}_b\bp{\text{floor}_a(B)}
    }\\
    &\leq \frac{c\log(d)+1}{2} \bn{\text{floor}_a(A)-\text{floor}_a(B)}_2 \tag{by \eqref{eq:ceil}}\\
    &\leq \bp{\frac{c\log(d)+1}{2}}^2 \|A-B\|_2 \tag{by \eqref{eq:floor}}.
\end{align*}
\end{proof}

\end{APPENDICES}

\end{document}